\documentclass[lettersize,journal]{IEEEtran}

\usepackage{amsmath,amsfonts,amssymb,amsthm}
\usepackage[linesnumbered,ruled,vlined,boxed,noend]{algorithm2e}
\usepackage{array}
\usepackage[caption=false,font=normalsize,labelfont=sf,textfont=sf]{subfig}
\usepackage{stfloats}
\usepackage{url}
\usepackage{graphicx}
\graphicspath{{../}}
\usepackage{booktabs}
\usepackage{multirow}
\usepackage[table]{xcolor}
\usepackage{enumitem}
\usepackage[hypertexnames=false,colorlinks=true,linkcolor=blue,citecolor=blue,urlcolor=blue]{hyperref}
\usepackage{xspace}
\usepackage{xcolor}

\definecolor{darkgray}{gray}{0.3}  
\newcommand{\bW}{\mathbf{W}}
\newcommand{\bU}{\mathbf{U}}
\newcommand{\bV}{\mathbf{V}}
\newcommand{\bA}{\mathbf{A}}
\newcommand{\bB}{\mathbf{B}}
\newcommand{\bC}{\mathbf{C}}
\newcommand{\bM}{\mathbf{M}}
\newcommand{\bX}{\mathbf{X}}
\newcommand{\bSigma}{\boldsymbol{\Sigma}}
\newcommand{\bLambda}{\boldsymbol{\Lambda}}

\newtheorem{definition}{Definition}
\newtheorem{proposition}{Proposition}

\newtheorem{assumption}{Assumption}
\newtheorem{lemma}{Lemma}
\newtheorem{theorem}{Theorem}

\newtheorem{remark}{Remark}

\setlist[itemize]{leftmargin=*}
 
\begin{document}

\title{Federated Foundation Models Fine-Tuning with Heterogeneous Compressed Clients}

\author{Shengkun~Zhu$^{\dagger}$,
        Jinshan~Zeng$^{\dagger}$,
        Zhihua~Allen-Zhao$^{\dagger}$,
        Mayi~Xu,
        Quanqing~Xu$^{*}$,
        Wei~Ren,
        Qiang~Yang,~\IEEEmembership{Fellow,~IEEE,}
        Yang~Liu$^{*}$,~\IEEEmembership{Member,~IEEE}
\thanks{$^{\dagger}$Equal contribution.}
\thanks{$^{*}$Corresponding authors: Q. Xu and Y. Liu.}
\thanks{S. Zhu is with The Hong Kong Polytechnic University, and also with OceanBase, Ant Group. E-mail: zhushengkun.zsk@antgroup.com}

\thanks{J. Zeng is with the School of Management, Xi'an Jiaotong University. E-mail: jsh.zeng@gmail.com}
\thanks{Z. Allen-Zhao is with the School of Mathematics and Statistics, Xidian University. E-mail: allenzhaozh@gmail.com}
\thanks{M. Xu is with the School of Computer Science, Wuhan University.}
\thanks{Q. Xu is with OceanBase, Ant Group. E-mail: xuquanqing.xqq@oceanbase.com}
\thanks{W. Ren is with the School of Computer Science, China University of Geosciences. E-mail: weirencs@cug.edu.cn}
\thanks{Y. Liu and Q. Yang are with The Hong Kong Polytechnic University. E-mail: yang-veronica.liu@polyu.edu.hk, profqiang.yang@polyu.edu.hk}}

\markboth{Preprint}%
{Zhu \MakeLowercase{\textit{et al.}}: Federated Fine-Tuning of FMs via Heterogeneous Small Model Collaboration}

\maketitle

\begin{abstract}
Federated learning of foundation models faces a fundamental resource-asymmetry challenge: the institutions holding the most valuable domain-specific data cannot host billion-parameter models. Existing heterogeneous federated approaches attempt to bridge this gap through parameter-efficient tuning, model pruning, or knowledge distillation, yet each trades away a critical property, whether full-model memory reduction, architectural self-containedness, or representational fidelity, leaving the core tension unresolved. We propose FedSLM, a parameter-centric framework for federated fine-tuning with heterogeneous compressed clients. FedSLM uses SVD-based decomposition to produce self-contained client models, whose low-rank subspaces form nested manifolds that are structurally compatible for aggregation. It then applies a two-stage protocol that synchronizes lightweight adapters within compression groups and fuses full-rank reconstructions across groups via structural alignment. Finally, a weak-to-strong elicitation step with auxiliary confidence loss transfers the aggregated knowledge to the full-scale server, while an explicit bias--variance trade-off mitigates compression artifacts. We provide theoretical guarantees for adapter-level aggregation, subspace-alignment bounds for cross-group fusion, and a characterization of how the confidence loss mitigates weak-supervision noise. Experiments on natural language and vision--language benchmarks show that FedSLM outperforms existing federated baselines under both IID and non-IID partitions, while client models operate at roughly 50\% of the GPU memory required by the full model.
\end{abstract}


\section{Introduction}

Foundation Models (FMs) derive their capabilities from scaling model parameters and training data \cite{achiam2023gpt, yang2025qwen3, liu2024deepseek}. 
As these models are increasingly deployed in specialized domains, the demand for domain-specific training data has grown correspondingly. 
However, much of this data resides in private silos, such as medical records \cite{thirunavukarasu2023large, moor2023foundation} and financial documents \cite{wu2023bloomberggpt}, where privacy regulations and commercial confidentiality preclude centralized collection.

Federated learning (FL)~\cite{mcmahan2017fedavg,yang2019federated} offers a principled solution by enabling multiple participants to collaboratively train a shared model without exposing raw data. 
Yet applying FL to billion-parameter FMs exposes a \emph{resource-asymmetry paradox}: the institutions that hold the most valuable domain-specific data (hospitals, financial firms, and enterprise departments) often operate under GPU memory budgets that preclude hosting or fine-tuning the full model.
This paradox is not merely an engineering inconvenience; it creates a fundamental tension between \emph{inference independence} (each client must run a self-contained model) and \emph{aggregation compatibility} (heterogeneous client updates must be fusible into a coherent global model). Existing approaches to heterogeneous FL, while effective in their respective settings, face three open challenges when applied to billion-parameter FMs.

Parameter-efficient fine-tuning methods such as LoRA~\cite{hu2022lora} reduce trainable parameters to mere megabytes, and recent federated variants~\cite{zhang2023fedpetuning,zhang2023fedit,bai2024flexlora,cho2024hetlora,wang2024flora} further allow adapters of varying ranks to accommodate client heterogeneity. 
However, this efficiency is confined to the \emph{update} dimension: every client must still load the complete pre-trained model for forward computation. 
For Llama-2-7B \cite{touvron2023llama}, the base weights alone occupy over 13\,GB of VRAM, so although the number of trainable parameters shrinks, the \emph{deployment footprint} remains that of the full model.

An alternative line of work enables model heterogeneity through architectural extraction. 
HeteroFL~\cite{diao2020heterofl} and FjORD~\cite{horvath2021fjord} allow clients to train sub-networks of varying widths via channel-wise slicing or ordered dropout, while FedRolex~\cite{alam2022fedrolex}, DepthFL~\cite{kim2023depthfl}, and ScaleFL~\cite{ilhan2023scalefl} explore rolling, depth-wise, and resource-adaptive partitioning. 
These methods are designed for CNNs where channels are relatively independent. 
In Transformer architectures, however, reducing hidden dimensions can disrupt the tightly coupled head-wise attention patterns formed during pre-training, leading to performance degradation in the extracted sub-models. 
Moreover, because each sub-model's parameters are a subset of the global model, updates from clients of different capacities may interfere during aggregation, making it difficult to balance client-side heterogeneity with server-side model quality.

Data-centric approaches~\cite{wu2022fedkd,fan2024fedmkt,chen2024pcevolve} sidestep architectural heterogeneity by exchanging synthetic samples or output probability distributions. However, this requires projecting rich hidden-layer representations onto the output vocabulary~\cite{cheng2020explaining,long2023ibkd}, and when client and server models use different tokenizers, where vocabulary overlap can be as low as 6\%~\cite{zhang2024dskd}, aligning the resulting logit spaces remains an open problem despite recent efforts in optimal transport~\cite{boizard2024uld}, approximate likelihood matching~\cite{minixhofer2025alm}, and token co-occurrence alignment~\cite{li2025tokalign}. Furthermore, capacity-limited client models tend to introduce systematic approximation errors that, once aggregated, may propagate to the server through the \emph{imitation fallacy}~\cite{burns2024weak}.

The common thread across these three challenges is the absence of a mechanism that simultaneously (i)~produces client models compact enough for deployment on resource-constrained clients, (ii)~maintains structural alignment for parameter-level aggregation across heterogeneous clients, and (iii)~enables the full-scale server to absorb federated domain knowledge without inheriting compression artifacts. We propose \textbf{FedSLM} to address all three requirements within a unified optimization framework. Rather than requiring clients to host the full model or relying on lossy logit-level communication, FedSLM embraces lightweight compressed client models and focuses on extracting and amplifying the domain-specific knowledge they acquire through FL.

FedSLM formulates the problem as a coupled optimization over three variables: client-side adapters, server-side reconstructed global weights, and the final strong server model. First, we derive a family of compressed client models from the server FM via SVD-based decomposition~\cite{yuan2023asvd,wang2024svdllm,wang2024dobisvd} at multiple compression ratios. Since all compressed models are derived from the same server weights via rank truncation, they retain a shared low-rank structure that naturally supports parameter-level aggregation across different compression ratios. 
Second, we design a two-stage aggregation protocol: the first stage synchronizes lightweight adapters within groups of clients sharing the same compression ratio; the second stage reconstructs full-parameter representations and fuses them across groups into a unified global model via structural alignment in the common full-rank space. Third, to bridge the capacity gap between compressed client models and the full-scale server, we introduce weak-to-strong elicitation with an auxiliary confidence loss~\cite{burns2024weak} that treats the aggregated model as a noisy but informative supervisor, transferring federated domain knowledge while mitigating compression artifacts through an explicit bias--variance trade-off.

Our contributions are as follows:
\begin{itemize}
    \item We propose FedSLM, a unified federated framework comprising three tightly coupled components: SVD-based decomposition that produces self-contained compressed models via nested low-rank manifolds at variable compression ratios, a two-stage aggregation protocol that achieves structural alignment within compression groups and fuses full-rank reconstructions across groups, and weak-to-strong elicitation with an auxiliary confidence loss that transfers federated knowledge to the server while mitigating compression artifacts.
    \item We provide theoretical guarantees from three complementary perspectives. We establish an $\mathcal{O}(T^{-1/2})$ convergence rate for adapter-level aggregation with an explicit client-drift term that accounts for data heterogeneity. We derive subspace-alignment bounds that characterize the geometric compatibility of heterogeneous compression groups and bound the cross-group aggregation error. We further show that the auxiliary confidence loss mitigates compression noise during weak-to-strong elicitation through a bias--variance trade-off, where the self-anchoring term provides genuine noise suppression and the mixing coefficient $\alpha$ controls the optimal operating point.
    \item We conduct extensive experiments on natural language and vision--language benchmarks under both IID and non-IID partitions. FedSLM consistently outperforms existing federated baselines across all settings. Compressed client models operate at roughly 50\% of the GPU memory required by the full model while achieving strong task-specific performance through federated adapter training, recovering and exceeding the zero-shot accuracy of the uncompressed model on multiple benchmarks.
\end{itemize}

\section{Background and Related Work}\label{sec:background}

We organize existing work by the \emph{knowledge carrier} used for cross-client transfer, which exposes the limitation that each paradigm inherits and motivates the design of FedSLM.

\subsection{Data-Centric Knowledge Transfer}

The earliest approaches to heterogeneous FL transfer knowledge through output-level representations, such as prediction probabilities, synthetic samples, or soft labels, rather than model parameters. Knowledge distillation~\cite{hinton2015distilling} forms the foundation: a teacher's soft targets serve as the carrier for transferring capabilities to a student of arbitrary architecture. In federated settings, FedGEMS~\cite{cheng2021fedgems} enables a larger server model to selectively learn from smaller client models via logit exchange, while FedGKT~\cite{he2020fedgkt} alternates knowledge transfer between edge CNNs and a server-side large model. FedKD~\cite{wu2022fedkd} reduces communication by up to 94.89\% through adaptive mutual distillation. For FM-specific collaboration, FedMKT~\cite{fan2024fedmkt} addresses tokenizer heterogeneity through minimum edit distance alignment and selective knowledge transfer. Generation-based methods like CrossLM~\cite{deng2023crosslm} leverage the FM's generative capability to synthesize training data guided by small-model feedback, and TAKFL~\cite{morafah2024takfl} addresses knowledge dilution from heterogeneous devices through task arithmetic integration.

Despite their architectural flexibility, data-centric methods face two key challenges in the FM regime. First, projecting rich hidden-layer representations onto the output vocabulary discards much of the knowledge encoded in intermediate geometry~\cite{cheng2020explaining,long2023ibkd}. Second, when client and server models use different tokenizers, where vocabulary overlap can be as low as 6\%~\cite{zhang2024dskd}, aligning the resulting logit spaces remains an open problem despite recent efforts~\cite{boizard2024uld,minixhofer2025alm,li2025bild,li2025tokalign}.

\subsection{Architecture-Centric Heterogeneity}

A second paradigm accommodates heterogeneous clients by extracting sub-models of varying sizes from a global model, using \emph{weight slices} as the knowledge carrier.
HeteroFL~\cite{diao2020heterofl} pioneers this approach by allowing clients to train submodels of varying widths through channel selection. 
FjORD~\cite{horvath2021fjord} introduces Ordered Dropout to create nested representations with importance-based pruning, while FedDrop~\cite{wen2022feddrop} applies random dropout to generate heterogeneous subnets. 
FedRolex~\cite{alam2022fedrolex} proposes rolling sub-model extraction, where the extraction window advances each round. 
Beyond width-based scaling, DepthFL~\cite{kim2023depthfl} explores depth-based scaling through layer pruning, and ScaleFL~\cite{ilhan2023scalefl} combines both dimensions for two-dimensional scaling.
These methods are designed for CNNs where channels are relatively independent; in Transformer architectures, however, reducing hidden dimensions can disrupt the coupled head-wise attention patterns formed during pre-training, leading to performance degradation.

\subsection{PEFT-Based Federated Fine-Tuning}

A third paradigm uses \emph{incremental low-rank matrices} (e.g., LoRA adapters) as the knowledge carrier. 
LoRA~\cite{hu2022lora} freezes pre-trained weights and learns low-rank decomposition matrices $\mathbf{B} \in \mathbb{R}^{m \times r}$ and $\mathbf{A} \in \mathbb{R}^{r \times n}$ where $r \ll \min(m,n)$, reducing trainable parameters by orders of magnitude. 
FedPETuning~\cite{zhang2023fedpetuning} systematically studied PEFT methods in federated settings, and FedIT~\cite{zhang2023fedit} introduced federated instruction tuning for FMs. FlexLoRA~\cite{bai2024flexlora} enables dynamic rank adjustment with SVD-based weight redistribution, HetLoRA~\cite{cho2024hetlora} proposes zero-padding aggregation with sparsity-weighted combination, and FLoRA~\cite{wang2024flora} introduces stacking-based aggregation to eliminate bias. 
FFA-LoRA~\cite{sun2024ffalora} identifies LoRA's vulnerabilities under differential privacy and proposes freezing one factor matrix. Industrial frameworks including FATE-LLM~\cite{fan2023fatellm}, FederatedScope-LLM~\cite{kuang2024federatedscope}, and OpenFedLLM~\cite{ye2024openfedllm} provide comprehensive infrastructure for federated FM fine-tuning.
However, PEFT methods only reduce the trainable parameter count; every client must still load the complete base model for forward computation, so the deployment footprint remains that of the full model. 

\subsection{SVD-Based Model Compression}

Deploying FMs on heterogeneous clients necessitates aggressive compression. Singular Value Decomposition (SVD) constitutes a mathematically principled approach by decomposing $\mathbf{W} \in \mathbb{R}^{m \times n}$ into low-rank factors $\mathbf{U} \in \mathbb{R}^{m \times r}$ and $\mathbf{V} \in \mathbb{R}^{r \times n}$, reducing parameters from $mn$ to $r(m+n)$. ASVD~\cite{yuan2023asvd} introduced activation-aware decomposition to preserve perceptually salient weight components. SVD-LLM~\cite{wang2024svdllm} proposed truncation-aware data whitening, SVD-LLM V2~\cite{svdllmv2-2025} extended this with adaptive layer-wise truncation, and Dobi-SVD~\cite{wang2024dobisvd} reformulated truncation position selection as a differentiable optimization problem. QSVD~\cite{qsvd2025} demonstrated effective combination of SVD with low-precision quantization. 
Compression alone, however, does not solve the federated aggregation problem. Clients at different compression ratios produce weight matrices of incompatible dimensions, and compression artifacts can accumulate during aggregation and propagate to the server through the \emph{imitation fallacy}~\cite{burns2024weak}.

\section{Proposed Method}

We use boldface uppercase letters ($\bW$, $\bU$, $\bV$, $\bSigma$, $\mathbf{S}$) for matrices, lowercase and Greek letters ($r$, $m$, $n$, $\eta$, $\alpha$) for scalars, and calligraphic letters ($\mathcal{S}$, $\mathcal{C}_i$, $\mathcal{D}_c$, $\mathcal{Y}$) for sets or spaces. The symbol $\triangleq$ means ``defined as,'' $\mathbf{A}^{\dagger}$ denotes the Moore--Penrose pseudoinverse, and $\|\cdot\|_F$, $\|\cdot\|_2$, $\|\cdot\|_1$ denote the Frobenius, spectral, and entry-wise $\ell_1$ norms, respectively. Table~\ref{tab:notation} lists the key symbols.

\begin{table}[t]
\centering
\caption{Summary of key notation.}
\vspace{-0.5em}
\label{tab:notation}
\renewcommand{\arraystretch}{1.15}
\small
\begin{tabular}{@{}cl@{}}
\toprule
\textbf{Symbol} & \textbf{Description} \\
\midrule
$\bW \in \mathbb{R}^{m \times n}$
  & Weight matrix of a selected layer \\
$\bU_r \in \mathbb{R}^{m \times r}$
  & Left factor from rank-$r$ decomposition \\
$\bV_r \in \mathbb{R}^{n \times r}$
  & Right factor from rank-$r$ decomposition \\
$\bU_r^{\circ},\,\bV_r^{\circ}$
  & Column-orthonormal counterparts of
    $\bU_r,\bV_r$ \\
$\bSigma_{i,c} \in \mathbb{R}^{r_i \times r_i}$
  & Trainable adapter of client $c$ in group $i$ \\
$\mathbf{S} \in \mathbb{R}^{r \times r}$
  & Invertible scaling matrix (SVD variants) \\
$\mathbf{P}_r,\,\mathbf{Q}_r$
  & Orthogonal projectors onto
    $\mathcal{U}_r,\mathcal{V}_r$ \\
$\mathcal{S}$
  & Set of selected layers for decomposition \\
$\mathcal{C}_i$
  & Client set of compression group $i$ \\
$\mathcal{D}_c$
  & Local dataset of client $c$ \\
$\mathcal{U}_r,\,\mathcal{V}_r$
  & Rank-$r$ left / right singular subspaces \\
\bottomrule
\end{tabular}
\end{table}

We formulate FedSLM as a constrained optimization problem over three coupled variables: client-side adapter parameters, server-side reconstructed global weights, and the final strong server model. The key idea is to restrict heterogeneous client models to a shared low-rank feasible set defined by nested SVD subspaces, then solve the resulting problem through alternating local optimization, structural alignment via server-side aggregation, and weak-to-strong refinement.

\subsection{Problem Formulation}
Let $\mathcal{M}_{LM}$ denote the pre-trained server-side large model with layer-wise weights $\{\bW^{(l)}\}_{l=1}^{L}$, and let $\mathcal{S} \subseteq [L]$ denote the subset of layers selected for low-rank decomposition. For each compression ratio $r_i$ and each selected layer $l \in \mathcal{S}$, the server applies an SVD-based compression operator to obtain a family of structurally compatible small models,
\begin{equation}
\hspace{-0.8em} \bW^{(l)} \approx \bU_i^{(l)}\bV_i^{(l)}, \; \bU_i^{(l)} \in \mathbb{R}^{m \times r_i}, \; \bV_i^{(l)} \in \mathbb{R}^{r_i \times n}, \; \forall l \in \mathcal{S},
\end{equation}
where $r_i \ll \min(m,n)$. For a client $c \in \mathcal{C}_i$, we insert a trainable adapter $\bSigma_{i,c}^{(l)} \in \mathbb{R}^{r_i \times r_i}$ only at layers in $\mathcal{S}$, so that the effective local layer is
\begin{equation}
\widetilde{\bW}_{i,c}^{(l)} = \bU_i^{(l)}\bSigma_{i,c}^{(l)}\bV_i^{(l)}, \qquad \forall l \in \mathcal{S}.
\end{equation}
Accordingly, each client is constrained to optimize only adapter parameters on the selected layers, while the low-rank bases $\{\bU_i^{(l)},\bV_i^{(l)}\}_{l \in \mathcal{S}}$ remain fixed. 
Let $F_c(\{\bSigma_{i,c}^{(l)}\}_{l \in \mathcal{S}};\mathcal{D}_c)$ denote the empirical loss on client dataset $\mathcal{D}_c$ with size $n_c$. The client-side optimization problem for group $\mathcal{C}_i$ is
\begin{equation}
\begin{split}
\min_{\{\bSigma_{i,c}^{(l)}\}_{c \in \mathcal{C}_i,\, l \in \mathcal{S}}} &\sum_{c \in \mathcal{C}_i} \frac{n_c}{n_i} F_c(\{\bSigma_{i,c}^{(l)}\}_{l \in \mathcal{S}};\mathcal{D}_c),
\end{split}
\label{eq:local_obj}
\end{equation}
where $n_i = \sum_{c \in \mathcal{C}_i} n_c$.
After intra-group local updates, the server reconstructs one full-parameter model for each compression group and solves a consensus fusion problem in the common full-rank space over the selected layers:
\begin{equation}
\begin{split}
\min_{\{\bW_{\mathrm{g}}^{(l)}\}_{l \in \mathcal{S}}} &\sum_{i=1}^{K} \frac{n_i}{N} \sum_{l \in \mathcal{S}} \left\| \bW_{\mathrm{g}}^{(l)} - \bW_i^{(l)} \right\|_F^2 \\
\text{s.t.} \quad &\bW_i^{(l)} = \bU_i^{(l)}\bSigma_{i}^{(T_1,l)}\bV_i^{(l)}, \; \forall i \in [K], \forall l \in \mathcal{S},
\end{split}
\label{eq:global_obj}
\end{equation}
where $\bSigma_{i}^{(T_1,l)}$ is the aggregated adapter of group $\mathcal{C}_i$ and $N = \sum_{i=1}^{K} n_i$. Although each $\bW_i^{(l)}$ has rank at most $r_i$, we do not treat them as points on different low-rank manifolds: after reconstruction they are all elements of the same Euclidean space $\mathbb{R}^{m \times n}$, all fine-tuned from the same pre-trained anchor $\bW^{(l)}$, and thus differ only by small ambient perturbations. In this common frame, each $\bW_i^{(l)}$ is a data-noisy estimate of a shared consensus layer, and the closed-form minimizer $\bW_{\mathrm{g}}^{(l)} = \sum_i (n_i/N)\bW_i^{(l)}$ is simply the weighted barycenter, recovering FedAvg semantics in a rank-heterogeneous setting.
Finally, let $\bW_{\mathrm{g}}$ denote the full weak model obtained by replacing the selected layers with $\{\bW_{\mathrm{g}}^{(l)}\}_{l \in \mathcal{S}}$ and keeping the unselected layers $\{\bW^{(l)}\}_{l \notin \mathcal{S}}$ unchanged. Inspired by the weak-to-strong generalization framework~\cite{burns2024weak}, we use this aggregated weak model as a supervisor to elicit the full capacity of the server FM. Specifically, the server refines the full FM by solving
\begin{equation}
\begin{split}
\min_{\bW_{\mathrm{s}}} \; \mathbb{E}_{x \sim \mathcal{D}_{\text{server}}}\Big[
&(1\!-\!\alpha)\,\ell_{\mathrm{CE}}\!\big(f_{\text{w}}(x;\bW_{\text{g}}),\,f_{\mathrm{s}}(x;\bW_{\mathrm{s}})\big) \\
+\; &\alpha\,\ell_{\mathrm{CE}}\!\big(\hat{f}_{\mathrm{s}}(x;\bW_{\mathrm{s}}),\,f_{\mathrm{s}}(x;\bW_{\mathrm{s}})\big)\Big],
\end{split}
\label{eq:server_obj}
\end{equation}
where $\ell_{\mathrm{CE}}$ is the cross-entropy loss, $\bW_{\mathrm{s}}$ denotes the strong server model weights, $\bW_{\mathrm{g}}$ is the reconstructed global weak model, $f_{\text{w}}$ and $f_{\mathrm{s}}$ are the forward functions of the weak and strong models respectively, $\hat{f}_{\mathrm{s}}$ is the strong model's output with stopped gradients, and $\alpha \in [0,1]$ balances imitation of the weak aggregated model against the strong model's own confident predictions. Overall, FedSLM solves Eqs.~(\ref{eq:local_obj})--(\ref{eq:server_obj}) by decomposing the coupled objective into three subproblems: low-rank model derivation, heterogeneous federated aggregation, and server-side weak-to-strong refinement.

\noindent\textit{\textbf{Threat Model.}}
We follow the standard honest-but-curious assumption of FL \cite{mcmahan2017fedavg}: server and clients faithfully execute the protocol but may attempt to infer information from exchanged messages, and only the lightweight adapters $\{\bSigma_{i,c}^{(l)}\}_{l \in \mathcal{S}}$ are transmitted. Raw data, gradients, and the fixed subspace factors $\{\bU_i^{(l)},\bV_i^{(l)}\}$ never leave their originating party, so orthogonal privacy-preserving primitives such as secure aggregation or differential privacy can be layered on adapter exchanges without altering the algorithm.
The complete optimization procedure is summarized in Algorithm~\ref{alg:sm2lm}, and Figure~\ref{fig:workflow} illustrates the overall framework.

\begin{algorithm}[t]
    \caption{FedSLM}
    \label{alg:sm2lm}
    \DontPrintSemicolon
    \KwIn{Weights $\{\bW^{(l)}\}_{l=1}^{L}$, layers $\mathcal{S}$, ratios $\{r_i\}_{i=1}^{K}$, groups $\{\mathcal{C}_i\}_{i=1}^{K}$, 
    client data $\mathcal{D}_{\mathrm{c}}$, rounds $T_1, T_2$, weight $\alpha$}
    \KwOut{Strong model $\bW_{\mathrm{s}}$, Weak supervisor $\bW_{\mathrm{g}}$, Adapters $\{\bSigma_{i}^{(l)}\}_{i,l}$, Compressed models: $\{\bU_i^{(l)}, \bV_i^{(l)}, \bSigma_{i}^{(l)}\}_{l \in \mathcal{S}}$} 
    
    \BlankLine
    \textcolor{darkgray}{\underline{\textit{\sffamily Step 1: SVD Compression \& Adapter Initialization}}}\\[2pt]
    \For{$i = 1, \ldots, K$}{
        \ForEach{$l \in \mathcal{S}$}{
            $\bU_i^{(l)}, \bV_i^{(l)} \leftarrow \texttt{SVD-Compress}(\bW^{(l)}, r_i)$ \\
            Initialize $\bSigma_{i}^{(l)} \in \mathbb{R}^{r_i \times r_i}$ \\
        }
        Send compressed model to $\mathcal{C}_i$ \\
    }
    
    \BlankLine
    \textcolor{darkgray}{\underline{\textit{\sffamily Step 2: Heterogeneous Clients Local Fine-Tuning}}}\\[2pt]
    \For{$t = 1, \ldots, T_1$}{
        \For{$i = 1, \ldots, K$}{
            \ForEach{$c \in \mathcal{C}_i$}{
                {\textcolor{darkgray}{{\textit{\underline{Local Optimization:}}}}}
                $\{\bSigma_{i,c}^{(t,l)}\}_{l} \leftarrow \texttt{LocalTrain}(\bU_i, \bV_i, \bSigma_{i}^{(t-1)},\mathcal{D}_c)$ \\
            }
            \textcolor{darkgray}{\underline{\textit{\sffamily Step 3: Two-Stage Aggregation}}}\\[2pt]
            {\textcolor{darkgray}{{\textit{\underline{Stage 1: Intra-Group Adapter Aggregation:}}}}} \\
            $\bSigma_{i}^{(t,l)} \leftarrow \sum_{c \in \mathcal{C}_i} \frac{n_c}{n_i}\, \bSigma_{i,c}^{(t,l)}, \;\; \forall\, l \in \mathcal{S}$ \\
            {\textcolor{darkgray}{{\textit{\underline{Stage 2: Cross-Group Model Fusion:}}}}}
            $\bW_i^{(l)} \leftarrow \bU_i^{(l)} \bSigma_{i}^{(t,l)} \bV_i^{(l)}, \;\; \forall\, l \in \mathcal{S}$ \\
            
            $\bW_{\mathrm{g}}^{(l)} \leftarrow \sum_{i=1}^{K} \frac{n_i}{N} \bW_i^{(l)}, \;\; \forall\, l \in \mathcal{S}$ \\
        }
    }

    \BlankLine
    \textcolor{darkgray}{\underline{\textit{\sffamily Step 4: Weak-to-Strong Elicitation}}} \\[2pt]
    Weak Supervisor: $\bW_{\mathrm{g}}$, Strong Student: $\bW_{\mathrm{s}} $ \\
    \For{$t = 1, \ldots, T_2$}{
        $\mathcal{L} \leftarrow (1\!-\!\alpha)\,\ell_{\mathrm{CE}}(f_{\mathrm{w}},\, f_{\mathrm{s}}) + \alpha\,\ell_{\mathrm{CE}}(\hat{f}_{\mathrm{s}},\, f_{\mathrm{s}})$\;
        Update $\bW_{\mathrm{s}}$ via $\nabla_{\bW_{\mathrm{s}}} \mathcal{L}$ \\
    }
    \Return $\bW_{\mathrm{s}}$, $\bW_{\mathrm{g}}$, $\{\bSigma_{i}^{(T_1,l)}\}_{i,l}$
\end{algorithm}

\begin{figure*}[t!]
  \centering
  \includegraphics[width=\linewidth]{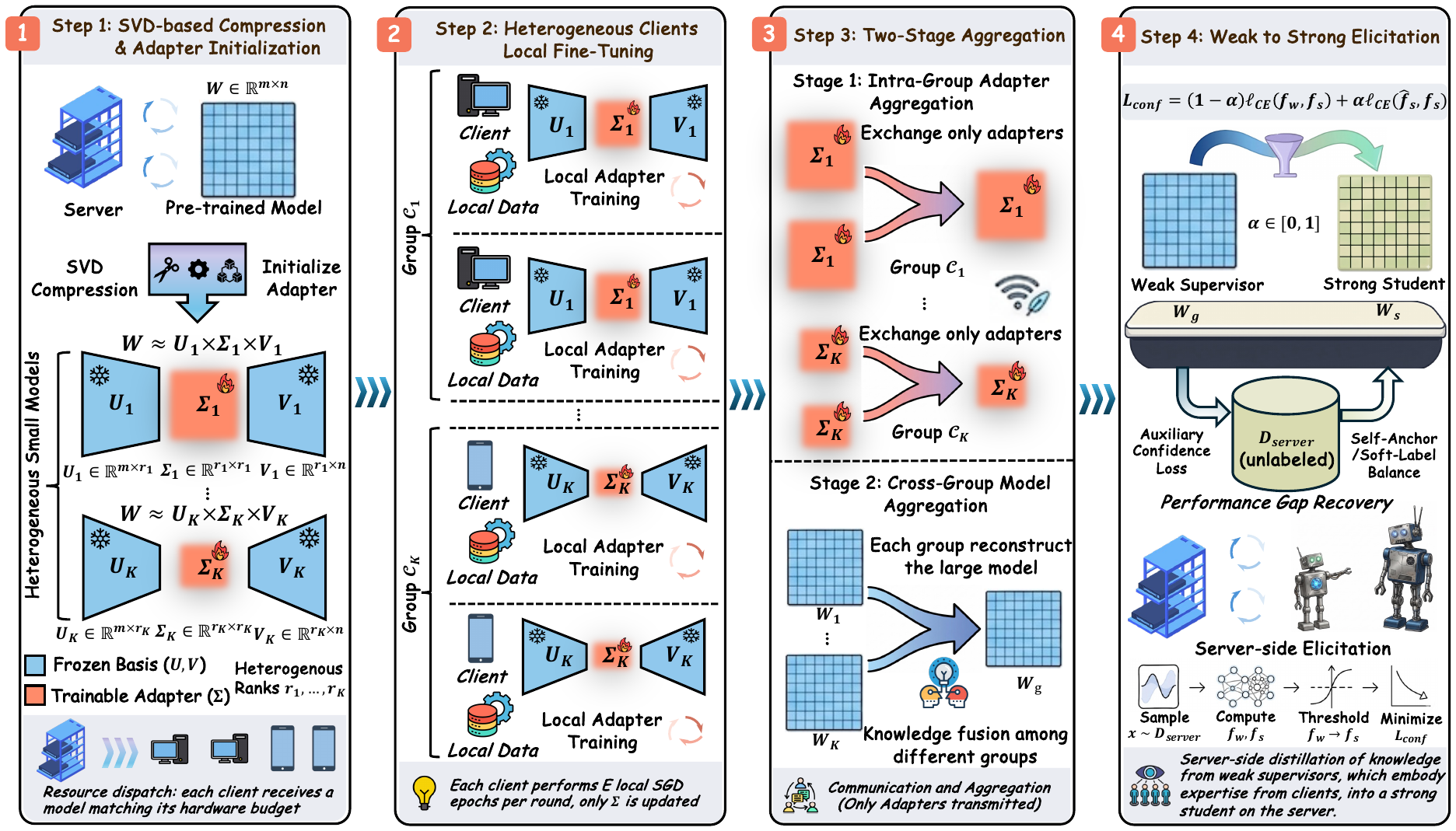}
  \vspace{-1em}
  \caption{Overview of the FedSLM framework. 
  Step~1: the server derives a family of structurally compatible compressed models at heterogeneous compression ratios via SVD-based decomposition and distributes the fixed low-rank factors $(\bU_i, \bV_i)$ along with trainable adapters $\bSigma_i$ to each client group. 
  Step~2 and Step 3 (Stage~1): clients optimize adapters locally, and the server aggregates them within each group. 
  Step~2 (Stage~2): the server reconstructs full-rank models from each group and fuses them into a unified global model via weighted averaging in the common ambient space. 
  Step~4: the server refines the full-scale model through weak-to-strong elicitation, using the aggregated model as a supervisor and an auxiliary confidence loss to mitigate compression artifacts.
  }
  \label{fig:workflow}
  \vspace{-0.5em}
\end{figure*}

\subsection{Low-Rank Model Derivation}
The first subproblem concerns the construction of a feasible model family that is simultaneously compatible with heterogeneous client resources and structurally aligned for downstream aggregation. A naive alternative would be to randomly initialize a collection of small models and train them from scratch. Such a strategy is fundamentally ill-suited to our setting, as resource-constrained clients do not have access to either the large-scale corpora or the computational budget required to recover the knowledge already embedded in the pre-trained server large model (LM). 

We therefore solve this subproblem by deriving client models directly from the server LM through SVD-based compression on a selected layer subset $\mathcal{S}$. Crucially, SVD operates as a \emph{basis transformation} that compresses the information rank of each layer without altering its dimensional interface. Concretely, for each selected layer $l \in \mathcal{S}$, we construct low-rank factors $\bU_i^{(l)}$ and $\bV_i^{(l)}$ at compression ratio $r_i$ and represent each client-specific effective layer:
\begin{equation}
\begin{split}
\bW^{(l)} &\approx \bU_i^{(l)}\bV_i^{(l)}, \\
\widetilde{\bW}_{i,c}^{(l)} &= \bU_i^{(l)}\bSigma_{i,c}^{(l)}\bV_i^{(l)}, \quad \forall c \in \mathcal{C}_i, \forall l \in \mathcal{S}.
\end{split}
\end{equation}
This factorization preserves the ambient dimensional structure of each selected layer while reducing its parameterization from $mn$ to $r_i(m+n)$. More importantly, it defines the nested low-rank subspace structure that appears explicitly in Eq.~(\ref{eq:local_obj}) and constrains all subsequent client updates to a structurally compatible form.
A direct truncated SVD, however, often introduces substantial approximation error and leads to severe performance degradation. 
We therefore instantiate this step with a suitable SVD-based compression method selected from the family of techniques \cite{yuan2023asvd, qsvd2025,wang2024dobisvd, svdllmv2-2025, wang2024svdllm} discussed in Section~\ref{sec:background}. 
The resulting factors $\{\bU_i^{(l)}, \bV_i^{(l)}\}_{l \in \mathcal{S}}$ are frozen throughout federated optimization and serve solely as the fixed subspace basis. 
Client-specific task adaptation is absorbed entirely by the adapters $\{\bSigma_{i,c}^{(l)}\}_{l \in \mathcal{S}}$, each initialized to the identity matrix $\mathbf{I}_{r_i}$ so that the initial reconstruction $\bU_i^{(l)}\bSigma_{i,c}^{(l)}\bV_i^{(l)}$ coincides with the compressed model itself. During training, only $\bSigma$ is updated; $\bU$ and $\bV$ remain fixed.

This design is particularly important under system heterogeneity. In practice, clients exhibit markedly different memory and compute budgets, so a single model size is generally suboptimal. We therefore instantiate a family of compressed models with heterogeneous compression ratios $\{r_1, r_2,\cdots, r_k\}$ on the shared selected layer set $\mathcal{S}$, where each $r_i$ corresponds to a distinct feasible set indexed by $\mathcal{C}_i$. Because all variants are derived from the same server model via nested SVD truncation, they inherit structurally compatible low-rank subspaces, which is precisely the property needed for the reconstruction constraints in Eq.~(\ref{eq:global_obj}) to remain meaningful across heterogeneous groups.

\subsection{Heterogeneous Federated Aggregation}\label{sec: Two-Stage Federated Aggregation}
Having specified the feasible model family, the second subproblem is to optimize client-specific adapters and fuse the resulting updates into a single global model. Although the decomposed models are sufficiently compact for inference, full-parameter fine-tuning remains infeasible on resource-constrained clients because backpropagation requires storing activations and gradients in addition to model weights. In practice, training memory is often several times larger than inference memory, making direct end-to-end optimization of the compressed model prohibitive.
Therefore, we restrict local optimization to the lightweight adapters $\{\bSigma_{i,c}^{(l)}\}_{l \in \mathcal{S}}$ inserted between the fixed factors $\{\bU_i^{(l)},\bV_i^{(l)}\}_{l \in \mathcal{S}}$. 
Under this parameterization, each client in group $\mathcal{C}_i$ solves Eq.~(\ref{eq:local_obj}) on the selected layers, while the server subsequently solves Eq.~(\ref{eq:global_obj}) after mapping all heterogeneous updates back to the common full-rank space.
This subproblem is solved in two stages:

\noindent\textbf{Stage 1: Intra-Group Adapter Optimization.}
For each compression group $\mathcal{C}_i$, clients independently optimize the variables appearing in Eq.~(\ref{eq:local_obj}), namely $\{\bSigma_{i,c}^{(l)}\}_{l \in \mathcal{S}}$, while keeping $\{\bU_i^{(l)}, \bV_i^{(l)}\}_{l \in \mathcal{S}}$ fixed. Because all models inside a group share the same rank $r_i$, the server can aggregate the resulting adapter updates on the selected layers via standard FedAvg \cite{mcmahan2017fedavg}:

\begin{equation}
\bSigma_{i}^{(t,l)} = \sum_{c \in \mathcal{C}_i} \frac{n_c}{n_i} \bSigma_{i,c}^{(t,l)}, \qquad \forall l \in \mathcal{S}.
\end{equation}

This first stage is therefore the explicit solver of the local group-level problem in Eq.~(\ref{eq:local_obj}), while preserving communication efficiency because only the lightweight parameters $\{\bSigma_{i,c}^{(t,l)}\}_{l \in \mathcal{S}}$ are exchanged.

\noindent\textbf{Stage 2: Inter-Group Reconstruction and Global Fusion.}
The central difficulty in heterogeneous FL is that adapters associated with different ranks cannot be aggregated directly. 
To remove this dimensional mismatch, we reconstruct the selected layers of one full-parameter model for each compression group. 
Specifically, after the final Stage-1 round, the server integrates the aggregated adapter $\bSigma_{i}^{(T_1,l)}$ into the decomposed structure and recovers each selected layer via
\begin{equation}
\bW_i^{(l)} = \bU_i^{(l)}\bSigma_{i}^{(T_1,l)}\bV_i^{(l)}, \qquad \forall i \in [K], \forall l \in \mathcal{S}.
\end{equation}
Note that the average is taken in ambient space, not on any single rank-$r_i$ manifold; because all groups share a nested SVD basis, the aggregated $\bW_{\mathrm{g}}^{(l)}$ may attain rank up to $\max_i r_i$, an intentional enrichment beyond any individual group that carries more information for the weak-to-strong supervisor in Sec.~\ref{sec:weak_to_strong}.
Once all groups are lifted into the same ambient space on the selected layers, the global fusion subproblem in Eq.~(\ref{eq:global_obj}) admits the weighted averaging solution
\begin{equation}
\bW_{\text{g}}^{(l)} = \sum_{i=1}^{K} \frac{n_i}{N} \bW_i^{(l)}, \qquad \forall l \in \mathcal{S}.
\end{equation}
This solution is the closed-form minimizer of the consensus objective above, so it inherits a clear statistical meaning: groups with more private data exert proportionally larger influence on the fused layer, while the shared pre-trained initialization keeps the optimization in a locally comparable parameter region. We then form the full weak model by replacing the selected layers of the original server LM with $\{\bW_{\text{g}}^{(l)}\}_{l \in \mathcal{S}}$ and leaving the remaining layers unchanged.

\subsection{Weak-to-Strong Refinement}
\label{sec:weak_to_strong}
After aggregation, the server obtains a weak global model $\bW_{\text{g}}$ whose selected layers encode the aggregated client knowledge and whose remaining layers are inherited directly from the original server LM. Yet this weak model remains intrinsically constrained by the low-rank structure imposed on the selected layers during compression. By contrast, the original server-side LM, now parameterized by $\bW_{\text{s}}$, preserves substantially richer representational capacity and broader world knowledge from pretraining, but it does not directly absorb the client-specific adaptations learned during federated optimization. 
The goal is therefore to transfer the federated knowledge encoded in $\bW_{\text{g}}$ back into the full-capacity server LM without inheriting the compression artifacts of the weak model.

Formally, let $f_{\text{w}}(\cdot;\bW_{\text{g}})$ denote the predictive function induced by the reconstructed weak model $\bW_{\text{g}}$, and let $f_{\text{s}}(\cdot;\bW_{\text{s}})$ denote the predictive function of the server-resident LM. Given an unlabeled server-side dataset $\mathcal{D}_{\text{server}}$, the server solves Eq.~(\ref{eq:server_obj}) to refine the strong model under weak supervision.

A naive objective that merely matches the weak model's soft labels is inadequate, because it would force the strong LM to imitate the systematic approximation errors induced by low-rank truncation. As noted by Burns et al.~\cite{burns2024weak}, this leads to the \emph{imitation fallacy}: a strong model may regress toward the weak supervisor instead of exploiting its own latent reasoning capabilities.
To overcome this difficulty, we adopt the \textit{Auxiliary Confidence Loss} ($\mathcal{L}_{\text{conf}}$), written in notation as
\begin{equation}
\label{eq:conf_loss}
\begin{split}
\mathcal{L}_{\text{conf}}(\bW_{\mathrm{s}}) =\;
&(1 \!-\! \alpha) \cdot \ell_{\mathrm{CE}}\!\big(f_{\text{w}}(x;\bW_{\text{g}}),\, f_{\mathrm{s}}(x;\bW_{\mathrm{s}})\big) \\
+\; &\alpha \cdot \ell_{\mathrm{CE}}\!\big(\hat{f}_{\mathrm{s}}(x;\bW_{\mathrm{s}}),\, f_{\mathrm{s}}(x;\bW_{\mathrm{s}})\big)
\end{split}
\end{equation}

\noindent where $\ell_{\mathrm{CE}}(p, q)$ denotes the cross-entropy loss between a target distribution $p$ and a predictive distribution $q$. We define the hardened self-prediction of the strong model as
\begin{equation}
\hat{f}_{\mathrm{s}}(x;\bW_{\mathrm{s}}) \triangleq \mathbb{I}[f_{\mathrm{s}}(x;\bW_{\mathrm{s}}) > \tau],
\label{eq:hardened}
\end{equation}
where $\mathbb{I}$ is the indicator function and the threshold $\tau$ is set adaptively within each batch such that $f_{\mathrm{s}}(x;\bW_{\mathrm{s}}) > \tau$ holds for exactly half of the samples. 
The hyperparameter $\alpha \in [0, 1]$ controls the trade-off between inheriting domain knowledge and trusting the strong model's own confident predictions.

In effect, this refinement stage serves as a capacity-recovery step: it transfers the domain specialization learned through federated optimization while allowing the full-scale LM to mitigate low-rank artifacts through the self-anchoring term and re-express that knowledge in a richer hypothesis class. 
The output is therefore not a mere expansion of the compressed model, but a refined strong model that combines specialization with the generalization capacity of the server LM.


\section{Theoretical Analysis}
\label{sec:theory}

We provide theoretical justifications for FedSLM from three complementary perspectives: (i)~a \emph{convergence analysis} that establishes the $\mathcal{O}(T^{-1/2})$ rate of the Stage-1 adapter-level aggregation (Section~\ref{subsec:convergence}); (ii)~a \emph{subspace-alignment theory} that characterizes the geometric compatibility of heterogeneous compression groups and bounds the cross-group aggregation error (Section~\ref{subsec:subspace}); and (iii)~an \emph{artifact-propagation analysis} that quantifies how the auxiliary confidence loss attenuates compression noise during weak-to-strong knowledge elicitation (Section~\ref{subsec:w2s}). All proofs and supporting lemmas are deferred to Appendix~\ref{sec:proofs}.

We fix a single selected layer $\bW \in \mathbb{R}^{m \times n}$; the analysis extends to the set $\mathcal{S}$ of selected layers by summation. Let the singular values of $\bW$ be $\lambda_1 \geq \lambda_2 \geq \cdots \geq \lambda_{\min(m,n)} \geq 0$, with associated left singular vectors $\mathbf{u}_1, \ldots, \mathbf{u}_{\min(m,n)} \in \mathbb{R}^{m}$ and right singular vectors $\mathbf{v}_1, \ldots, \mathbf{v}_{\min(m,n)} \in \mathbb{R}^{n}$. 
For a rank $r$, the server stores the low-rank factorization
\begin{equation}
    \bW \;\approx\; \bU_r\,\bV_r,
    \qquad
    \bU_r \in \mathbb{R}^{m \times r},\;
    \bV_r \in \mathbb{R}^{r \times n},
    \label{eq:factor_notation}
\end{equation}
whose factors are obtained from the rank-$r$ truncated SVD of $\bW$ by splitting the retained spectrum $\bLambda_r \triangleq \mathrm{diag}(\lambda_1, \ldots, \lambda_r)$ evenly across the two sides,
\begin{equation}
    \bU_r \;=\; \bU_r^{\circ}\,\bLambda_r^{1/2},
    \qquad\qquad
    \bV_r \;=\; \bLambda_r^{1/2}\,\bV_r^{\circ},
    \label{eq:balanced_factorization}
\end{equation}
where $\bU_r^{\circ} \in \mathbb{R}^{m\times r}$ and $\bV_r^{\circ} \in \mathbb{R}^{r\times n}$ collect the leading $r$ left/right singular vectors and are column-/row-orthonormal ($\bU_r^{\circ\,\top}\bU_r^{\circ} = \mathbf{I}_r$, $\bV_r^{\circ}\bV_r^{\circ\,\top} = \mathbf{I}_r$), so that $\bU_r\bV_r = \bU_r^{\circ}\bLambda_r\bV_r^{\circ}$ exactly recovers the rank-$r$ truncation $[\bW]_r$. We write the adapter reconstruction at client $c$ of group $\mathcal{C}_i$ as
\begin{equation}
    \widetilde{\bW}_{i,c}
    \;=\;
    \bU_{r_i}\,\bSigma_{i,c}\,\bV_{r_i},
    \qquad
    \bSigma_{i,c} \in \mathbb{R}^{r_i \times r_i},
    \label{eq:reconstruction_notation}
\end{equation}
where the trainable core $\bSigma_{i,c}$ (initialized to $\mathbf{I}_{r_i}$) is the only quantity optimized by clients; the factors $\bU_{r_i}, \bV_{r_i}$ are frozen.

A truncated SVD picks $\bU_r, \bV_r$ to minimize the weight reconstruction error $\|\bW - \bU_r\bV_r\|_F$. 
For an FM layer this is the wrong objective: accuracy is governed by the \emph{output} error $\|(\bW - \bU_r\bV_r)\mathbf{X}\|_F$ on representative activations $\mathbf{X}$, and because FM activations are correlated and contain outlier channels, the smallest singular values of $\bW$ need not correspond to the least important output directions. 
Truncating them discards activation-critical components and renders the compressed model unusable. 
Activation-aware and truncation-aware variants \cite{yuan2023asvd,wang2024svdllm} resolve this by reweighting the decomposition with an invertible scaling $\mathbf{S} \in \mathbb{R}^{r \times r}$ (an activation-magnitude diagonal or a whitening matrix from the activation covariance), so that singular-value truncation in the scaled coordinates aligns with the output error:
\begin{equation}
    \bW \;\approx\;
    \underbrace{(\bU_r^{\circ}\,\bLambda_r^{1/2}\,\mathbf{S}^{-1})}_{\bU_r}
    \underbrace{(\mathbf{S}\,\bLambda_r^{1/2}\,\bV_r^{\circ})}_{\bV_r},
    \label{eq:scaled_factorization}
\end{equation}
with the plain SVD recovered at $\mathbf{S} = \mathbf{I}_r$. In either case the deployed factors equal the orthonormal singular vectors $\bU_r^{\circ}, \bV_r^{\circ}$ up to the invertible scalings $\bLambda_r^{1/2}$ and $\mathbf{S}$, and are therefore non-orthonormal; our analysis absorbs this spectral scaling into a single effective smoothness constant.

\begin{definition}[Effective smoothness constant]
\label{def:eff_const}
Let $\kappa(\mathbf{S}) \triangleq \|\mathbf{S}\|_2\,\|\mathbf{S}^{-1}\|_2$ denote the condition number of $\mathbf{S}$ in \eqref{eq:scaled_factorization}; set $\kappa(\mathbf{S}) = 1$ for the plain factorization~\eqref{eq:balanced_factorization}. Writing $\lambda_1$ for the largest retained singular value, and given a task-loss smoothness constant $L$ (Assumption~\ref{ass:standard}\,(a)), we define
\begin{equation}
    L_{\mathrm{eff}}
    \;\triangleq\; L\,\lambda_1^2\,\kappa(\mathbf{S})^{2}.
    \label{eq:effective_constants}
\end{equation}
\end{definition}

The factor $\lambda_1^2\,\kappa(\mathbf{S})^2$ in $L_{\mathrm{eff}}$ is the exact price of optimizing the core $\bSigma$ through the frozen, non-orthonormal factors $(\bU_r, \bV_r)$. Since $f_c(\bSigma) = \mathcal{L}_c(\bU_r\bSigma\bV_r)$ composes the $L$-smooth loss with the linear map $\bSigma \mapsto \bU_r\bSigma\bV_r$, its smoothness constant is amplified by the squared operator norm $\|\bU_r\|_2^2\,\|\bV_r\|_2^2$ of that map. 
For the split~\eqref{eq:balanced_factorization}, the leading singular value $\lambda_1$ enters each factor as $\sqrt{\lambda_1}$, so $\|\bU_r\|_2 = \|\bV_r\|_2 = \sqrt{\lambda_1}$ and the amplification equals $\|\bU_r\|_2^2\,\|\bV_r\|_2^2 = \lambda_1^2$; activation-aware whitening contributes the extra factor $\kappa(\mathbf{S})^2$, which collapses to $1$ for a plain SVD. Crucially, both $\lambda_1 = \|\bW\|_2$ and $\kappa(\mathbf{S})$ are fixed by the one-time SVD of the pretrained weight and remain constant throughout federated training; they neither grow with the communication round nor depend on the optimization trajectory. For pretrained transformer layers the spectral norm $\lambda_1$ is a moderate, layer-dependent constant (weight matrices stay spectrally bounded under standard initialization and normalization), and the whitening conditioning $\kappa(\mathbf{S})$ remains finite (and is regularized to be well-conditioned in practice~\cite{wang2024svdllm}).

\subsection{Convergence Analysis}
\label{subsec:convergence}

We focus on Stage-1 adapter-level aggregation (Algorithm~\ref{alg:sm2lm}), the iterative federated optimization component of FedSLM. Stage~2 is a single-step closed-form weighted average whose approximation quality is characterized by the subspace-alignment theory in Section~\ref{subsec:subspace}.
The non-trivial challenge arises in Stage~1, where, within each compression group $\mathcal{C}_i$, clients optimize in the reparameterized adapter space $\bSigma_{i,c} \in \mathbb{R}^{r_i \times r_i}$: the factorization $\bW_i \approx \bU_{r_i}\bSigma_{i,c}\bV_{r_i}$ distorts the loss landscape through the factor scaling $\lambda_1\,\kappa(\mathbf{S})$, and the analysis must account for how this spectral distortion interacts with FedAvg \cite{mcmahan2017fedavg} and client drift.
Without loss of generality, we fix an arbitrary compression group $\mathcal{C}_i$ operating on $\bW_i \in \mathbb{R}^{m \times n}$. Each client $c \in \mathcal{C}_i$ holds a local dataset $\mathcal{D}_c$ of size $n_c$ (with group mass $n_i = \sum_{c \in \mathcal{C}_i} n_c$) and optimizes its local adapter $\bSigma_{i,c} \in \mathbb{R}^{r_i \times r_i}$ with $\widetilde{\bW}_{i,c}(\bSigma_{i,c}) = \bU_{r_i}\bSigma_{i,c}\bV_{r_i}$ as in~\eqref{eq:reconstruction_notation}. The local adapter loss is $f_c(\bSigma) \triangleq \mathcal{L}_c(\bU_{r_i}\bSigma\bV_{r_i})$ and the group-level adapter loss is $f_i(\bSigma) \triangleq \sum_{c \in \mathcal{C}_i}\frac{n_c}{n_i} f_c(\bSigma)$. We require the following standard assumptions.

\begin{samepage}
\begin{assumption}
\label{ass:standard}
\label{ass:smooth}
\label{ass:unbiased}
\label{ass:variance}
\label{ass:hetero}
The following conditions hold for all clients $c \in \mathcal{C}_i$ and all adapter parameters $\bSigma$:
\begin{enumerate}[leftmargin=2.5em]
\item[\textbf{(a)}] \textbf{$L$-Smoothness.}
For every client $c$, the map $\bW \mapsto \mathcal{L}_c(\bW)$ is differentiable with Lipschitz-continuous gradient:
$\bigl\|\nabla_{\bW}\mathcal{L}_c(\bW) - \nabla_{\bW}\mathcal{L}_c(\bW')\bigr\|_F \leq L\|\bW - \bW'\|_F$ for all $\bW, \bW'$.
\item[\textbf{(b)}] \textbf{Unbiased Stochastic Gradients with Bounded Variance.}
$\mathbb{E}[\widetilde{\mathbf{g}}_c(\bSigma) \mid \bSigma] = \nabla_{\bSigma} f_c(\bSigma)$, and there exists $\sigma^2 \geq 0$ such that
$\mathbb{E}\bigl\|\widetilde{\mathbf{g}}_c(\bSigma) - \nabla_{\bSigma} f_c(\bSigma)\bigr\|_F^2 \leq \sigma^2$.
\item[\textbf{(c)}] \textbf{Bounded Data Heterogeneity.}
There exists $G_i \geq 0$ such that
$\sum_{c \in \mathcal{C}_i}\frac{n_c}{n_i} \bigl\|\nabla_{\bSigma} f_c(\bSigma) - \nabla_{\bSigma} f_i(\bSigma) \bigr\|_F^2 \leq G_i^2$.
\end{enumerate}
\end{assumption}
\end{samepage}

\begin{theorem}[Stage-1 convergence with client drift]
\label{thm:convergence}
Let Assumption~\ref{ass:standard} hold, and let $L_{\mathrm{eff}} = L\,\lambda_1^2\,\kappa(\mathbf{S})^2$ be the effective smoothness constant from Definition~\ref{def:eff_const}. Suppose each client in $\mathcal{C}_i$ performs $E \geq 1$ local SGD steps per round with constant step size $\eta$ satisfying $\eta \leq \dfrac{1}{8\,L_{\mathrm{eff}}\,E}$. After $T_1$ rounds of FedAvg, let $\bSigma_i^{(t)} \in \mathbb{R}^{r_i \times r_i}$ denote the aggregated group adapter at the start of round $t$ (with $\bSigma_i^{(0)} = \mathbf{I}_{r_i}$, the initialization in~\eqref{eq:reconstruction_notation}). The iterates $\{\bSigma_i^{(t)}\}_{t=0}^{T_1}$ produced by Algorithm~1 satisfy
\begin{equation}
\begin{split}
    \frac{1}{T_1}\sum_{t=0}^{T_1-1}
    &\mathbb{E}\bigl\|\nabla f_i(\bSigma_i^{(t)})
    \bigr\|_F^2
    \;\leq\;
    \underbrace{\frac{4\Delta_0}{\eta E\,T_1}}
        _{\text{optimization}}
    + \underbrace{4\,L_{\mathrm{eff}}\,\eta\,\sigma^2}
        _{\text{stochastic}} \\
    &+ \underbrace{8\,L_{\mathrm{eff}}^{\,2}\,\eta^2 E^2
        \bigl(\sigma^2 + 6E\,G_i^2\bigr)}
        _{\text{client drift}}
    + \underbrace{2\,G_i^2}
        _{\text{heterogeneity floor}},
\end{split}
\label{eq:thm1_general}
\end{equation}
where $\Delta_0 \triangleq f_i(\bSigma_i^{(0)}) - \inf_{\bSigma} f_i(\bSigma)$. Choosing $\eta = \dfrac{1}{L_{\mathrm{eff}}\sqrt{T_1 E}}$ (which meets the condition whenever $T_1 \geq 64E$) yields
\begin{equation}
\begin{split}
    &\frac{1}{T_1}\sum_{t=0}^{T_1-1}
    \mathbb{E}\bigl\|\nabla f_i(\bSigma_i^{(t)})
    \bigr\|_F^2
    \\
    =&\; \mathcal{O}\!\left(
        \frac{L_{\mathrm{eff}}\Delta_0 + \sigma^2}
             {\sqrt{T_1 E}}
        \right)
    + \mathcal{O}\!\left(
        \frac{E(\sigma^2 + EG_i^2)}{T_1}
        \right) + 2\,G_i^2.
\end{split}
\label{eq:thm1_rate}
\end{equation}
\end{theorem}

Theorem~\ref{thm:convergence} shows that the adapter-level FedAvg in FedSLM converges at the same $\mathcal{O}(T_1^{-1/2})$ rate as standard FedAvg, up to the spectral penalty $\lambda_1^2\,\kappa(\mathbf{S})^2$ absorbed into $L_{\mathrm{eff}}$. The $\lambda_1^2$ factor is the price of folding the singular values into the frozen factors $(\bU_{r_i}, \bV_{r_i})$ so that $\bU_{r_i}\bV_{r_i}$ reconstructs $\bW_i$ directly; it merely rescales the natural step size $\eta = 1/(L_{\mathrm{eff}}\sqrt{T_1 E})$ and leaves the convergence rate intact. When the server uses a pure SVD, $\kappa(\mathbf{S}) = 1$ and only the $\lambda_1^2$ factor remains; when SVD variants are used with a well-conditioned whitening matrix, the extra penalty $\kappa(\mathbf{S})^2$ is modest.
The client-drift term $\mathcal{O}(E(\sigma^2 + EG_i^2)/T_1)$ is the standard artifact of local SGD~\cite{karimireddy2020scaffold} and shows that increasing $E$ does \emph{not} uniformly improve the rate; the residual $2G_i^2$ is the irreducible heterogeneity floor of group $\mathcal{C}_i$. Since group $\mathcal{C}_i$ was arbitrary, the bound holds for every $i \in [K]$, with the constants $G_i, \Delta_0$ specialized to each group.

\subsection{Subspace Alignment Theory}
\label{subsec:subspace}

A key challenge in FedSLM's two-stage aggregation is that different compression groups operate in subspaces of different ranks: group $\mathcal{C}_i$ with rank $r_i$ can only represent weight matrices within a rank-$r_i$ feasible set. This raises a natural question: when the server reconstructs full-rank weights from heterogeneous groups and averages them in Stage~2, how much error does this cross-group fusion introduce? We answer this in three steps. We first establish that the SVD-derived subspaces possess a nested structure (Proposition~\ref{prop:nesting}). We then derive a general upper bound on the cross-group aggregation error in terms of three interpretable residuals (Theorem~\ref{thm:relaxed_error_bound}), and show how a Polyak--{\L}ojasiewicz inequality and a common-projection-target condition each control one of those residuals (Proposition~\ref{prop:residual_control}). Finally, the closed-form specialization is collected in Theorem~\ref{thm:error-bound}.

\begin{definition}[Compression-group subspaces]
\label{def:subspaces}
For a weight matrix $\bW \in \mathbb{R}^{m \times n}$, let $\mathbf{u}_1, \ldots, \mathbf{u}_{\min(m,n)}$ and $\mathbf{v}_1, \ldots, \mathbf{v}_{\min(m,n)}$ be its left and right singular vectors, ordered by non-increasing singular values $\lambda_1 \geq \lambda_2 \geq \cdots \geq \lambda_{\min(m,n)} \geq 0$. We define the rank-$r$ left and right subspaces of $\mathbb{R}^{m}$ and $\mathbb{R}^{n}$ respectively as
\begin{equation}
\begin{split}
    \mathcal{U}_r &\triangleq \mathrm{span}\{\mathbf{u}_1,
        \ldots, \mathbf{u}_r\},\\
    \mathcal{V}_r &\triangleq \mathrm{span}\{\mathbf{v}_1,
        \ldots, \mathbf{v}_r\}.
\end{split}
\end{equation}
The associated orthogonal projectors are $\mathbf{P}_r \triangleq \bU_r^{\circ}\,\bU_r^{\circ\,\top} \in \mathbb{R}^{m \times m}$ and $\mathbf{Q}_r \triangleq \bV_r^{\circ\,\top}\,\bV_r^{\circ} \in \mathbb{R}^{n \times n}$, where $\bU_r^{\circ}$ and $\bV_r^{\circ}$ are the orthonormal factors from~\eqref{eq:balanced_factorization}--\eqref{eq:scaled_factorization} (related to the deployed $\bU_r, \bV_r$ by the spectral scaling $\bLambda_r^{1/2}$ and, for whitened variants, by $\mathbf{S}$).
\end{definition}

Note that $\mathbf{P}_r$ and $\mathbf{Q}_r$ are defined using the \emph{orthonormal} factors $(\bU_r^{\circ}, \bV_r^{\circ})$ of $\mathcal{U}_r$ and $\mathcal{V}_r$, not the scaled factors $\bU_r, \bV_r$ seen by clients. This is essential because orthogonal projectors depend only on the subspace, not on the choice of basis. The column space of $\bU_r$ coincides with that of $\bU_r^{\circ}$ and the row space of $\bV_r$ coincides with that of $\bV_r^{\circ}$ (as $\mathbf{S}$ is invertible), so $\mathcal{U}_r$ and $\mathcal{V}_r$ are well-defined regardless of the factorization variant used.

\begin{proposition}[Nested subspace structure]
\label{prop:nesting}
For any two compression ratios $r_1 < r_2$, the corresponding subspaces satisfy $\mathcal{U}_{r_1} \subsetneq \mathcal{U}_{r_2}$ and $\mathcal{V}_{r_1} \subsetneq \mathcal{V}_{r_2}$, and the projectors satisfy the absorption identities
\begin{equation}
\begin{split}
    \mathbf{P}_{r_1}\mathbf{P}_{r_2}
    &= \mathbf{P}_{r_2}\mathbf{P}_{r_1}
    = \mathbf{P}_{r_1},\\
    \mathbf{Q}_{r_1}\mathbf{Q}_{r_2}
    &= \mathbf{Q}_{r_2}\mathbf{Q}_{r_1}
    = \mathbf{Q}_{r_1}.   
\end{split}
\end{equation}
\end{proposition}

Proposition~\ref{prop:nesting} establishes that the representation spaces of different compression groups form a nested hierarchy: every low-rank client's feasible set is geometrically contained within that of any higher-rank client. This nesting is a direct consequence of all client factors being derived from the \emph{same} server weight $\bW$ via truncated SVD. 
The absorption identities further imply that projecting onto a lower-rank subspace and then onto a higher-rank one is equivalent to the lower-rank projection alone, which is central to the two-stage aggregation protocol.

We now fix $K$ compression groups with ranks $r_1, \ldots, r_K$ and per-group data masses $n_1, \ldots, n_K$ with $N = \sum_i n_i$. Let $\bSigma_i^{(T_1)} \in \mathbb{R}^{r_i \times r_i}$ denote the adapter of group $i$ after Stage~1, and define the group-specific reconstruction
\begin{equation}
    \bW_i \;\triangleq\;
    \bU_{r_i}\,\bSigma_i^{(T_1)}\,\bV_{r_i},
    \qquad
    \bW_{\mathrm{g}} \;\triangleq\;
    \sum_{i=1}^{K}\frac{n_i}{N}\,\bW_i.
\end{equation}


To obtain a meaningful target for the aggregation error, we fix a reference weight $\widehat{\bW} \in \mathbb{R}^{m \times n}$ against which each group is compared. The reference is a free design parameter; setting $\widehat{\bW} = \bW$ (the pre-trained weight) yields the closed-form Eckart--Young expression of Definition~\ref{def:residual} below, while setting $\widehat{\bW}$ to the centralized minimizer $\arg\min_{\bW}\sum_i(n_i/N)\mathcal{L}_i(\bW)$ recovers the comparison with centralized training. For each group $i \in [K]$, define the geometric projection adapter
\begin{equation}
    \widehat{\bSigma}_i
    \;\triangleq\;
    \bU_{r_i}^{\dagger}\,\widehat{\bW}\,\bV_{r_i}^{\dagger},
    \label{eq:sigma_star_def}
\end{equation}
which by Lemma~\ref{lem:P_Q_pinv} (Appendix~\ref{app:prelim_projector}) satisfies the unconditional identity
\begin{equation}
    \bU_{r_i}\widehat{\bSigma}_i\bV_{r_i}
    \;=\;
    \mathbf{P}_{r_i}\widehat{\bW}\mathbf{Q}_{r_i}.
    \label{eq:pinv_main}
\end{equation}
Here $\mathbf{A}^{\dagger}$ denotes the Moore--Penrose pseudoinverse, which for a column-full-rank matrix $\mathbf{A} \in \mathbb{R}^{m \times r}$ with $m \geq r$ admits the explicit form $\mathbf{A}^{\dagger} = (\mathbf{A}^{\top}\mathbf{A})^{-1}\mathbf{A}^{\top}$, and analogously for a row-full-rank $\mathbf{B} \in \mathbb{R}^{r \times n}$ with $n \geq r$, $\mathbf{B}^{\dagger} = \mathbf{B}^{\top}(\mathbf{B}\mathbf{B}^{\top})^{-1}$; see Appendix~\ref{app:prelim} for a self-contained treatment. The corresponding identities $\bU_{r_i}\bU_{r_i}^{\dagger} = \mathbf{P}_{r_i}$ and $\bV_{r_i}^{\dagger}\bV_{r_i} = \mathbf{Q}_{r_i}$ hold independently of whether the factorization is pure SVD or scaled.
Throughout the remainder of this subsection we let $\bSigma_i^{*}$ denote an element of $\arg\min_{\bSigma} f_i$ (the one visited by Stage-1 SGD if it is non-unique); the bounds below depend on $\bSigma_i^{*}$ only through Frobenius distances, so the choice is immaterial. Define the \emph{misalignment quantity}
\begin{equation}
    \Delta_i(\widehat{\bW})
    \;\triangleq\;
    f_i(\widehat{\bSigma}_i) - f_i^{*}
    \;\geq\; 0,
    \qquad
    f_i^{*} \triangleq \min_{\bSigma} f_i(\bSigma),
    \label{eq:Delta_def_main}
\end{equation}
which measures the loss gap between the geometric projection $\widehat{\bSigma}_i$ and the data-driven minimizer; it is local and data-estimable from each group's training trajectory once Stage-1 has stabilized.

\begin{definition}[Compression residual]
\label{def:residual}
Define $\delta_{r}(\widehat{\bW}) \triangleq \|\widehat{\bW} - \mathbf{P}_{r}\widehat{\bW}\mathbf{Q}_{r}\|_F$. When $\widehat{\bW} = \bW$ (the pre-trained layer itself), the Eckart--Young--Mirsky theorem gives the closed form $\delta_{r}(\bW) = \bigl(\sum_{k > r}\lambda_k^2\bigr)^{1/2}$.
\end{definition}

We organize the cross-group aggregation error into three interpretable residuals, each capturing a distinct source of approximation:
\begin{align*}
    \varepsilon_i^{\mathrm{opt}}
    &\triangleq
    \bU_{r_i}\!\bigl(\bSigma_i^{(T_1)} - \bSigma_i^{*}\bigr)\bV_{r_i}
    && \text{(optimization residual),} \\
    \varepsilon_i^{\mathrm{mis}}
    &\triangleq
    \bU_{r_i}\!\bigl(\bSigma_i^{*} - \widehat{\bSigma}_i\bigr)\bV_{r_i}
    && \text{(misalignment residual),} \\
    \varepsilon_i^{\mathrm{sub}}
    &\triangleq
    \mathbf{P}_{r_i}\widehat{\bW}\mathbf{Q}_{r_i} - \widehat{\bW}
    && \text{(subspace residual),}
\end{align*}
so that $\bW_i - \widehat{\bW} = \varepsilon_i^{\mathrm{opt}} + \varepsilon_i^{\mathrm{mis}} + \varepsilon_i^{\mathrm{sub}}$ by direct telescoping.

We now state the general cross-group aggregation bound of the three residuals. The proof is given in Appendix~\ref{app:relaxation}.

\begin{theorem}[Cross-group aggregation error]
\label{thm:relaxed_error_bound}
Let Assumption~\ref{ass:standard} hold. For any reference weight $\widehat{\bW} \in \mathbb{R}^{m \times n}$,
\begin{equation}
\begin{split}
    \mathbb{E}\bigl\|\bW_{\mathrm{g}} - \widehat{\bW}\bigr\|_F
    \;\leq\;
    \sum_{i=1}^{K}\frac{n_i}{N}\!
    \Bigl[\,
        &\mathbb{E}\|\varepsilon_i^{\mathrm{opt}}\|_F
        + \mathbb{E}\|\varepsilon_i^{\mathrm{mis}}\|_F \\
        +\;
        &\delta_{r_i}(\widehat{\bW})
    \Bigr].
\end{split}
\label{eq:thm_relaxed_full}
\end{equation}
The bound separates the cross-group aggregation error into one term per residual, weighted by data mass and aggregated by triangle inequality.
\end{theorem}

The three terms in~\eqref{eq:thm_relaxed_full} are conceptually orthogonal: $\varepsilon_i^{\mathrm{opt}}$ measures how far Stage-1 SGD has progressed, $\varepsilon_i^{\mathrm{mis}}$ measures whether the data-driven minimizer aligns with the geometric projection of $\widehat{\bW}$, and $\varepsilon_i^{\mathrm{sub}}$ measures the intrinsic cost of compressing $\widehat{\bW}$ to rank $r_i$. The next proposition shows how two standard regularity conditions sharpen the first two residuals: a Polyak--{\L}ojasiewicz inequality drives $\varepsilon_i^{\mathrm{opt}}$ to zero with the number of rounds, and a common-projection-target condition drives $\varepsilon_i^{\mathrm{mis}}$ to zero outright.

\begin{proposition}[Residual control]
\label{prop:residual_control}
Fix any reference weight $\widehat{\bW} \in \mathbb{R}^{m \times n}$.
\begin{enumerate}[leftmargin=2.5em]
\item[\textbf{(a)}] \textbf{(P{\L} condition $\Rightarrow$ vanishing optimization residual.)} If, for every $i \in [K]$, the adapter loss $f_i$ satisfies the Polyak--{\L}ojasiewicz inequality
\begin{equation}
    \hspace{-1em}\bigl\|\nabla_{\bSigma} f_i(\bSigma)\bigr\|_F^{\,2}
    \;\geq\;
    2\mu\bigl(f_i(\bSigma) - f_i^{*}\bigr),
    \, \forall\bSigma \in \mathbb{R}^{r_i \times r_i},
    \label{eq:pl}
\end{equation}
with constant $\mu > 0$, then under Assumption~\ref{ass:standard} the optimization residual obeys
\begin{equation}
    \mathbb{E}\|\varepsilon_i^{\mathrm{opt}}\|_F
    \;\leq\;
    \lambda_1\,\kappa(\mathbf{S})\!\cdot\!
    \Bigl(\mathcal{O}(T_1^{-1/4}) + \mathcal{O}(G/\sqrt{\mu})\Bigr),
    \label{eq:opt_pl_rate}
\end{equation}
and in particular $\mathbb{E}\|\varepsilon_i^{\mathrm{opt}}\|_F \to 0$ as $T_1 \to \infty$ in the IID regime ($G = 0$).
\item[\textbf{(b)}] \textbf{(Common projection target $\Rightarrow$ zero misalignment residual.)} If, for every $i \in [K]$, the geometric projection adapter $\widehat{\bSigma}_i$ defined in~\eqref{eq:sigma_star_def} satisfies
\begin{equation}
    \widehat{\bSigma}_i
    \;\in\; \arg\min_{\bSigma} f_i,
\label{eq:sigma_star_min}
\end{equation}
then $\widehat{\bSigma}_i = \bSigma_i^{*}$ and $\varepsilon_i^{\mathrm{mis}} = \mathbf{0}$ for every $i$ (the special case $\Delta_i(\widehat{\bW}) = 0$ of the bound below). Without assuming~\eqref{eq:sigma_star_min}, if part~(a) holds, then
\begin{equation}
    \mathbb{E}\|\varepsilon_i^{\mathrm{mis}}\|_F
    \;\leq\;
    \lambda_1\,\kappa(\mathbf{S})\,\sqrt{\tfrac{2\,\Delta_i(\widehat{\bW})}{\mu}}.
    \label{eq:mis_pl_bound}
\end{equation}
\end{enumerate}
\end{proposition}

The P{\L} inequality is a standard regularity assumption~\cite{karimi2016linear} satisfied by strong convexity, overparameterized neural losses near interpolating minima~\cite{liu2022loss}, low-rank matrix recovery, and many compositional objectives; P{\L}-based analyses underlie much of FL theory~\cite{haddadpour2019convergence,Ying2025Exact,Sun2023Decentralized}. The common-projection-target condition~\eqref{eq:sigma_star_min} is a non-trivial geometric requirement: it asks that the rank-$r_i$ minimizer of the \emph{loss} coincide with the \emph{Frobenius} projection of $\widehat{\bW}$ onto the SVD subspace.
To expose the theorem's core mechanism cleanly, Theorem~\ref{thm:error-bound} below analyzes the case in which the common-projection-target condition~\eqref{eq:sigma_star_min} holds, so $\varepsilon_i^{\mathrm{mis}} = \mathbf{0}$ and the aggregation error reduces to pure subspace geometry. When the condition fails, $\varepsilon_i^{\mathrm{mis}}$ is non-zero but admits explicit bounds via either a curvature-side or a data-side analysis; the analytical structure is the same, and the full treatment is given in Appendix~\ref{app:relaxation}.

\begin{theorem}[Common-projection-target specialization]
\label{thm:error-bound}
Suppose Assumption~\ref{ass:standard} and Proposition~\ref{prop:residual_control}\,(b) both hold, and the optimization residual is negligible. Then the global reconstructed weight satisfies
\begin{equation}
    \bigl\|\bW_{\mathrm{g}} - \widehat{\bW}\bigr\|_F
    \;\leq\;
    \sum_{i=1}^{K} \frac{n_i}{N}\,\delta_{r_i}(\widehat{\bW})
    \;\leq\;
    \delta_{r_{\min}}(\widehat{\bW}),
    \label{eq:thm2_tight}
\end{equation}
where $r_{\min} = \min_i r_i$. If, in addition, $\widehat{\bW}$ coincides with the pre-trained weight~$\bW$, then
\begin{equation}
    \bigl\|\bW_{\mathrm{g}} - \widehat{\bW}\bigr\|_F
    \;\leq\;
    \sum_{i=1}^{K} \frac{n_i}{N}\,
    \biggl(\sum_{k > r_i}\lambda_k^{2}\biggr)^{\!1/2}.
    \label{eq:thm2_ey}
\end{equation}
\end{theorem}

Theorem~\ref{thm:error-bound} is Theorem~\ref{thm:relaxed_error_bound} specialized to $\varepsilon_i^{\mathrm{opt}} \approx \mathbf{0}$ (vanishing rounds) and $\varepsilon_i^{\mathrm{mis}} = \mathbf{0}$ (Proposition~\ref{prop:residual_control}\,(b)): the bound collapses to the data-weighted average of subspace residuals. The \emph{operational} bound is therefore $\sum_i(n_i/N)\,\delta_{r_i}(\widehat{\bW})$; the loose majorant $\delta_{r_{\min}}(\widehat{\bW})$ should be read as a \emph{worst-case} guarantee, not a design target. Allocating more data to higher-rank groups reduces the operational bound even when $r_{\min}$ is fixed, and whenever $\widehat{\bW} \approx \bW$, the server can \emph{pre-compute} the subspace residual from the singular spectrum of~$\bW$ before any training occurs, enabling spectrum-informed rank allocation.
\subsection{Artifact Mitigation via Confidence Loss}
\label{subsec:w2s}

The weak-to-strong refinement stage (Phase~3) poses a subtle risk: when the strong model is supervised by the aggregated weak model, it may not only acquire the federated domain knowledge but also replicate the systematic errors introduced by SVD truncation~\cite{burns2024weak}.  The auxiliary confidence loss $\mathcal{L}_{\mathrm{conf}}$ (Eq.~\eqref{eq:conf_loss}) is designed to mitigate this effect.  We analyze its role in two steps: first, we derive a gradient-level decomposition that exposes how the compression artifact enters the update of the strong model (Proposition~\ref{prop:attenuation}); second, we show that the mixing coefficient $\alpha$ can be tuned to minimize an explicit bias--variance trade-off (Proposition~\ref{prop:bias-var}).

\begin{definition}[Ideal vs.\ actual weak predictions]
\label{def:artifact}
Let $\mathcal{Y}$ denote the label set (vocabulary for language models) with $C = |\mathcal{Y}|$ classes, and let $\Delta^{C-1}$ denote the $(C\!-\!1)$-dimensional probability simplex. For any input $x$, write $\mathbf{z}(x;\,\bW) \in \mathbb{R}^{C}$ for the pre-softmax logit vector produced by the model with weight configuration $\bW$. Define the \emph{ideal} weak prediction as the output the aggregated model would produce if $\bW_{\mathrm{g}}$ were replaced by the reference weight $\widehat{\bW}$ from Section~\ref{subsec:subspace}:
\[
    f^{*}(x) \;\triangleq\; \mathrm{softmax}\!\bigl(\mathbf{z}(x;\,\widehat{\bW})\bigr) \;\in\; \Delta^{C-1}.
\]
The weak prediction then admits the additive decomposition
\begin{equation}
    f_{\mathrm{w}}(x;\,\bW_{\mathrm{g}})
    \;=\;
    f^{*}(x) + \xi(x),
    \label{eq:artifact_decomp}
\end{equation}
where $\xi(x) \in \mathbb{R}^{C}$ is the compression artifact. Since $f_{\mathrm{w}}$ and $f^{*}$ are both valid probability distributions (i.e., their entries are non-negative and sum to one), the entries of $\xi(x)$ must sum to zero: $\mathbf{1}^{\top}\xi(x) = 0$.
\end{definition}

Intuitively, $f^{*}(x)$ represents the ``clean'' prediction that the weak model would make if the federated optimization had converged perfectly and no approximation error were introduced by low-rank truncation. In practice, the aggregated global model $\bW_{\mathrm{g}}$ deviates from $\widehat{\bW}$ due to both finite-round optimization error and the intrinsic subspace approximation error analyzed in Theorem~\ref{thm:error-bound}. The artifact $\xi(x)$ captures the combined effect of these two sources of error in the output space: it shifts probability mass across classes in a zero-sum manner, systematically distorting the weak model's predictions. The magnitude of $\xi(x)$ is controlled by the weight-space aggregation error $\|\bW_{\mathrm{g}} - \widehat{\bW}\|_F$ through the Lipschitz forward-pass assumption below, which provides the bridge between the weight-space bounds of Section~\ref{subsec:subspace} and the output-space analysis that follows.

\begin{assumption}[Lipschitz forward pass]
\label{ass:lipschitz}
There exists $L_f \geq 0$ such that for every input $x$ and every pair of parameter configurations $(\bW, \bW')$ differing only in the selected layer,
\begin{equation}
    \bigl\|f(x;\,\bW) - f(x;\,\bW')\bigr\|_1
    \;\leq\;
    L_f\,\|\bW - \bW'\|_F.
\end{equation}
\end{assumption}

Having established that the artifact $\xi(x)$ exists as a well-defined perturbation in the output space, we now quantify its magnitude. The following lemma combines the Lipschitz forward-pass assumption with the weight-space aggregation error from Theorem~\ref{thm:error-bound} to obtain a uniform bound on $\|\xi(x)\|_1$ across all inputs.

\begin{lemma}[Artifact magnitude bound]
\label{lem:artifact_bound}
Under Assumption~\ref{ass:lipschitz} and Theorem~\ref{thm:error-bound}, for $\forall\, x$
\begin{equation}
\begin{split}
    \|\xi(x)\|_1
    &\;\leq\;
    L_f\!\cdot\!
    \sum_{i=1}^{K}\frac{n_i}{N}\,\delta_{r_i}(\widehat{\bW})
    \\
    &\;\leq\;
    L_f\,
    \delta_{r_{\min}}(\widehat{\bW})
    \;=:\;
    B_\xi.
\end{split}
\end{equation}
\end{lemma}

The bound $B_\xi$ is determined entirely by the compression configuration: it grows with the Lipschitz constant $L_f$ of the forward pass and with the tail singular-value energy $\delta_{r_{\min}}$ of the most aggressively compressed group. This quantity serves as the noise budget in the subsequent gradient-level analysis.

We now turn to the central question: how does the auxiliary confidence loss $\mathcal{L}_{\mathrm{conf}}$ interact with this artifact when updating the strong model? Recall from Eq.~\eqref{eq:conf_loss} that the loss combines imitation of the weak model and self-anchoring of the strong model's own confident predictions.
The next proposition decomposes the logit-level gradient of this loss into three interpretable components, revealing where the compression artifact enters and how $\alpha$ modulates its influence.

\begin{proposition}[Gradient-level artifact attenuation]
\label{prop:attenuation}
Let $\mathbf{z}_s(x) \in \mathbb{R}^{C}$ denote the pre-softmax logits of the strong model. The logit gradient of $\mathcal{L}_{\mathrm{conf}}$ admits the following exact decomposition:
\begin{equation}
\begin{split}
    \nabla_{\mathbf{z}_s}\mathcal{L}_{\mathrm{conf}}
    =\;&
    \underbrace{(1\!-\!\alpha)(f_{\mathrm{s}}
        - f^{*})}_{\text{(I) clean signal}} \\
    &-\underbrace{(1\!-\!\alpha)\,\xi(x)}
        _{\text{(II) attenuated artifact}}
    +\underbrace{\alpha(f_{\mathrm{s}}
        - \hat{f}_{\mathrm{s}})}
        _{\text{(III) self-anchoring}}.
\end{split}
    \label{eq:grad_decomp}
\end{equation}
Under Assumption~\ref{ass:lipschitz}, the artifact term~(II) is bounded by
\begin{equation}
    \bigl\|(1\!-\!\alpha)\,\xi(x)\bigr\|_2
    \;\leq\;
    (1\!-\!\alpha)\,L_f\,
        \delta_{r_{\min}}(\widehat{\bW}).
    \label{eq:attenuation}
\end{equation}
\end{proposition}

The decomposition~\eqref{eq:grad_decomp} reveals three forces acting on the strong model at each gradient step. Term~(I) drives the strong model toward the ideal prediction $f^*$, which is the useful federated knowledge. 
Term~(II) is the compression artifact that pushes the model in a spurious direction; its magnitude is scaled down by $(1-\alpha)$. 
Term~(III) is the self-anchoring regularizer that pulls the strong model toward its own high-confidence predictions, providing an independent reference signal that does not depend on the noisy weak supervisor. 
The factor $(1-\alpha)$ in~\eqref{eq:attenuation} is an algebraic consequence of the convex combination defining $\mathcal{L}_{\mathrm{conf}}$: the same $(1-\alpha)$ scales the \emph{clean} signal~(I). Proposition~\ref{prop:attenuation} therefore does \emph{not} claim that the confidence loss selectively filters artifacts.  Its role is to make the bias--variance trade-off explicit and to \emph{couple} the weight-space error of Theorem~\ref{thm:error-bound} with the output-space error of the refined strong model.  A genuine noise-suppression argument must rely on the self-anchoring term~(III), which is analyzed in Proposition~\ref{prop:bias-var} below.

\begin{proposition}[Bias--variance trade-off of $\alpha$]
\label{prop:bias-var}
Suppose that on a random sample $x \sim \mathcal{D}_{\mathrm{server}}$, the self-anchoring direction $f_{\mathrm{s}} - \hat{f}_{\mathrm{s}}$ has bounded second moment $\mathbb{E}\|f_{\mathrm{s}} - \hat{f}_{\mathrm{s}}\|_2^2 \leq V_{\mathrm{self}}$, and that the artifact $\xi(x)$ and the self-anchoring residual $f_{\mathrm{s}} - \hat{f}_{\mathrm{s}}$ are uncorrelated, i.e., $\mathbb{E}\langle \xi(x),\, f_{\mathrm{s}} - \hat{f}_{\mathrm{s}}\rangle = 0$. We define the excess risk of the logit-gradient estimator against the clean target $f_{\mathrm{s}} - f^{*}$ as follows:
\begin{equation}
    \mathcal{R}(\alpha)
    \;\triangleq\;
    \underbrace{(1-\alpha)^2\,
        \mathbb{E}\|\xi(x)\|_2^2}
        _{\text{bias from artifact}}
    \;+\;
    \underbrace{\alpha^2\,V_{\mathrm{self}}}
        _{\text{variance from self-anchor}}.
\end{equation}
Then $\mathcal{R}$ is minimized at
\begin{equation}
    \alpha^{*}
    \;=\;
    \frac{\mathbb{E}\|\xi(x)\|_2^2}
         {V_{\mathrm{self}}
            + \mathbb{E}\|\xi(x)\|_2^2}
    \;\leq\;
    \frac{B_\xi^{\,2}}
         {V_{\mathrm{self}} + B_\xi^{\,2}},
    \label{eq:alpha_star}
\end{equation}
and $\alpha^{*}$ is monotone non-decreasing in $B_\xi$: heavier compression calls for a larger mixing weight on the self-anchor.
\end{proposition}

Propositions~\ref{prop:attenuation} and~\ref{prop:bias-var} together furnish an end-to-end guarantee: Theorem~\ref{thm:error-bound} bounds the \emph{weight-space} aggregation error, Lemma~\ref{lem:artifact_bound} transports this bound into \emph{output space}, and Proposition~\ref{prop:bias-var} prescribes how to trade off the resulting bias against self-anchor variance via $\alpha$. The monotone dependence of $\alpha^{*}$ on $B_\xi$ yields a spectrum-based heuristic: as the compression ratio increases and tail singular-value energy decreases, $B_\xi \to 0$ and $\alpha^{*} \to 0$, so the strong model should defer more to the weak supervisor. Conversely, at high compression, $\alpha^{*} \to 1$ and the strong model should rely on its own confident predictions.

\section{Experiments}

\noindent\textbf{Goals.} Our experiments evaluate FedSLM along four axes: (1)~server-side performance against existing federated baselines on language and multimodal tasks; (2)~client-side performance of SVD-compressed models; (3)~convergence behavior across heterogeneous compression ratios; and (4)~the effect of adapter placement strategy on downstream performance.

\noindent\textbf{Models.} We evaluate FedSLM and all baselines using LLaMA-2-7B and LLaMA-2-13B~\cite{touvron2023llama} as the server-side language models and LLaVA-NeXT-7B~\cite{liu2024llavanext} as the server-side multimodal model. 
Client-side compressed models are obtained by decomposing the pre-trained server weights into fixed low-rank factors $\bU$ and $\bV$ shared within each compression group. 
Specifically, we use DobiSVD~\cite{wang2024dobisvd} at compression ratios 0.4\footnote{\url{https://huggingface.co/Qinsi1/DobiSVD-Llama-2-7b-hf-0.4}} and 0.6\footnote{\url{https://huggingface.co/Qinsi1/DobiSVD-Llama-2-7b-hf-0.6}} for LLaMA-2-7B, DobiSVD at the same ratios 0.4\footnote{\url{https://huggingface.co/Qinsi1/DobiSVD-Llama-2-13b-hf-0.4}} and 0.6\footnote{\url{https://huggingface.co/Qinsi1/DobiSVD-Llama-2-13b-hf-0.6}} for LLaMA-2-13B, and QSVD~\cite{qsvd2025} at compression ratios 0.6 and 0.9\footnote{\url{https://github.com/SAI-Lab-NYU/QSVD}} for LLaVA-NeXT-7B.

For SVD-compressed clients (DobiSVD and QSVD), trainable adapters are inserted only into the three MLP projections (\texttt{gate\_proj}, \texttt{up\_proj}, \texttt{down\_proj}). 
Each adapter is a square matrix $\bSigma$ placed between the two frozen SVD factors, with $\bSigma$ initialized to the identity; the SVD-decomposed attention layers remain fully frozen. 
For full-model clients (LLaMA-2-7B, LLaMA-2-13B, and LLaVA-NeXT-7B), we apply LoRA~\cite{hu2022lora} adapters to all linear projections per transformer block: the attention projections (\texttt{q\_proj}, \texttt{k\_proj}, \texttt{v\_proj}, \texttt{o\_proj}) and the MLP projections above.

\noindent\textbf{Datasets.} For training, we use AI2 Reasoning Challenge (ARC-Easy and ARC-Challenge)~\cite{clark2018arc}, PIQA~\cite{bisk2020piqa}, WinoGrande~\cite{sakaguchi2021winogrande}, Social IQA~\cite{sap2019socialiqa}, HellaSwag~\cite{zellers2019hellaswag}, and COPA~\cite{roemmele2011choice}, together with Medical-Flashcards~\cite{medicalflashcards} as a domain-specific medical corpus. For medical evaluation, we report zero-shot transfer accuracy on PubMedQA~\cite{jin2019pubmedqa} and MedMCQA~\cite{pal2022medmcqa}. For multimodal evaluation, we use ScienceQA~\cite{lu2022scienceqa} and VizWiz~\cite{gurari2018vizwiz}.

\noindent\textbf{Baselines.} We compare against eight federated learning methods: (1)~FedAvg+LoRA~\cite{mcmahan2017fedavg,hu2022lora}, which applies standard federated averaging to LoRA adapters; (2)~FFA-LoRA~\cite{sun2024ffalora}, a frozen-weight variant of federated LoRA; (3)~HetLoRA~\cite{cho2024hetlora}, which supports heterogeneous LoRA ranks across clients; (4)~FlexLoRA~\cite{bai2024flexlora}, which dynamically allocates LoRA ranks based on client capacity; (5)~Fed-RAC-LoRA~\cite{malinovsky2024randomized}, which randomly freezes one LoRA factor per step and iteratively merges the update into the base weights to form a convergent chain; (6)~FedMKT~\cite{fan2024fedmkt}, a mutual knowledge transfer approach for heterogeneous models; (7)~FedProto~\cite{tan2022fedproto}, which exchanges class prototypes instead of gradients to tolerate client heterogeneity; and (8)~FedBiOT~\cite{wu2024fedbiot}, which has clients fine-tune a lightweight adapter on a server-compressed LLM via bi-level optimization. 
We also report two ablated variants of our method: FedSLM~($\alpha\!=\!0$), the server model refined via weak-to-strong elicitation using only the weak-imitation loss; and FedSLM~($\alpha\!=\!0.5$), which additionally activates the auxiliary confidence loss that balances weak supervision with the strong model's confident predictions.

\begin{table*}[t]
  \centering
  \caption{Performance comparison on natural language understanding benchmarks with LLaMA-2-7B as the server model (accuracy \%). We report results under IID and non-IID (Dirichlet $\beta = 1.0$) data partitions. The best result in each column is \textbf{boldfaced}.}
  \label{tab:nlu}
  \setlength{\tabcolsep}{4pt}
  \setlength{\aboverulesep}{0pt}
  \setlength{\belowrulesep}{0pt}
  \renewcommand{\arraystretch}{1.15}
  \resizebox{\textwidth}{!}{%
  \begin{tabular}{cl ccccccc cc c}
  \toprule
  \rowcolor{gray!15}
  \textbf{Partition} & \textbf{Methods} & \textbf{ARC\_e} & \textbf{ARC\_c} & \textbf{COPA} & \textbf{HellaS} & \textbf{PIQA} & \textbf{Social} & \textbf{WinoG} & \textbf{PubMed} & \textbf{MedMCQA} & \textbf{Avg} \\
  \midrule
  \multirow{10}{*}{IID}
  & FedAvg+LoRA & 80.0 & 40.8 & 87.0 & 60.5 & 80.2 & 58.8 & 68.4 & 68.1 & 32.6 & 64.0 \\
  & FFA-LoRA & 80.2 & 40.8 & 87.0 & 60.3 & 80.4 & 58.5 & 69.1 & 67.9 & 32.4 & 64.1 \\
  & HetLoRA & 80.6 & 41.2 & 87.0 & 60.6 & 79.0 & 59.4 & 69.8 & 68.2 & 32.4 & 64.2 \\
  & FlexLoRA & 80.1 & 41.4 & 87.0 & 60.8 & 80.4 & 60.9 & 69.5 & 68.8 & 32.8 & 64.6 \\
  & Fed-RAC-LoRA & 78.2 & 40.5 & 88.0 & 56.7 & 78.5 & 52.7 & 53.2 & 64.2 & 30.1 & 60.2 \\
  & FedMKT & 81.4 & 52.2 & 89.0 & 60.7 & 80.6 & 59.6 & 74.3 & 70.0 & 34.5 & 66.9 \\
  & FedProto & 72.6 & 36.1 & 80.0 & 56.7 & 77.0 & 56.7 & 74.9 & 61.6 & 27.3 & 60.3 \\
  & FedBiOT & 73.8 & 37.7 & 85.0 & 53.2 & 77.2 & 56.9 & 66.8 & 65.7 & 29.1 & 60.6 \\
  & FedSLM ($\alpha\!=\!0$) & 85.3 & 53.2 & 90.0 & 75.1 & 80.7 & \textbf{64.7} & 74.0 & 71.3 & 35.0 & 69.9 \\
  & FedSLM ($\alpha\!=\!0.5$) & \textbf{85.6} & \textbf{54.8} & \textbf{91.0} & \textbf{80.2} & \textbf{81.1} & \textbf{64.7} & 74.6 & \textbf{71.5} & \textbf{35.1} & \textbf{71.0} \\
  \midrule
  \multirow{10}{*}{\shortstack{Non-IID}}
  & FedAvg+LoRA & 80.4 & 40.1 & 87.0 & 60.4 & 79.9 & 59.1 & 53.0 & 67.8 & 32.4 & 62.2 \\
  & FFA-LoRA & 80.0 & 40.8 & 87.0 & 60.5 & 80.1 & 58.8 & 53.0 & 67.4 & 31.6 & 62.1 \\
  & HetLoRA & 80.8 & 40.5 & 87.0 & 60.3 & 78.2 & 58.8 & 52.1 & 67.2 & 32.5 & 61.9 \\
  & FlexLoRA & 80.4 & 40.2 & 86.0 & 60.7 & 78.6 & 59.2 & 52.4 & 68.4 & 32.9 & 62.1 \\
  & Fed-RAC-LoRA & 78.1 & 41.1 & 87.0 & 56.7 & 78.6 & 51.8 & 51.7 & 63.3 & 29.7 & 59.8 \\
  & FedMKT & 81.2 & 52.8 & 88.0 & 60.1 & 80.3 & 58.5 & \textbf{72.7} & 69.5 & 34.2 & 66.4 \\
  & FedProto & 71.8 & 36.5 & 84.0 & 56.7 & 76.9 & 57.2 & 56.3 & 61.2 & 27.1 & 58.6 \\
  & FedBiOT & 73.9 & 37.5 & 85.0 & 53.0 & 77.4 & 57.1 & 67.4 & 64.1 & 29.0 & 60.5 \\
  & FedSLM ($\alpha\!=\!0$) & 82.5 & 54.8 & \textbf{90.0} & \textbf{76.6} & \textbf{82.2} & 60.7 & 57.9 & 71.0 & \textbf{34.8} & 67.8 \\
  & FedSLM ($\alpha\!=\!0.5$) & \textbf{83.5} & \textbf{57.2} & \textbf{90.0} & 76.0 & 80.7 & \textbf{62.1} & 58.4 & \textbf{71.2} & \textbf{34.8} & \textbf{68.2} \\
  \bottomrule
  \end{tabular}%
  }
\end{table*}
\begin{figure*}[t]
  \centering
  \includegraphics[width=\linewidth]{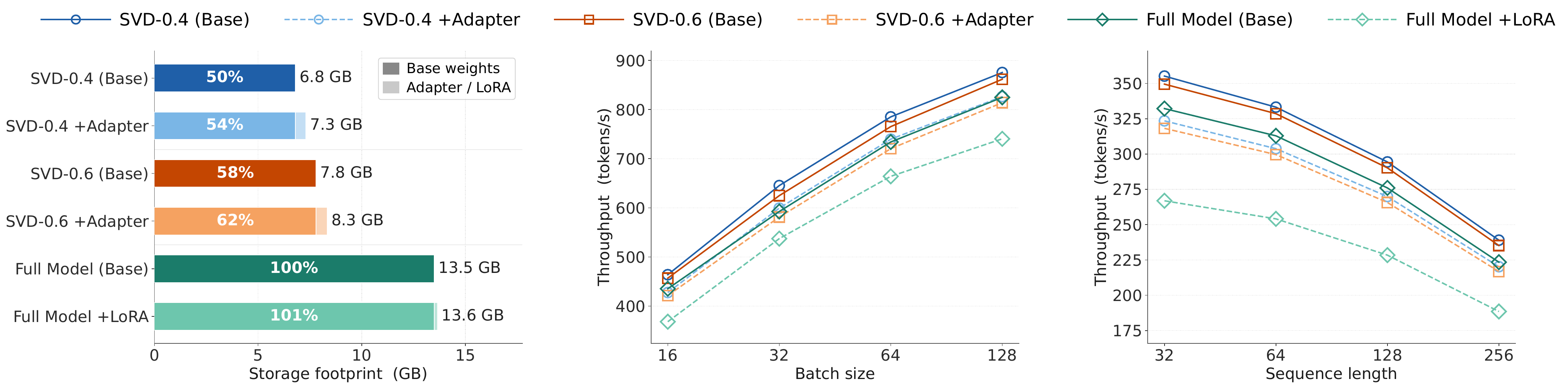}
  \vspace{-2em}
  \caption{Client-side efficiency of SVD-compressed LLaMA-2-7B models, with a Full+LoRA reference. Left: GPU memory footprint before and after attaching adapters (or LoRA modules for the full model). Middle: inference throughput versus batch size (input length 128, 32 generated tokens). Right: inference throughput versus input sequence length (batch size 8, 32 generated tokens).}
  \label{fig:efficiency}
\end{figure*}

\noindent\textbf{Federated Setup.} We partition training data across clients under two settings: IID (uniform) and non-IID (Dirichlet $\beta = 1.0$). 
For commonsense reasoning tasks, we simulate a federation of 10 clients comprising three compression groups (SVD-0.4$\times$3, SVD-0.6$\times$4, full server model $\times$3) and run $T_1=5$ rounds of adapter-level aggregation. For the medical task, we scale to 20 clients (SVD-0.4$\times$7, SVD-0.6$\times$7, full server model $\times$6) with $T_1=10$ rounds. 
For multimodal tasks, we deploy a federation of 10 clients with three QSVD compression groups (QSVD-0.6$\times$3, QSVD-0.9$\times$4, full LLaVA-NeXT-7B$\times$3) under the same $T_1=5$ rounds of adapter-level aggregation. The heterogeneous client composition above is shared by FedSLM and the heterogeneity-aware baselines FedMKT, FedProto, and FedBiOT. The remaining baselines assume a homogeneous client pool by design; for fairness, we evaluate them on the subset of clients that natively host the full server model, with all other hyperparameters (number of clients, data partition, and training rounds) kept consistent. In all settings, weak-to-strong elicitation runs for $T_2=5$ epochs.

\begin{table*}[t]
  \centering
  \caption{Performance comparison on natural language understanding benchmarks with LLaMA-2-13B as the server model (accuracy \%). We report results under IID and non-IID (Dirichlet $\beta = 1.0$) data partitions. The best result in each column is \textbf{boldfaced}.}
  \label{tab:nlu_13b}
  \setlength{\tabcolsep}{4pt}
  \setlength{\aboverulesep}{0pt}
  \setlength{\belowrulesep}{0pt}
  \renewcommand{\arraystretch}{1.15}
  \resizebox{\textwidth}{!}{%
  \begin{tabular}{cl ccccccc cc c}
  \toprule
  \rowcolor{gray!15}
  \textbf{Partition} & \textbf{Methods} & \textbf{ARC\_e} & \textbf{ARC\_c} & \textbf{COPA} & \textbf{HellaS} & \textbf{PIQA} & \textbf{Social} & \textbf{WinoG} & \textbf{PubMed} & \textbf{MedMCQA} & \textbf{Avg} \\
  \midrule
  \multirow{10}{*}{IID}
  & FedAvg+LoRA & 83.5 & 43.2 & 88.0 & 62.8 & 81.5 & 61.0 & 70.2 & 70.5 & 34.0 & 66.1 \\
  & FFA-LoRA & 83.8 & 43.4 & 88.0 & 62.6 & 81.7 & 60.8 & 70.6 & 70.3 & 33.8 & 66.1 \\
  & HetLoRA & 84.1 & 44.0 & 89.0 & 62.9 & 81.4 & 61.5 & 71.2 & 70.8 & 34.1 & 66.6 \\
  & FlexLoRA & 83.6 & 44.2 & 89.0 & 63.1 & 82.0 & 62.8 & 70.9 & 71.0 & 34.3 & 66.8 \\
  & Fed-RAC-LoRA & 81.8 & 43.0 & 89.0 & 59.2 & 80.3 & 55.1 & 55.8 & 66.4 & 31.5 & 62.5 \\
  & FedMKT & 84.6 & 48.5 & 90.0 & 63.0 & \textbf{82.3} & 61.8 & 76.5 & 72.2 & 36.0 & 68.3 \\
  & FedProto & 76.2 & 38.5 & 82.0 & 59.2 & 79.1 & 59.0 & 76.8 & 64.1 & 28.8 & 62.6 \\
  & FedBiOT & 77.3 & 40.0 & 87.0 & 55.7 & 79.3 & 59.1 & 68.9 & 67.9 & 30.6 & 62.9 \\
  & FedSLM ($\alpha\!=\!0$) & 86.0 & 52.2 & 89.0 & 67.5 & 82.1 & 65.2 & \textbf{83.0} & 73.7 & 36.5 & 70.6 \\
  & FedSLM ($\alpha\!=\!0.5$) & \textbf{87.2} & \textbf{53.2} & \textbf{92.0} & \textbf{68.0} & \textbf{82.3} & \textbf{66.0} & \textbf{83.0} & \textbf{73.9} & \textbf{36.6} & \textbf{71.4} \\
  \midrule
  \multirow{10}{*}{\shortstack{Non-IID}}
  & FedAvg+LoRA & 83.9 & 42.5 & 88.0 & 62.7 & 81.2 & 61.3 & 55.2 & 70.2 & 33.8 & 64.3 \\
  & FFA-LoRA & 83.5 & 43.2 & 88.0 & 62.8 & 81.4 & 61.0 & 55.2 & 69.8 & 33.0 & 64.2 \\
  & HetLoRA & 84.2 & 42.8 & 88.0 & 62.6 & 80.4 & 61.0 & 54.3 & 69.6 & 33.9 & 64.1 \\
  & FlexLoRA & 83.8 & 42.5 & 87.0 & 63.0 & 80.8 & 61.5 & 54.6 & 70.8 & 34.3 & 64.3 \\
  & Fed-RAC-LoRA & 81.7 & 43.6 & 88.0 & 59.2 & 80.4 & 54.2 & 53.9 & 65.5 & 31.1 & 62.0 \\
  & FedMKT & 84.4 & 47.1 & \textbf{89.0} & 62.4 & 82.0 & 60.8 & \textbf{70.2} & 71.8 & 35.7 & 67.0 \\
  & FedProto & 75.4 & 38.9 & 86.0 & 59.2 & 79.0 & 59.5 & 58.5 & 63.7 & 28.6 & 61.0 \\
  & FedBiOT & 77.4 & 40.0 & 87.0 & 55.5 & 79.5 & 59.3 & 69.6 & 66.5 & 30.5 & 62.8 \\
  & FedSLM ($\alpha\!=\!0$) & 85.1 & 51.8 & 87.0 & 67.2 & 82.8 & 63.8 & 59.3 & 73.3 & \textbf{36.3} & 67.4 \\
  & FedSLM ($\alpha\!=\!0.5$) & \textbf{86.1} & \textbf{53.5} & 88.0 & \textbf{67.7} & \textbf{83.0} & \textbf{64.8} & 60.2 & \textbf{73.5} & \textbf{36.3} & \textbf{68.1} \\
  \bottomrule
  \end{tabular}%
  }
\end{table*}

\subsection{Main Results}

Tables~\ref{tab:nlu} and~\ref{tab:nlu_13b} report accuracy on nine benchmarks under IID and non-IID partitions, with LLaMA-2-7B and LLaMA-2-13B serving as the server model respectively. 

On LLaMA-2-7B (Table~\ref{tab:nlu}), FedSLM ($\alpha\!=\!0.5$) achieves 71.0\% average accuracy under IID and 68.2\% under non-IID, surpassing FedMKT (66.9\% / 66.4\%) by 4.1\% / 1.8\% and FedAvg+LoRA (64.0\% / 62.2\%) by 7.0\% / 6.0\%. The gains are pronounced on benchmarks requiring deeper reasoning: under IID, HellaSwag reaches 80.2\% versus 60.5\% for FedAvg+LoRA and ARC\_c reaches 54.8\% versus 40.8\%; under non-IID, ARC\_c reaches 57.2\% versus 52.8\% for FedMKT. On the medical benchmarks, FedSLM ($\alpha\!=\!0.5$) reaches 71.5\% on PubMedQA and 35.1\% on MedMCQA under IID, exceeding FedAvg+LoRA by 3.4\% and 2.5\%.

On LLaMA-2-13B (Table~\ref{tab:nlu_13b}), the trend scales up: FedSLM ($\alpha\!=\!0.5$) attains 71.4\% under IID and 68.1\% under non-IID, outperforming FedMKT (68.3\% / 67.0\%) by 3.1\% / 1.1\% and FedAvg+LoRA (66.1\% / 64.3\%) by 5.3\% / 3.8\%. The advantage on reasoning-heavy benchmarks persists, with HellaSwag reaching 68.0\% under IID versus 62.8\% for FedAvg+LoRA and ARC\_c reaching 53.2\% versus 43.2\%. On the medical benchmarks, FedSLM ($\alpha\!=\!0.5$) reaches 73.9\% on PubMedQA and 36.6\% on MedMCQA under IID, exceeding FedAvg+LoRA by 3.4\% and 2.6\%.

Comparing FedSLM ($\alpha\!=\!0$) with FedSLM ($\alpha\!=\!0.5$) isolates the contribution of the auxiliary confidence loss across both model scales: under IID, average accuracy improves from 69.9\% to 71.0\% on 7B and from 70.6\% to 71.4\% on 13B; under non-IID, from 67.8\% to 68.2\% on 7B and from 67.4\% to 68.1\% on 13B. The confidence loss is most valuable on benchmarks where the weak supervisor is noisier. For example, HellaSwag improves from 75.1\% to 80.2\% under IID on 7B, and ARC\_c improves by 2.4\% on 7B and 1.7\% on 13B under non-IID. This matches the theoretical prediction in Proposition~\ref{prop:attenuation} that the $(1-\alpha)$ factor attenuates compression artifacts when the weak supervisor is noisy. In summary, FedSLM consistently ranks first or second on all nine benchmarks under both data partitions and at both model scales, demonstrating that the two-stage aggregation combined with weak-to-strong elicitation effectively transfers federated knowledge to the server model, with the auxiliary confidence loss providing a reliable additional gain.

\subsection{Efficiency Analysis}

\begin{figure*}[t]
  \centering
  \includegraphics[width=\linewidth]{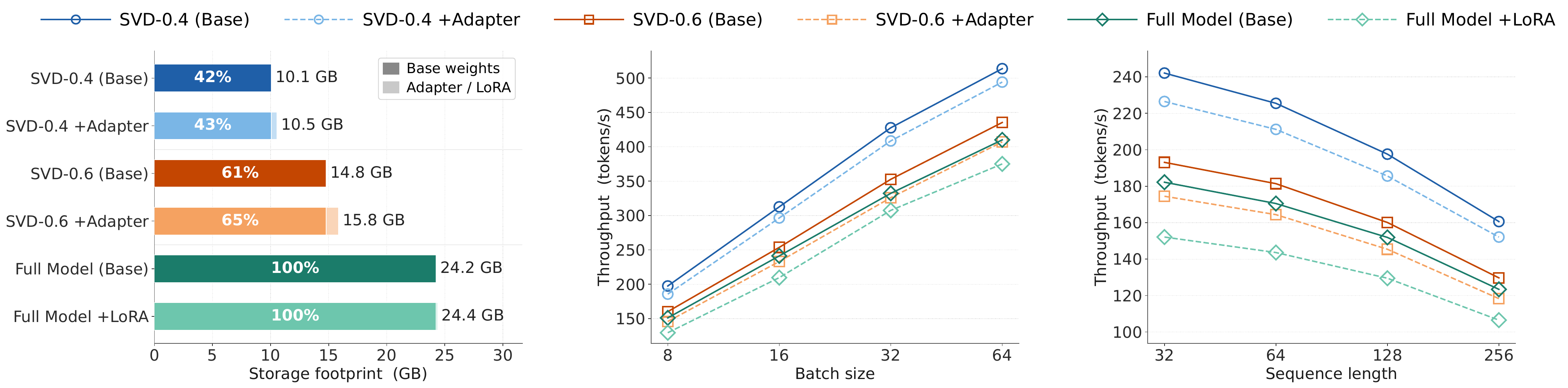}
      \vspace{-2em}
  \caption{Client-side efficiency of SVD-compressed LLaMA-2-13B models, with a Full+LoRA reference. Left: GPU memory footprint before and after attaching adapters (or LoRA modules for the full model). Middle: inference throughput versus batch size (input length 128, 32 generated tokens). Right: inference throughput versus input sequence length (batch size 8, 32 generated tokens).}
  \label{fig:efficiency_13b}
\end{figure*}

A practical federated system must ensure that compressed client models fit within the memory and latency budgets of participating devices. We therefore profile the GPU memory footprint and inference throughput of the compressed models deployed on clients, benchmarked against the respective full models at both 7B and 13B scales (Figure~\ref{fig:efficiency} and Figure~\ref{fig:efficiency_13b}).

Figure~\ref{fig:efficiency} (left) compares the GPU memory occupied by DobiSVD-compressed LLaMA-2-7B models before and after inserting the adapter $\bSigma$. The full LLaMA-2-7B occupies 13.5\,GB, whereas DobiSVD-0.4 and DobiSVD-0.6 require roughly 6.8\,GB and 7.8\,GB respectively. After attaching adapters, the memory overhead is negligible: DobiSVD-0.4+Adapter and DobiSVD-0.6+Adapter remain close to their base counterparts, confirming that the lightweight adapter parameterization adds minimal cost to the compressed model. Overall, FedSLM clients require only about 54\% (0.4-ratio) to 62\% (0.6-ratio) of the GPU memory needed by the full model.

Figure~\ref{fig:efficiency_13b} (left) shows the corresponding memory profile for LLaMA-2-13B, which additionally includes a Full+LoRA reference for comparison against parameter-efficient fine-tuning of the uncompressed model. The full LLaMA-2-13B occupies 24.2\,GB, while DobiSVD-0.4 and DobiSVD-0.6 require only 10.1\,GB (42\%) and 14.8\,GB (61\%) respectively. After attaching the trainable adapter, DobiSVD-0.4+Adapter rises to 10.5\,GB (43\%) and DobiSVD-0.6+Adapter to 15.8\,GB (65\%), an overhead of $\le 1$\,GB in both cases. By contrast, attaching LoRA modules to the full model raises memory only marginally (24.4\,GB), confirming that the dominant memory cost is the base weights themselves.

Figures~\ref{fig:efficiency} and~\ref{fig:efficiency_13b} (middle) show inference throughput as a function of batch size. At both scales, the compressed models consistently achieve higher throughput than the full model across all batch sizes, with the 0.4-ratio model delivering the largest speedup. For LLaMA-2-13B at batch size 64, DobiSVD-0.4 reaches 522 tokens/s and DobiSVD-0.6 reaches 444 tokens/s, compared to 416 tokens/s for the full model and 384 tokens/s for the LoRA-tuned full model---speedups of $1.25\times$ and $1.07\times$ over the Full baseline, and $1.36\times$ and $1.16\times$ over Full+LoRA. The throughput advantage grows with batch size, as the smaller weight matrices benefit from more efficient matrix multiplications. Adapter attachment incurs only a small throughput penalty (e.g., 512$\to$497 tokens/s for SVD-0.4 at batch 64), preserving most of the compression benefit.

Figures~\ref{fig:efficiency} and~\ref{fig:efficiency_13b} (right) vary the input sequence length over $\{32, 64, 128, 256\}$ with batch size fixed at 8. Throughput decreases for all configurations as the sequence length grows, reflecting the quadratic cost of self-attention, but the relative ordering is preserved across all tested lengths: at sequence length 256 on LLaMA-2-13B, DobiSVD-0.4+Adapter sustains 160 tokens/s versus 122 tokens/s for the full model and 108 tokens/s for Full+LoRA, a $1.31\times$ and $1.48\times$ advantage, respectively. 
The gap between Full and Full+LoRA widens slightly at long sequences, indicating that LoRA's overhead becomes more visible when attention dominates the wall-clock budget; the compressed variants are insulated from this effect because their compression also shrinks the projection matrices that drive attention computation.

Taken together, these results confirm that FedSLM clients operate with roughly half the GPU memory of the full model while achieving higher inference throughput at both 7B and 13B scales. Critically, at the 13B scale FedSLM-0.4 outperforms not only the full model but also the parameter-efficient Full+LoRA baseline on both memory and throughput simultaneously, making federated participation feasible on GPUs that cannot host the full model.

\subsection{Convergence Analysis}

\begin{figure*}[t]
  \centering
  \includegraphics[width=\linewidth]{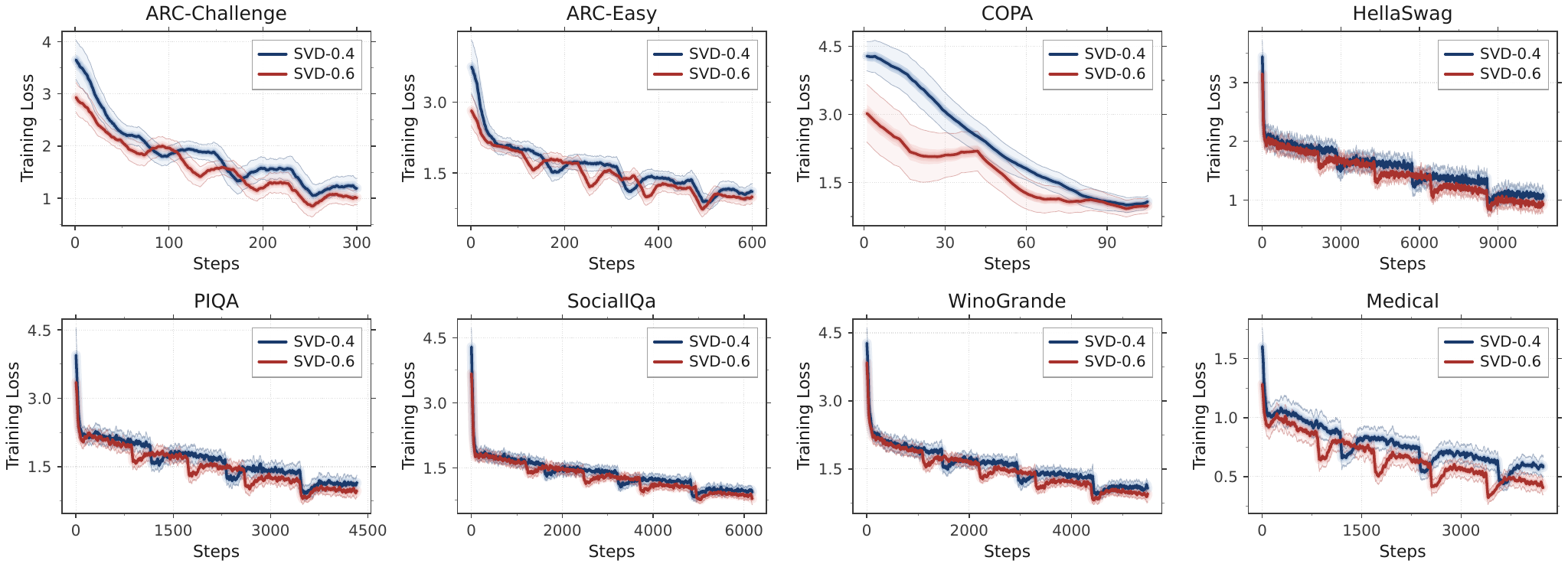}
  \vspace{-1em}
  \caption{Training loss convergence of LLaMA-2-7B clients over communication rounds on all eight benchmarks under the non-IID ($\beta=1.0$) partition.}
  \label{fig:convergence}
\end{figure*}

\begin{figure*}[t]
  \centering
  \includegraphics[width=\linewidth]{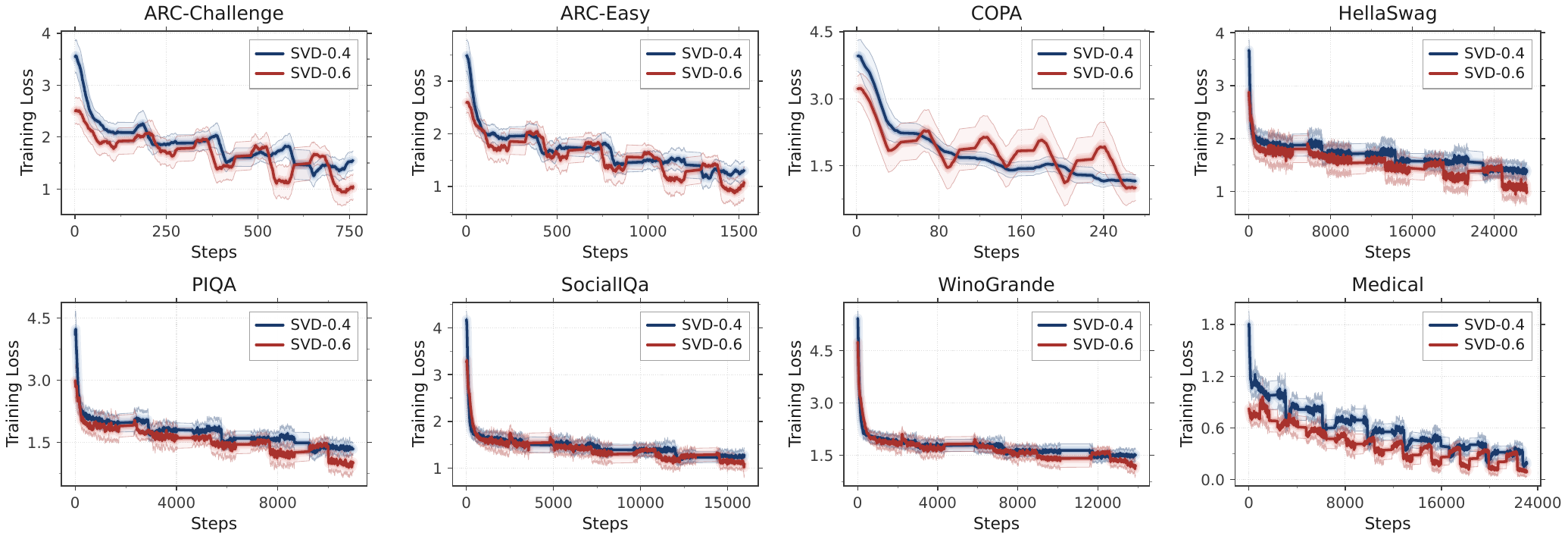}
    \vspace{-1em}
  \caption{Training loss convergence of LLaMA-2-13B clients over communication rounds on all eight benchmarks under the non-IID ($\beta=1.0$) partition.}
  \label{fig:convergence_13b}
\end{figure*}

Figure~\ref{fig:convergence} and Figure~\ref{fig:convergence_13b} report the training loss of FedSLM clients at compression ratios 0.4 and 0.6 over 10 communication rounds on all eight training benchmarks under the non-IID partition, for LLaMA-2-7B and LLaMA-2-13B, respectively.

Across all benchmarks, both compression variants exhibit smooth and monotonic convergence without divergence or oscillation, indicating that the adapter-level aggregation in FedSLM remains stable under heterogeneous model capacities and data distributions. This is a desirable property for practical deployment, where clients with different hardware constraints must coexist within the same federation.

Comparing the two compression ratios, the 0.6-ratio model consistently converges to lower final loss values than the 0.4-ratio model, which is expected given its greater representational capacity. However, the gap between the two curves narrows as training progresses on most benchmarks, suggesting that continued federated aggregation partially compensates for the capacity disparity. Scaling from 7B to 13B preserves this convergence pattern: the 13B model achieves lower absolute loss values across all benchmarks while maintaining the same stable convergence trajectory, confirming that FedSLM's adapter-level aggregation scales gracefully with model size. The consistent behavior across both scales validates the theoretical $\mathcal{O}(T_1^{-1/2})$ convergence rate established in Theorem~\ref{thm:convergence}, which is independent of the underlying model dimension.

\begin{table*}[t]
    \centering
    \caption{Client-side model accuracy (\%) after Stage~1 federated aggregation for both the LLaMA-2-7B and LLaMA-2-13B server models. The upper sub-block of each server model reports pre-trained baselines without any federated training. The best federated result in each column within each server-model block is \textbf{boldfaced}.}
    \label{tab:client}
    \setlength{\tabcolsep}{4pt}
    \setlength{\aboverulesep}{0pt}
    \setlength{\belowrulesep}{0pt}
    \renewcommand{\arraystretch}{1.15}
    \resizebox{\textwidth}{!}{%
    \begin{tabular}{cl ccccccc cc c}
    \toprule
    \rowcolor{gray!15}
    \textbf{Server} & \textbf{Method} & \textbf{ARC\_e} & \textbf{ARC\_c} & \textbf{COPA} & \textbf{HellaS} & \textbf{PIQA} & \textbf{Social} & \textbf{WinoG} & \textbf{PubMed} & \textbf{MedMCQA} & \textbf{Avg} \\
    \midrule
    \multirow{7}{*}{LLaMA-2-7B}
    & SVD-0.4 & 48.4 & 23.1 & 73.0 & 35.4 & 64.4 & 37.6 & 51.3 & 55.0 & 28.7 & 46.3 \\
    & SVD-0.6 & 62.6 & 33.4 & 81.0 & 41.6 & 71.0 & 41.1 & 50.7 & 54.8 & 30.6 & 51.9 \\
    & Full (pretrained) & 76.0 & 40.8 & 87.0 & 52.0 & 77.3 & 46.1 & 49.6 & 55.2 & 31.3 & 57.3 \\
    \cmidrule{2-12}
    & FedSLM-0.4 (IID) & 71.4 & 40.5 & 84.0 & 56.5 & 76.8 & 57.2 & \textbf{74.7} & 66.6 & \textbf{31.8} & 62.2 \\
    & FedSLM-0.6 (IID) & \textbf{75.3} & \textbf{46.5} & \textbf{84.0} & \textbf{58.9} & \textbf{78.6} & \textbf{58.6} & 73.9 & \textbf{68.8} & 31.7 & \textbf{64.0} \\
    \cmidrule{2-12}
    & FedSLM-0.4 (Non-IID) & 72.5 & 37.5 & 83.0 & 56.2 & 77.3 & 58.1 & 56.3 & 64.6 & 31.9 & 59.7 \\
    & FedSLM-0.6 (Non-IID) & \textbf{74.7} & \textbf{45.2} & \textbf{82.0} & \textbf{59.0} & \textbf{78.1} & \textbf{59.3} & \textbf{57.3} & \textbf{65.0} & \textbf{32.9} & \textbf{61.5} \\
    \midrule
    \multirow{7}{*}{LLaMA-2-13B}
    & SVD-0.4 & 51.6 & 24.4 & 72.0 & 31.6 & 60.8 & 37.6 & 50.6 & 54.2 & 32.0 & 46.9 \\
    & SVD-0.6 & 71.9 & 35.8 & 80.0 & 45.2 & 73.8 & 42.7 & 51.1 & 55.2 & 28.6 & 57.2 \\
    & Full (pretrained) & 79.1 & 43.8 & 86.0 & 54.8 & 79.4 & 47.5 & 50.7 & 55.2 & 34.0 & 63.0 \\
    \cmidrule{2-12}
    & FedSLM-0.4 (IID) & 70.5 & 37.5 & 77.0 & 54.3 & 75.8 & 56.9 & 73.6 & \textbf{68.6} & 31.3 & 63.7 \\
    & FedSLM-0.6 (IID) & \textbf{79.3} & \textbf{49.2} & \textbf{84.0} & \textbf{63.6} & \textbf{80.5} & \textbf{61.6} & \textbf{80.1} & 66.2 & \textbf{32.7} & \textbf{71.2} \\
    \cmidrule{2-12}
    & FedSLM-0.4 (Non-IID) & 70.9 & 38.8 & 74.0 & 54.6 & 76.1 & 56.8 & 55.2 & \textbf{71.4} & 32.1 & 60.9 \\
    & FedSLM-0.6 (Non-IID) & \textbf{77.5} & \textbf{48.2} & \textbf{85.0} & \textbf{63.7} & \textbf{80.0} & \textbf{62.5} & \textbf{60.1} & 70.6 & \textbf{32.7} & \textbf{68.2} \\
    \bottomrule
    \end{tabular}%
    }
\end{table*}
\begin{table}[t]
    \centering
    \vspace{-1em}
    \caption{Effect of adapter placement on server-side accuracy (\%). The best result in each row is \textbf{boldfaced}.}
    \label{tab:adapter_placement}
    \setlength{\tabcolsep}{4pt}
    \setlength{\aboverulesep}{0pt}
    \setlength{\belowrulesep}{0pt}
    \renewcommand{\arraystretch}{1.15}
    \begin{tabular}{l cccc}
    \toprule
    \rowcolor{gray!15}
    \textbf{Benchmark} & \textbf{\textsc{Atten-QV}} & \textbf{\textsc{Atten-All}} & \textbf{\textsc{FFN-Only}} & \textbf{\textsc{All}} \\
    \midrule
    ARC\_e  & 68.1 & 72.1 & 77.4 & \textbf{78.1} \\
    ARC\_c  & 30.8 & 32.1 & 39.5 & \textbf{44.1} \\
    COPA    & 84.0 & \textbf{85.0} & 83.0 & \textbf{85.0} \\
    HellaS  & 48.7 & 50.9 & 57.2 & \textbf{60.7} \\
    PIQA    & 74.2 & 75.0 & 78.5 & \textbf{79.8} \\
    Social  & 51.2 & 53.4 & 59.2 & \textbf{61.4} \\
    WinoG   & 52.4 & 50.8 & 57.3 & \textbf{57.9} \\
    PubMed  & 64.2 & 65.7 & 67.9 & \textbf{68.4} \\
    MedMCQA & 30.6 & 32.9 & 34.4 & \textbf{35.1} \\
    \midrule
    Avg        & 56.0 & 57.5 & 61.6 & \textbf{63.4} \\
    \bottomrule
    \end{tabular}
\end{table}
\subsection{Client-Side Model Performance}

Table~\ref{tab:client} reports client-side accuracy after Stage~1 adapter-level aggregation on both server models, alongside the pre-trained baselines without federated training.

On LLaMA-2-7B, federated aggregation yields consistent improvements. Without training, SVD-0.4 and SVD-0.6 average 46.3\% and 51.9\% respectively, reflecting the information loss from low-rank decomposition. After 10 rounds, FedSLM-0.4 reaches 62.2\% (IID) and 59.7\% (non-IID), while FedSLM-0.6 reaches 64.0\% and 61.5\%. Both variants surpass the zero-shot LLaMA-2-7B baseline (57.3\%) despite containing fewer parameters, with the largest FedSLM-0.6 vs. FedSLM-0.4 gap appearing on ARC\_c (46.5\% vs.\ 40.5\% under IID). On the medical benchmarks, pre-trained baselines cluster around 55\% on PubMedQA regardless of model size, while federated training lifts FedSLM-0.6 to 68.8\% (IID) and MedMCQA from 30.6\% (pre-trained) to 32.9\% (non-IID). We note that this comparison is between task-specific adapter-tuned models and a zero-shot pre-trained model; the improvement therefore reflects the combined effect of federated task adaptation and the adapter's ability to recover capacity lost during compression, rather than a claim that compression is free. Nevertheless, the result demonstrates that SVD-compressed models are effective knowledge carriers in the federated setting: even at 40\% compression, adapter training recovers and exceeds the zero-shot performance of the uncompressed model.

On LLaMA-2-13B, the client-side trend scales consistently. On the seven commonsense benchmarks, FedSLM-0.6 reaches 71.2\% (IID) and 68.2\% (non-IID), exceeding FedSLM-0.4 (63.7\% / 60.9\%) by 7.5\% / 7.3\%. Both compressed variants surpass the zero-shot LLaMA-2-13B baseline (63.0\%) despite containing fewer parameters, with FedSLM-0.6 closing the gap most sharply on reasoning-heavy tasks such as HellaSwag (63.6\% vs.\ 54.8\% zero-shot), ARC\_c (49.2\% vs.\ 43.8\%), and WinoGrande (80.1\% vs.\ 50.7\%). Across both scales, the non-IID partition introduces a visible drop on WinoGrande (7B FedSLM-0.4: 74.7\% $\to$ 56.3\%; 13B FedSLM-0.6: 80.1\% $\to$ 60.1\%), which relies on fine-grained coreference resolution sensitive to distributional skew, whereas benchmarks such as COPA and PIQA remain relatively stable. Together these results confirm that SVD-compressed models can achieve strong client-side performance through federated adapter training at both scales, providing a reliable weak supervision signal for Stage~2 that explains the effectiveness of the subsequent weak-to-strong elicitation reported in Tables~\ref{tab:nlu} and~\ref{tab:nlu_13b}.

\subsection{Adapter Placement Ablation}

To understand which weight matrices benefit most from adapter-level fine-tuning, we vary the set of linear layers to which adapters are attached while keeping all other hyperparameters fixed. 
We consider four configurations: \textsc{Atten-QV} adapts only the query and value projections; \textsc{Atten-All} extends to all attention projections (Q, K, V, O); \textsc{FFN-Only} targets the feed-forward sub-layers (gate, up, down projections); and \textsc{All} adapts all linear layers per transformer block. 
For SVD-compressed clients, the configuration name indicates which projection matrices receive low-rank adapters; for the LLaMA clients, the corresponding LoRA modules target the same layers. 

Table~\ref{tab:adapter_placement} summarizes the results. Adapting only the attention query and value projections (\textsc{Atten-QV}) yields the lowest average accuracy (56.0\%), and extending to all four attention projections (\textsc{Atten-All}) provides only a modest improvement of 1.5\%. 
In contrast, targeting the feed-forward sub-layers alone (\textsc{FFN-Only}) raises the average to 61.6\%, a gain of 5.6\% over \textsc{Atten-QV}. This gap suggests that the feed-forward network stores a larger share of the task-relevant knowledge than the attention mechanism, consistent with recent findings on knowledge localization in transformer models~\cite{geva2021transformer,meng2022locating}.

\begin{table}[t]
  \centering
      \vspace{-1em}
  \caption{Performance on vision--language benchmarks (accuracy \%). SVD-0.6, SVD-0.9, and Full-model are pre-trained baselines without federated training. The best result is \textbf{boldfaced}.}
  \label{tab:vlm}
  \setlength{\tabcolsep}{6pt}
  \setlength{\aboverulesep}{0pt}
  \setlength{\belowrulesep}{0pt}
  \renewcommand{\arraystretch}{1.15}
  \begin{tabular}{l cc cc}
  \toprule
  \rowcolor{gray!15}
  & \multicolumn{2}{c}{\textbf{ScienceQA}} & \multicolumn{2}{c}{\textbf{VizWiz}} \\
  \cmidrule(lr){2-3} \cmidrule(lr){4-5}
  \rowcolor{gray!15}
  \multirow{-2}{*}{\textbf{Method}} & \textbf{IID} & \textbf{Non-IID} & \textbf{IID} & \textbf{Non-IID} \\
  \midrule
  SVD-0.6 & \multicolumn{2}{c}{57.5} & \multicolumn{2}{c}{44.8} \\
  SVD-0.9 & \multicolumn{2}{c}{64.9} & \multicolumn{2}{c}{55.0} \\
  LLaVA-NeXT-7B & \multicolumn{2}{c}{66.2} & \multicolumn{2}{c}{58.4} \\
  \midrule
  FedSLM ($\alpha\!=\!0$) & 70.9 & 71.8 & 69.9 & 69.7 \\
  FedSLM ($\alpha\!=\!0.5$) & \textbf{71.8} & \textbf{72.7} & \textbf{70.1} & \textbf{69.9} \\
  \bottomrule
  \end{tabular}
\end{table}

The best performance is achieved by the \textsc{All} configuration (63.4\%), which adapts all linear layers per transformer block. 
The incremental gain of 1.8\% over \textsc{FFN-Only} indicates that attention projections still contribute complementary information once the feed-forward layers are already adapted. 
The improvement is most pronounced on ARC\_c (+4.6\%) and HellaS (+3.5\%), both of which require multi-hop reasoning that benefits from jointly adapting the attention routing and the knowledge retrieval pathways.

\subsection{Vision--Language Experiments}

To evaluate FedSLM beyond language-only tasks, we apply it to LLaVA-NeXT-7B~\cite{liu2024llavanext} on two vision--language benchmarks: ScienceQA~\cite{lu2022scienceqa} and VizWiz~\cite{gurari2018vizwiz}. Client-side compressed models are obtained via QSVD~\cite{qsvd2025} at compression ratios 0.6 and 0.9. We simulate a federation of 10 clients.

Table~\ref{tab:vlm} reports the results. The pre-trained baselines reveal a clear capacity hierarchy: SVD-0.6 scores 57.5\% on ScienceQA and 44.8\% on VizWiz, while the full LLaVA-NeXT-7B reaches 66.2\% and 58.4\%. FedSLM ($\alpha\!=\!0.5$) achieves 71.8\% on ScienceQA and 70.1\% on VizWiz under IID, surpassing the full LLaVA-NeXT-7B by 5.6\% and 11.7\% respectively, confirming that federated adapter training effectively recovers and exceeds the capacity lost to compression. Under non-IID, FedSLM ($\alpha\!=\!0.5$) attains 72.7\% on ScienceQA and 69.9\% on VizWiz. Notably, the non-IID results on ScienceQA are slightly higher than IID, likely because the distributional diversity across clients provides complementary visual reasoning patterns that benefit the server model after aggregation.

These results demonstrate that FedSLM generalizes beyond language tasks to the multimodal setting, where the SVD-compressed vision--language models serve as effective federated participants despite operating at a fraction of the full model's capacity.

\section{Conclusion and Future Work}

We presented FedSLM, an FL framework that enables clients with heterogeneous SVD-compressed models to collaboratively improve a server-side foundation model. 
By integrating SVD-based model derivation, two-stage heterogeneous aggregation, and weak-to-strong elicitation with an auxiliary confidence loss, FedSLM achieves theoretical guarantees on convergence, subspace alignment, and artifact mitigation. 
Experiments on natural language and vision--language benchmarks show consistent improvements over existing baselines under both IID and non-IID partitions, with client-side memory reduced to roughly half that of the full model.

Several directions remain for future work, including adaptive compression-ratio selection based on client resource availability, extension to other compression paradigms such as pruning and quantization, and integration of secure aggregation for deployment in regulated domains.

\bibliographystyle{IEEEtran}
\bibliography{references}

\newpage
\appendices


\section{Proofs and Supporting Lemmas}
\label{sec:proofs}
\allowdisplaybreaks

Throughout the appendix we adopt the notation of the main text.  In particular, $\bU_r \in \mathbb{R}^{m \times r}$ and $\bV_r \in \mathbb{R}^{r \times n}$ denote the (non-orthonormal) factors delivered to clients, as in~\eqref{eq:balanced_factorization}--\eqref{eq:scaled_factorization}, while $\bU_r^{\circ} \in \mathbb{R}^{m\times r}$ and $\bV_r^{\circ} \in \mathbb{R}^{r\times n}$ denote the orthonormal singular-vector factors ($\bU_r^{\circ\,\top}\bU_r^{\circ} = \mathbf{I}_r$, $\bV_r^{\circ}\bV_r^{\circ\,\top} = \mathbf{I}_r$). The retained spectrum is $\bLambda_r = \mathrm{diag}(\lambda_1,\ldots,\lambda_r)$, and $\mathbf{S} \in \mathbb{R}^{r \times r}$ is the additional invertible scaling matrix, so that $\bU_r = \bU_r^{\circ}\bLambda_r^{1/2}\mathbf{S}^{-1}$ and $\bV_r = \mathbf{S}\,\bLambda_r^{1/2}\bV_r^{\circ}$ (the balanced factorization is $\mathbf{S} = \mathbf{I}_r$). The condition number $\kappa(\mathbf{S}) = \|\mathbf{S}\|_2\,\|\mathbf{S}^{-1}\|_2$ equals $1$ in the balanced case. The effective smoothness constant $L_{\mathrm{eff}} = L\,\lambda_1^2\,\kappa(\mathbf{S})^2$ is from Definition~\ref{def:eff_const}.

\subsection{Mathematical Preliminaries}
\label{app:prelim}

This subsection collects the standard tools invoked in the main analysis: the Moore--Penrose pseudoinverse (used to define the geometric projection adapter $\widehat{\bSigma}_i$ in Section~\ref{subsec:subspace}), orthogonal projectors and their algebraic properties (used throughout the subspace alignment theory), and the P{\L} condition together with its quadratic growth consequence (used in Proposition~\ref{prop:residual_control}). Readers familiar with these tools may skip ahead.

\subsubsection{Moore--Penrose pseudoinverse}
\label{app:prelim_pinv}

For a column-full-rank matrix $\mathbf{A} \in \mathbb{R}^{m \times r}$ with $m \geq r$, the Moore--Penrose pseudoinverse is defined by
\begin{equation}
    \mathbf{A}^{\dagger}
    \;\triangleq\;
    (\mathbf{A}^{\top}\mathbf{A})^{-1}\mathbf{A}^{\top}
    \;\in\; \mathbb{R}^{r \times m} ,
    \label{eq:pinv_def}
\end{equation}
where $\mathbf{A}^{\top}\mathbf{A}$ is invertible precisely because $\mathbf{A}$ has full column rank.  When $\mathbf{A}$ is itself square and invertible, $\mathbf{A}^{\dagger}$ coincides with the ordinary inverse $\mathbf{A}^{-1}$.  For a row-full-rank matrix $\mathbf{B} \in \mathbb{R}^{r \times n}$ with $n \geq r$, the analogous formula is $\mathbf{B}^{\dagger} = \mathbf{B}^{\top} (\mathbf{B}\mathbf{B}^{\top})^{-1}$.

The pseudoinverse is the matrix form of the least-squares-solution operator.  For any $\mathbf{b} \in \mathbb{R}^{m}$, the unique minimizer of $\|\mathbf{A}\mathbf{x} - \mathbf{b}\|_2^{\,2}$ over $\mathbf{x} \in \mathbb{R}^{r}$ satisfies the normal equations $\mathbf{A}^{\top}\mathbf{A}\mathbf{x} = \mathbf{A}^{\top}\mathbf{b}$, yielding $\mathbf{x}^{*} = \mathbf{A}^{\dagger}\mathbf{b}$. The corresponding fitted value $\mathbf{A}\mathbf{A}^{\dagger}\mathbf{b}$ has a geometric interpretation as an orthogonal projection, which we develop after introducing the projector formalism in Appendix~\ref{app:prelim_projector} below.

\subsubsection{Orthogonal projectors}
\label{app:prelim_projector}

We begin with the formal definition.

\begin{definition}[Orthogonal projector]
\label{def:orth_proj}
A matrix $\mathbf{P} \in \mathbb{R}^{m \times m}$ is an \emph{orthogonal projector} onto a subspace $\mathcal{M} \subseteq \mathbb{R}^{m}$ if it satisfies the following three equivalent properties:
\begin{enumerate}[leftmargin=2.5em]
\item[\textbf{(a)}] $\mathbf{P}^{2} = \mathbf{P}$, $\mathbf{P}^{\top} = \mathbf{P}$, and $\mathrm{range}(\mathbf{P}) = \mathcal{M}$;
\item[\textbf{(b)}] $\mathbf{P}$ acts as the identity on $\mathcal{M}$ and as zero on the orthogonal complement $\mathcal{M}^{\perp}$;
\item[\textbf{(c)}] for every $\mathbf{v} \in \mathbb{R}^{m}$, $\mathbf{P}\mathbf{v}$ is the unique element of $\mathcal{M}$ minimizing $\|\mathbf{v} - \mathbf{w}\|_2$ over $\mathbf{w} \in \mathcal{M}$.
\end{enumerate}
\end{definition}

Property~(c) in Definition~\ref{def:orth_proj} shows that an orthogonal projector is determined by the subspace alone, independently of any particular basis chosen to represent it. A constructive realization uses the Moore--Penrose pseudoinverse~\eqref{eq:pinv_def}: for any matrix whose columns span the target subspace, the product $\mathbf{A}\mathbf{A}^{\dagger}$ delivers the projector explicitly. Concretely, the fitted value $\hat{\mathbf{b}} = \mathbf{A}\mathbf{A}^{\dagger}\mathbf{b}$ is the element of the column space $\mathrm{col}(\mathbf{A}) \triangleq \{\mathbf{A}\mathbf{x} : \mathbf{x} \in \mathbb{R}^{r}\}$ closest to $\mathbf{b}$ in Euclidean norm, so the map $\mathbf{b} \mapsto \mathbf{A}\mathbf{A}^{\dagger}\mathbf{b}$ acts as an orthogonal projection onto $\mathrm{col}(\mathbf{A})$. The following lemma makes this precise.

\begin{lemma}[Pseudoinverse--projector identity]
\label{lem:pinv_proj}
For a column-full-rank $\mathbf{A} \in \mathbb{R}^{m \times r}$, the matrix $\mathbf{A}\mathbf{A}^{\dagger}$ is the orthogonal projector onto $\mathrm{col}(\mathbf{A})$.
\end{lemma}

\begin{proof}
We verify the three properties of Definition~\ref{def:orth_proj}\,(a): idempotence, symmetry, and the fixing property.

\emph{Idempotence:}
\begin{equation*}
\begin{split}
    (\mathbf{A}\mathbf{A}^{\dagger})^{2}
    &= \mathbf{A}(\mathbf{A}^{\top}\mathbf{A})^{-1}
      (\mathbf{A}^{\top}\mathbf{A})
      (\mathbf{A}^{\top}\mathbf{A})^{-1}
      \mathbf{A}^{\top} \\
    &= \mathbf{A}(\mathbf{A}^{\top}\mathbf{A})^{-1}
      \mathbf{A}^{\top}
    = \mathbf{A}\mathbf{A}^{\dagger}.
\end{split}
\end{equation*}

\emph{Symmetry:}
\begin{equation*}
\begin{split}
    (\mathbf{A}\mathbf{A}^{\dagger})^{\top}
    &= \mathbf{A}
      ((\mathbf{A}^{\top}\mathbf{A})^{-1})^{\top}
      \mathbf{A}^{\top} \\
    &= \mathbf{A}(\mathbf{A}^{\top}\mathbf{A})^{-1}
      \mathbf{A}^{\top}
    = \mathbf{A}\mathbf{A}^{\dagger},
\end{split}
\end{equation*}
since $(\mathbf{A}^{\top}\mathbf{A})^{-1}$ inherits symmetry from $\mathbf{A}^{\top}\mathbf{A}$.

\emph{Fixing $\mathrm{col}(\mathbf{A})$:} for $\mathbf{v} = \mathbf{A}\mathbf{x}$,
\begin{equation*}
\begin{split}
    \mathbf{A}\mathbf{A}^{\dagger}\mathbf{v}
    &= \mathbf{A}(\mathbf{A}^{\top}\mathbf{A})^{-1}
      (\mathbf{A}^{\top}\mathbf{A})\mathbf{x} \\
    &= \mathbf{A}\mathbf{x}
    = \mathbf{v}.
\end{split}
\end{equation*}
Idempotence and symmetry together imply that $\mathbf{A}\mathbf{A}^{\dagger}$ is an orthogonal projector, and the fixing property identifies the projection range with $\mathrm{col}(\mathbf{A})$.
\end{proof}

\begin{remark}[Basis invariance]
\label{rem:basis_invar}
Lemma~\ref{lem:pinv_proj} implies that $\mathbf{A}\mathbf{A}^{\dagger}$ depends on $\mathbf{A}$ only through its column space.  Concretely, if $\mathbf{A} = \mathbf{A}^{\circ}\mathbf{S}^{-1}$ for some invertible $\mathbf{S}$ (so $\mathbf{A}$ and $\mathbf{A}^{\circ}$ span the same subspace), then $\mathbf{A}\mathbf{A}^{\dagger} = \mathbf{A}^{\circ}(\mathbf{A}^{\circ})^{\dagger}$. We use this for Lemma~\ref{lem:P_Q_pinv} below to show that $\bU_{r_i}\bU_{r_i}^{\dagger} = \mathbf{P}_{r_i}$ for every factorization variant, since the deployed factor $\bU_{r_i}$ differs from $\bU_{r_i}^{\circ}$ only by an invertible right factor ($\bLambda_{r_i}^{1/2}$ for the balanced split, $\bLambda_{r_i}^{1/2}\mathbf{S}^{-1}$ when whitening is applied) and hence spans the same column space.
\end{remark}

In our setting, the rank-$r$ subspaces $\mathcal{U}_{r} \subseteq \mathbb{R}^{m}$ and $\mathcal{V}_{r} \subseteq \mathbb{R}^{n}$ (Definition~\ref{def:subspaces}) have orthogonal projectors
\begin{equation}
    \mathbf{P}_{r} = \bU_{r}^{\circ}\,\bU_{r}^{\circ\,\top},
    \qquad
    \mathbf{Q}_{r} = \bV_{r}^{\circ\,\top}\,\bV_{r}^{\circ}.
\end{equation}

\begin{lemma}[Pseudoinverse formulas for $\mathbf{P}_{r_i}$
and $\mathbf{Q}_{r_i}$]
\label{lem:P_Q_pinv}
For the deployed factors $\bU_{r_i} = \bU_{r_i}^{\circ}\mathbf{M}_U$ and $\bV_{r_i} = \mathbf{M}_V\,\bV_{r_i}^{\circ}$ with $\mathbf{M}_U, \mathbf{M}_V \in \mathbb{R}^{r_i \times r_i}$ invertible (this covers both the balanced factorization with $\mathbf{M}_U = \mathbf{M}_V = \bLambda_{r_i}^{1/2}$ and the whitened variant~\eqref{eq:scaled_factorization} with $\mathbf{M}_U = \bLambda_{r_i}^{1/2}\mathbf{S}^{-1}$ and $\mathbf{M}_V = \mathbf{S}\,\bLambda_{r_i}^{1/2}$), we have
\begin{equation}
    \bU_{r_i}\,\bU_{r_i}^{\dagger}
    \;=\; \mathbf{P}_{r_i},
    \qquad
    \bV_{r_i}^{\dagger}\,\bV_{r_i}
    \;=\; \mathbf{Q}_{r_i}.
\end{equation}
\end{lemma}

\begin{proof}
Since $\bU_{r_i}^{\circ\,\top}\bU_{r_i}^{\circ} = \mathbf{I}$, we have $\bU_{r_i}^{\top}\bU_{r_i} = \mathbf{M}_U^{\top}\mathbf{M}_U$, hence $(\bU_{r_i}^{\top}\bU_{r_i})^{-1} = \mathbf{M}_U^{-1}\mathbf{M}_U^{-\top}$ and
\begin{equation}
\begin{split}
    \bU_{r_i}^{\dagger}
    &= (\bU_{r_i}^{\top}\bU_{r_i})^{-1}\bU_{r_i}^{\top} \\
    &= \mathbf{M}_U^{-1}\mathbf{M}_U^{-\top}\mathbf{M}_U^{\top}\bU_{r_i}^{\circ\,\top}
    = \mathbf{M}_U^{-1}\bU_{r_i}^{\circ\,\top},
\end{split}
\end{equation}
so that
\begin{equation}
    \bU_{r_i}\bU_{r_i}^{\dagger}
    = \bU_{r_i}^{\circ}\,\mathbf{M}_U\mathbf{M}_U^{-1}\,\bU_{r_i}^{\circ\,\top}
    = \bU_{r_i}^{\circ}\bU_{r_i}^{\circ\,\top}
    = \mathbf{P}_{r_i} .
\end{equation}
The invertible scaling $\mathbf{M}_U$ cancels completely. The verification for $\mathbf{Q}_{r_i}$ is analogous: writing $\bV_{r_i} = \mathbf{M}_V\,\bV_{r_i}^{\circ}$, row-fullness of $\bV_{r_i}$ gives $\bV_{r_i}^{\dagger} = \bV_{r_i}^{\top}(\bV_{r_i}\bV_{r_i}^{\top})^{-1} = \bV_{r_i}^{\circ\,\top}\mathbf{M}_V^{-1}$, and $\bV_{r_i}^{\dagger}\bV_{r_i} = \bV_{r_i}^{\circ\,\top}\mathbf{M}_V^{-1}\mathbf{M}_V\,\bV_{r_i}^{\circ} = \bV_{r_i}^{\circ\,\top}\bV_{r_i}^{\circ} = \mathbf{Q}_{r_i}$.
\end{proof}

Lemma~\ref{lem:P_Q_pinv} equates $\bU_{r_i}\widehat{\bSigma}_i\bV_{r_i}$ with $\mathbf{P}_{r_i}\widehat{\bW}\mathbf{Q}_{r_i}$: substituting $\widehat{\bSigma}_i = \bU_{r_i}^{\dagger}\widehat{\bW}\bV_{r_i}^{\dagger}$ and applying the identities in Lemma~\ref{lem:P_Q_pinv} gives
\begin{equation}
    \bU_{r_i}\widehat{\bSigma}_i\bV_{r_i}
    =
    \underbrace{\bU_{r_i}\bU_{r_i}^{\dagger}}_{=\mathbf{P}_{r_i}}
    \widehat{\bW}
    \underbrace{\bV_{r_i}^{\dagger}\bV_{r_i}}_{=\mathbf{Q}_{r_i}}
    = \mathbf{P}_{r_i}\widehat{\bW}\mathbf{Q}_{r_i}.
\end{equation}

\begin{lemma}[Absorption identity for nested projectors]
\label{lem:absorption}
Let $\mathbf{P}_{1}, \mathbf{P}_{2}$ be orthogonal projectors onto subspaces $\mathcal{M}_{1} \subseteq \mathcal{M}_{2}$.  Then $\mathbf{P}_{1}\mathbf{P}_{2} = \mathbf{P}_{2}\mathbf{P}_{1} = \mathbf{P}_{1}$.
\end{lemma}

\begin{proof}
For any $\mathbf{v}$, $\mathbf{P}_{2}\mathbf{v}$ decomposes orthogonally as $\mathbf{P}_{2}\mathbf{v} = \mathbf{P}_{1}\mathbf{v} + (\mathbf{P}_{2}-\mathbf{P}_{1})\mathbf{v}$, where the second term lies in $\mathcal{M}_{2} \cap \mathcal{M}_{1}^{\perp}$.  Applying $\mathbf{P}_{1}$, which is zero on $\mathcal{M}_{1}^{\perp}$, gives $\mathbf{P}_{1}\mathbf{P}_{2}\mathbf{v} = \mathbf{P}_{1}\mathbf{v}$.  The symmetric identity $\mathbf{P}_{2}\mathbf{P}_{1} = \mathbf{P}_{1}$ follows by transposing.
\end{proof}

Lemma~\ref{lem:absorption} is the abstract fact underlying Proposition~\ref{prop:nesting}; the main-text proof simply computes it in coordinates using the explicit block structure of $\bU_{r_1}^{\circ}$ relative to $\bU_{r_2}^{\circ}$.

\subsubsection{Polyak--{\L}ojasiewicz condition and quadratic growth}
\label{app:prelim_pl}

A differentiable function $f:\mathbb{R}^{d} \to \mathbb{R}$ with nonempty minimizer set $\arg\min f$ and infimum $f^{*} = \inf f$ satisfies the \emph{Polyak--{\L}ojasiewicz (P{\L}) inequality} with constant $\mu > 0$ if
\begin{equation}
    \|\nabla f(\mathbf{x})\|^{2}
    \;\geq\;
    2\mu\bigl(f(\mathbf{x}) - f^{*}\bigr),
    \qquad \forall\, \mathbf{x}.
    \label{eq:pl_abstract}
\end{equation}
Inequality~\eqref{eq:pl_abstract} is weaker than strong convexity: it does not require convexity of $f$ and permits non-unique minimizers.  A standard sufficient condition is $f(\mathbf{x}) = g(\mathbf{A}\mathbf{x})$ with $g$ strongly convex, where $\mathbf{A}$ may be rank-deficient~\cite{karimi2016linear}; compositions of this form arise naturally in overparameterized learning.

\begin{lemma}[Quadratic growth; {\cite{karimi2016linear}}]
\label{lem:pl_qg}
If $f$ satisfies the P{\L} inequality~\eqref{eq:pl_abstract} with constant $\mu$, then it satisfies the \emph{quadratic-growth (QG) inequality}
\begin{equation}
    f(\mathbf{x}) - f^{*}
    \;\geq\;
    \tfrac{\mu}{2}\,
    \bigl\|\mathbf{x} - \Pi_{\arg\min f}(\mathbf{x})\bigr\|^{2},
    \qquad \forall\, \mathbf{x},
\label{eq:qg_abstract}
\end{equation}
where $\Pi_{\arg\min f}$ denotes Euclidean projection onto the (closed) minimizer set.
\end{lemma}

The implication P{\L}~$\Rightarrow$~QG is obtained by integrating the P{\L} inequality along the gradient flow of $f$ starting at $\mathbf{x}$: the flow converges to $\arg\min f$ with total path length bounded by $\sqrt{2(f(\mathbf{x})-f^{*})/\mu}$, which in turn upper-bounds $\|\mathbf{x} - \Pi_{\arg\min f}(\mathbf{x})\|$. We use Lemma~\ref{lem:pl_qg} in the proof of Proposition~\ref{prop:residual_control} to pass from an objective-gap bound $\mathbb{E}[f_i(\bSigma_i^{(T_1)}) - f_i^{*}]$ to an iterate-distance bound $\mathbb{E}\|\bSigma_i^{(T_1)} - \bSigma_i^{*}\|_F$.

\subsection{Proofs for Convergence Analysis}
\label{app:convergence_proofs}

\subsubsection{Supporting Lemmas}

\begin{lemma}[SVD isometry with scaling]
\label{lem:isometry}
Let $\bU_r = \bU_r^{\circ}\bLambda_r^{1/2}\mathbf{S}^{-1}$ and $\bV_r = \mathbf{S}\,\bLambda_r^{1/2}\bV_r^{\circ}$ with $\bU_r^{\circ\,\top}\bU_r^{\circ} = \mathbf{I}_r$, $\bV_r^{\circ}\bV_r^{\circ\,\top} = \mathbf{I}_r$, $\bLambda_r = \mathrm{diag}(\lambda_1,\ldots,\lambda_r) \succ \mathbf{0}$, and $\mathbf{S} \in \mathbb{R}^{r \times r}$ invertible. For any $\mathbf{A} \in \mathbb{R}^{r \times r}$,
\begin{equation}
    \|\bU_r\,\mathbf{A}\,\bV_r\|_F
    \;=\;
    \|\bLambda_r^{1/2}\,\mathbf{S}^{-1}\mathbf{A}\,\mathbf{S}\,\bLambda_r^{1/2}\|_F ,
\end{equation}
and consequently
\begin{equation}
    \lambda_r\,\kappa(\mathbf{S})^{-1}\,\|\mathbf{A}\|_F
    \;\leq\;
    \|\bU_r\,\mathbf{A}\,\bV_r\|_F
    \;\leq\;
    \lambda_1\,\kappa(\mathbf{S})\,\|\mathbf{A}\|_F.
    \label{eq:isometry_general}
\end{equation}
In the balanced case $\mathbf{S} = \mathbf{I}_r$, the bound reads $\lambda_r\,\|\mathbf{A}\|_F \leq \|\bU_r\mathbf{A}\bV_r\|_F \leq \lambda_1\,\|\mathbf{A}\|_F$, with equality throughout iff the retained spectrum is flat ($\lambda_1 = \lambda_r$).
\end{lemma}

\begin{proof}
Write $\mathbf{B} \triangleq \bLambda_r^{1/2}\,\mathbf{S}^{-1}\mathbf{A}\,\mathbf{S}\,\bLambda_r^{1/2}$, so that $\bU_r\mathbf{A}\bV_r = \bU_r^{\circ}\mathbf{B}\,\bV_r^{\circ}$. Left multiplication by the column-orthonormal $\bU_r^{\circ}$ and right multiplication by the row-orthonormal $\bV_r^{\circ}$ both preserve the Frobenius norm, hence $\|\bU_r\mathbf{A}\bV_r\|_F = \|\mathbf{B}\|_F$. For the upper bound, using $\|\bLambda_r^{1/2}\|_2 = \lambda_1^{1/2}$,
\begin{equation}
\begin{split}
    \|\mathbf{B}\|_F
    &\;\leq\;
    \|\bLambda_r^{1/2}\|_2^2\,\|\mathbf{S}^{-1}\mathbf{A}\mathbf{S}\|_F \\
    &\;\leq\;
    \lambda_1\,\|\mathbf{S}^{-1}\|_2\,\|\mathbf{S}\|_2\,\|\mathbf{A}\|_F
    \;=\;
    \lambda_1\,\kappa(\mathbf{S})\,\|\mathbf{A}\|_F .
\end{split}
\end{equation}
The lower bound follows analogously from $\|\bLambda_r^{1/2}\mathbf{C}\bLambda_r^{1/2}\|_F \geq \lambda_{\min}(\bLambda_r^{1/2})^2\,\|\mathbf{C}\|_F = \lambda_r\,\|\mathbf{C}\|_F$ with $\mathbf{C} = \mathbf{S}^{-1}\mathbf{A}\mathbf{S}$, together with $\|\mathbf{S}^{-1}\mathbf{A}\mathbf{S}\|_F \geq \kappa(\mathbf{S})^{-1}\|\mathbf{A}\|_F$.
\end{proof}

\begin{lemma}[$L_{\mathrm{eff}}$-smoothness of the adapter loss]
\label{lem:smoothness}
Under Assumption~\ref{ass:standard}\,(a) and with the factor structure of Lemma~\ref{lem:isometry}, for any client $c \in \mathcal{C}_i$, the local loss $f_c(\bSigma) = \mathcal{L}_c(\bU_{r_i}\bSigma\bV_{r_i})$ is $L_{\mathrm{eff}}$-smooth in $\bSigma$:
\begin{equation}
    \|\nabla_{\bSigma} f_c(\bSigma)
      - \nabla_{\bSigma} f_c(\bSigma')\|_F
    \;\leq\;
    L_{\mathrm{eff}}\,\|\bSigma - \bSigma'\|_F,
\end{equation}
where $L_{\mathrm{eff}} = L\,\lambda_1^2\,\kappa(\mathbf{S})^2$. The same constant applies to $f_i(\bSigma)$ by convexity of the weighted sum.
\end{lemma}

\begin{proof}
By the chain rule,
\begin{equation}
    \nabla_{\bSigma} f_c(\bSigma)
    = \bU_{r_i}^{\top}\,
      \nabla_{\bW}\mathcal{L}_c(\bU_{r_i}\bSigma\bV_{r_i})\,
      \bV_{r_i}^{\top} .
\end{equation}
Set $\mathbf{D} = \nabla_{\bW}\mathcal{L}_c(\bU_{r_i}\bSigma\bV_{r_i}) - \nabla_{\bW}\mathcal{L}_c(\bU_{r_i}\bSigma'\bV_{r_i})$. Then
\begin{equation}
\begin{split}
    &\|\nabla_{\bSigma} f_c(\bSigma)
      - \nabla_{\bSigma} f_c(\bSigma')\|_F
    = \|\bU_{r_i}^{\top}\mathbf{D}\,\bV_{r_i}^{\top}\|_F \\
    &\;\stackrel{(a)}{\leq}
       \|\bU_{r_i}^{\top}\|_2\,\|\mathbf{D}\|_F\,
       \|\bV_{r_i}^{\top}\|_2 \\
    &\;\stackrel{(b)}{\leq}
       \|\bU_{r_i}\|_2\,\|\bV_{r_i}\|_2\,
       L\,\|\bU_{r_i}(\bSigma - \bSigma')
       \bV_{r_i}\|_F \\
    &\;\stackrel{(c)}{\leq}
       \|\bU_{r_i}\|_2\,\|\bV_{r_i}\|_2\!\cdot\!
       L\!\cdot\!\lambda_1\,\kappa(\mathbf{S})\,
       \|\bSigma - \bSigma'\|_F,
\end{split}
\end{equation}
where $(a)$ uses the operator-norm bound on sandwiching by matrices (and $\|\mathbf{A}^{\top}\|_2 = \|\mathbf{A}\|_2$), $(b)$ invokes Assumption~\ref{ass:standard}\,(a), and $(c)$ applies the upper bound of~\eqref{eq:isometry_general}. Finally, $\|\bU_{r_i}\|_2 = \|\bU_{r_i}^{\circ}\bLambda_{r_i}^{1/2}\mathbf{S}^{-1}\|_2 \leq \lambda_1^{1/2}\|\mathbf{S}^{-1}\|_2$ and $\|\bV_{r_i}\|_2 = \|\mathbf{S}\,\bLambda_{r_i}^{1/2}\bV_{r_i}^{\circ}\|_2 \leq \lambda_1^{1/2}\|\mathbf{S}\|_2$, so $\|\bU_{r_i}\|_2\,\|\bV_{r_i}\|_2 \leq \lambda_1\,\kappa(\mathbf{S})$. Combining,
\begin{equation}
\begin{split}
    &\|\nabla_{\bSigma} f_c(\bSigma)
      - \nabla_{\bSigma} f_c(\bSigma')\|_F \\
    &\quad\leq L\,\lambda_1^2\,\kappa(\mathbf{S})^2\,
         \|\bSigma - \bSigma'\|_F
    = L_{\mathrm{eff}}\,
         \|\bSigma - \bSigma'\|_F.
    \qedhere
\end{split}
\end{equation}
\end{proof}

\begin{lemma}[Local update drift bound]
\label{lem:drift}
Let client $c \in \mathcal{C}_i$ run $E \geq 1$ local SGD steps with constant step size $\eta \leq \dfrac{1}{8\,L_{\mathrm{eff}}\,E}$ starting from the round-$t$ aggregated group adapter $\bSigma_i^{(t)}$.  Denote the local iterates by $\bSigma_{i,c}^{(t,0)} = \bSigma_i^{(t)}$ and $\bSigma_{i,c}^{(t,e+1)} = \bSigma_{i,c}^{(t,e)} - \eta\,\widetilde{\mathbf{g}}_c(\bSigma_{i,c}^{(t,e)})$ for $e = 0, \ldots, E-1$. Then for every $0 \leq e \leq E$,
\begin{equation}
\begin{split}
    \sum_{c\in\mathcal{C}_i}\frac{n_c}{n_i}
    \mathbb{E}\bigl\|\bSigma_{i,c}^{(t,e)}
        - \bSigma_i^{(t)}\bigr\|_F^2
        &\leq
    4\eta^2 E\,(\sigma^2 + 6E\,G_i^2) \\
    &+ 8\eta^2 E^2\,
       \bigl\|\nabla f_i(\bSigma_i^{(t)})
       \bigr\|_F^2.
\end{split}
\label{eq:drift_bound}
\end{equation}
\end{lemma}

\begin{proof}
The bound is a standard result from the analysis of local SGD~\cite{karimireddy2020scaffold,koloskova2020unified}. Set $\mathbf{\Delta}_c^{(e)} \triangleq \bSigma_{i,c}^{(t,e)} - \bSigma_i^{(t)}$ and decompose $\widetilde{\mathbf{g}}_c(\bSigma_{i,c}^{(t,e)}) = \nabla f_i(\bSigma_i^{(t)}) + \bigl[\nabla f_c(\bSigma_{i,c}^{(t,e)}) - \nabla f_i(\bSigma_i^{(t)})\bigr] + \bigl[\widetilde{\mathbf{g}}_c(\bSigma_{i,c}^{(t,e)}) - \nabla f_c(\bSigma_{i,c}^{(t,e)})\bigr]$. Using $\|a+b+c\|^2 \leq 3(\|a\|^2 + \|b\|^2 + \|c\|^2)$, $L_{\mathrm{eff}}$-smoothness (Lemma~\ref{lem:smoothness}), and Assumption~\ref{ass:standard}\,(b)--(c), an induction on $e$ with the step-size condition $\eta \leq \frac{1}{8 L_{\mathrm{eff}} E}$ produces~\eqref{eq:drift_bound}.
\end{proof}

\subsubsection{Proof of Theorem~\ref{thm:convergence}}

\begin{proof}
We begin with a one-round descent. By Lemma~\ref{lem:smoothness},
\begin{equation}
\begin{split}
    f_i(\bSigma_i^{(t+1)})
    \leq\;
    &f_i(\bSigma_i^{(t)})
    + \langle \nabla f_i^{(t)},\;
        \bSigma_i^{(t+1)}
        - \bSigma_i^{(t)} \rangle \\
    &+ \frac{L_{\mathrm{eff}}}{2}
        \|\bSigma_i^{(t+1)}
        - \bSigma_i^{(t)}\|_F^2.
\end{split}
\label{eq:descent}
\end{equation}
We next bound the inner-product term. The server update is
\begin{equation}
    \bSigma_i^{(t+1)} - \bSigma_i^{(t)}
    = -\eta\sum_{c\in\mathcal{C}_i}\frac{n_c}{n_i}
       \sum_{e=0}^{E-1}
       \widetilde{\mathbf{g}}_c(\bSigma_{i,c}^{(t,e)}).
\end{equation}
Taking expectations and using Assumption~\ref{ass:standard}\,(b),
\begin{equation}
\begin{split}
    &\mathbb{E}\langle \nabla f_i^{(t)},\,
        \bSigma_i^{(t+1)}
        - \bSigma_i^{(t)}\rangle \\
    \;=& -\eta\!\sum_{c,e}\frac{n_c}{n_i}
       \mathbb{E}\langle \nabla f_i^{(t)},\,
         \nabla f_c(\bSigma_{i,c}^{(t,e)})\rangle \\
    \;=& -\eta E\,\|\nabla f_i^{(t)}\|_F^2 \\
    & -\eta\!\sum_{c,e}\frac{n_c}{n_i}
         \mathbb{E}\langle \nabla f_i^{(t)},\,
         \nabla f_c(\bSigma_{i,c}^{(t,e)})
           - \nabla f_i^{(t)}\rangle.
\end{split}
\label{eq:inner_raw}
\end{equation}
Applying Young's inequality $\langle a, b\rangle \geq -\tfrac{1}{2}\|a\|^2 - \tfrac{1}{2}\|b\|^2$ and decomposing $\nabla f_c(\bSigma_{i,c}^{(t,e)}) - \nabla f_i^{(t)} = [\nabla f_c(\bSigma_{i,c}^{(t,e)}) - \nabla f_c( \bSigma_i^{(t)})] + [\nabla f_c(\bSigma_i^{(t)}) - \nabla f_i(\bSigma_i^{(t)})]$, the second term above is controlled by $L_{\mathrm{eff}}$-smoothness (Lemma~\ref{lem:smoothness}), Lemma~\ref{lem:drift}, and Assumption~\ref{ass:standard}\,(c), giving
\begin{equation}
\begin{split}
    &\mathbb{E}\langle \nabla f_i^{(t)},\,
        \bSigma_i^{(t+1)}
        - \bSigma_i^{(t)}\rangle \\
    \leq&
    -\tfrac{\eta E}{2}\|\nabla f_i^{(t)}\|_F^2
    + \eta E\,G_i^2
    + L_{\mathrm{eff}}^{\,2}\,\eta^3 E^3
       (\sigma^2 + 6EG_i^2).
\end{split}
\label{eq:inner_bound}
\end{equation}
For the quadratic term, Jensen's inequality and convexity of the squared norm give
\begin{equation}
\begin{split}
    &\mathbb{E}\|\bSigma_i^{(t+1)}
        - \bSigma_i^{(t)}\|_F^2 \leq \eta^2 E\sum_{c,e}
       \frac{n_c}{n_i}
       \mathbb{E}\|
       \widetilde{\mathbf{g}}_c(\bSigma_{i,c}^{(t,e)})\|_F^2.
\end{split}
\label{eq:quad_raw}
\end{equation}
Decomposing each stochastic gradient as $\widetilde{\mathbf{g}}_c = \nabla f_i^{(t)} + [\nabla f_c(\bSigma_{i,c}^{(t,e)}) - \nabla f_i^{(t)}] + [\widetilde{\mathbf{g}}_c - \nabla f_c(\bSigma_{i,c}^{(t,e)})]$ and applying $\|a+b+c\|^2 \leq 3(\|a\|^2+\|b\|^2+\|c\|^2)$, Assumption~\ref{ass:standard}\,(b)--(c), $L_{\mathrm{eff}}$-smoothness (Lemma~\ref{lem:smoothness}), and (Lemma~\ref{lem:drift}) yield
\begin{equation}
\begin{split}
    &\mathbb{E}\|\bSigma_i^{(t+1)}
        - \bSigma_i^{(t)}\|_F^2 \\
    \leq&
    3\eta^2 E^2\,\|\nabla f_i^{(t)}\|_F^2
    + 3\eta^2 E^2 G_i^2
    + 3\eta^2 E\,\sigma^2 \\
    &
    + 12\,L_{\mathrm{eff}}^{\,2}\,\eta^4 E^3
       (\sigma^2 + 6E\,G_i^2) \\
    &
    + 24\,L_{\mathrm{eff}}^{\,2}\,\eta^4 E^4\,
       \|\nabla f_i^{(t)}\|_F^2.
\end{split}
\label{eq:quad}
\end{equation}
Substituting~\eqref{eq:inner_bound} and~\eqref{eq:quad} into~\eqref{eq:descent}, taking expectations, and using $\eta \leq \frac{1}{8 L_{\mathrm{eff}} E}$ to absorb the higher-order terms, we obtain the per-round inequality
\begin{equation}
\begin{split}
    \mathbb{E}[f_i^{(t+1)}]
    &\leq\;
    \mathbb{E}[f_i^{(t)}]
    - \tfrac{\eta E}{4}\,
       \mathbb{E}\|\nabla f_i^{(t)}\|_F^2
    + \tfrac{\eta E}{2}\,G_i^2 \\
    &+ L_{\mathrm{eff}}\,\eta^2 E\,\sigma^2
    + 2\,L_{\mathrm{eff}}^{\,2}\,\eta^3 E^3
         (\sigma^2 + 6EG_i^2).
\end{split}
\label{eq:round}
\end{equation}

Finally, summing~\eqref{eq:round} from $t = 0$ to $T_1 - 1$, dividing by $\eta E T_1 / 4$, and using $f_i^{(T_1)} \geq \inf_{\bSigma} f_i$ yields
\begin{equation}
\begin{split}
    \frac{1}{T_1}\sum_{t=0}^{T_1-1}
    \mathbb{E}\|\nabla f_i^{(t)}\|_F^2 \;&\leq\;
    \frac{4\Delta_0}{\eta E T_1}
    + 2 G_i^2
    + 4 L_{\mathrm{eff}}\,\eta\,\sigma^2 \\
    &\quad+ 8 L_{\mathrm{eff}}^{\,2}\,\eta^2 E^2
       (\sigma^2 + 6 E G_i^2),
\end{split}
\end{equation}
which matches~\eqref{eq:thm1_general}. Setting $\eta = \frac{1}{L_{\mathrm{eff}}\sqrt{T_1 E}}$ (which satisfies $\eta \leq \frac{1}{8 L_{\mathrm{eff}} E}$ whenever $T_1 \geq 64E$) produces the rate~\eqref{eq:thm1_rate}.
\end{proof}

\subsection{Rate Conversion for the Optimization Residual}
\label{app:rate_conversion}

We give the full derivation of the rate~\eqref{eq:opt_pl_rate} stated in Proposition~\ref{prop:residual_control}\,(a). By Lemma~\ref{lem:pl_qg}, the P{\L} inequality~\eqref{eq:pl} implies the quadratic-growth (QG) inequality
\begin{equation}
    f_i(\bSigma) - f_i^{*}
    \;\geq\;
    \tfrac{\mu}{2}\,
    \bigl\|\bSigma
        - \Pi_{\arg\min f_i}(\bSigma)\bigr\|_F^{\,2},
    \label{eq:qg}
\end{equation}
where $\Pi_{\arg\min f_i}$ denotes the Euclidean projection onto the minimizer set.
We first convert the gradient-norm bound into an objective gap.  Combining Theorem~\ref{thm:convergence} with~\eqref{eq:pl} gives
\begin{equation}
\begin{split}
    &\frac{1}{T_1}\sum_{t=0}^{T_1-1}
    \mathbb{E}\bigl[f_i(\bSigma_i^{(t)})
        - f_i^{*}\bigr]
    \\
    \leq&
    \frac{1}{2\mu}\!\cdot\!
    \frac{1}{T_1}\sum_{t=0}^{T_1-1}
    \mathbb{E}\bigl\|\nabla f_i(\bSigma_i^{(t)})
    \bigr\|_F^{\,2}
    \\
    =&
    \underbrace{\mathcal{O}\!\Bigl(
        \tfrac{L_{\mathrm{eff}}\Delta_0
            + \sigma^2}
              {\mu\sqrt{T_1 E}}\Bigr)}
        _{\text{vanishing}}
    +\;
    \underbrace{\tfrac{G^2}{\mu}}
        _{\text{floor}}.
\end{split}
\label{eq:rem_step1}
\end{equation}
Applying~\eqref{eq:qg} to $\bSigma_i^{(T_1)}$ and Jensen's inequality ($t \mapsto \sqrt{t}$ is concave) then turns the objective gap into an iterate-distance bound:
\begin{equation}
\begin{split}
    &\mathbb{E}\bigl\|\bSigma_i^{(T_1)}
        - \widehat{\bSigma}_i\bigr\|_F \\
    \;\leq\;&
    \sqrt{\tfrac{2}{\mu}\,
        \mathbb{E}\bigl[f_i(\bSigma_i^{(T_1)})
        - f_i^{*}\bigr]} \\
    \;=\;&
    \mathcal{O}\!\bigl(T_1^{-1/4}\bigr)
    + \mathcal{O}\!\bigl(G/\sqrt{\mu}\bigr).
\end{split}
    \label{eq:rem_step2}
\end{equation}
Finally, lifting from adapter space to weight space via Lemma~\ref{lem:isometry}, $\|\bU_{r_i} A \bV_{r_i}\|_F \leq \lambda_1\,\kappa(\mathbf{S})\,\|A\|_F$, so
\begin{equation}
    \mathbb{E}\|\varepsilon_i^{\mathrm{opt}}\|_F
    \;\leq\;
    \lambda_1\,\kappa(\mathbf{S})\!\cdot\!
    \Bigl(\mathcal{O}(T_1^{-1/4})
        + \mathcal{O}(G/\sqrt{\mu})\Bigr).
\end{equation}

The exponent $-1/4$ is the exact price of converting a second-moment gradient bound into a first-moment iterate bound: one factor of $1/2$ from the square root in~\eqref{eq:qg}, another from Jensen's inequality. In the non-IID regime, the additive $\mathcal{O}(\lambda_1\,\kappa(\mathbf{S})\,G/\sqrt{\mu})$ term should be absorbed into $\delta_{r_i}(\widehat{\bW})$; the qualitative conclusions of Theorem~\ref{thm:error-bound} and Proposition~\ref{prop:bias-var} are unchanged.

\subsection{Proofs for Subspace Alignment Theory}
\label{app:subspace_proofs}

\begin{proof}[Proof of Proposition~\ref{prop:nesting}]
Since the singular vectors are ordered by decreasing singular value, $\{\mathbf{u}_1, \ldots, \mathbf{u}_{r_1}\} \subset \{\mathbf{u}_1, \ldots, \mathbf{u}_{r_2}\}$ for $r_1 < r_2$, so $\mathcal{U}_{r_1} \subsetneq \mathcal{U}_{r_2}$ follows immediately.  The same argument applies to the right factors.  In matrix form, with the orthonormal factors,
\begin{equation}
    \bU_{r_1}^{\circ}
    \;=\;
    \bU_{r_2}^{\circ}
    \begin{bmatrix}\mathbf{I}_{r_1}\\ \mathbf{0}\end{bmatrix},
    \qquad
    \bV_{r_1}^{\circ}
    \;=\;
    \begin{bmatrix}\mathbf{I}_{r_1} & \mathbf{0}\end{bmatrix}
    \bV_{r_2}^{\circ},
\end{equation}
so
\begin{equation}
\begin{split}
    \mathbf{P}_{r_1}\mathbf{P}_{r_2}
    &= \bU_{r_1}^{\circ}\bU_{r_1}^{\circ\,\top}
       \bU_{r_2}^{\circ}\bU_{r_2}^{\circ\,\top}
    = \bU_{r_1}^{\circ}[\mathbf{I}_{r_1};\mathbf{0}]
       \bU_{r_2}^{\circ\,\top} \\
    &= \bU_{r_1}^{\circ}\bU_{r_1}^{\circ\,\top}
    = \mathbf{P}_{r_1}.
\end{split}
\end{equation}
The argument for $\mathbf{Q}$ is analogous.
\end{proof}

\subsection{Proof of Theorem~\ref{thm:relaxed_error_bound}}
\label{app:relaxation}

We give the full proof of the general cross-group aggregation bound (Theorem~\ref{thm:relaxed_error_bound}) and Proposition~\ref{prop:residual_control} (residual control under P{\L} and the common-projection-target conditions). The companion identity
\begin{equation}
    \bU_{r_i}\widehat{\bSigma}_i\bV_{r_i}
    \;=\;
    \bU_{r_i}\bU_{r_i}^{\dagger}\,\widehat{\bW}\,\bV_{r_i}^{\dagger}\bV_{r_i}
    \;=\;
    \mathbf{P}_{r_i}\widehat{\bW}\mathbf{Q}_{r_i}
    \label{eq:pinv_identity}
\end{equation}
holds unconditionally by Lemma~\ref{lem:P_Q_pinv}; we use it freely throughout.

\begin{proof}[Proof of Theorem~\ref{thm:relaxed_error_bound}]
Adding and subtracting $\bU_{r_i}\bSigma_i^{*}\bV_{r_i}$ and $\mathbf{P}_{r_i}\widehat{\bW}\mathbf{Q}_{r_i}$ inside $\bW_i - \widehat{\bW}$ produces the telescoping identity
\begin{equation}
\begin{split}
    \bW_i - \widehat{\bW}
    \;=\;
    \varepsilon_i^{\mathrm{opt}}
    + \varepsilon_i^{\mathrm{mis}}
    + \varepsilon_i^{\mathrm{sub}},
\end{split}
\label{eq:three_term}
\end{equation}
with the three residuals defined as in the main text. Each adjacent pair cancels by direct subtraction. Apply the triangle inequality and convexity of the Frobenius norm to $\bW_{\mathrm{g}} - \widehat{\bW} = \sum_i (n_i/N)(\bW_i - \widehat{\bW})$, and bound each $\|\bW_i - \widehat{\bW}\|_F$ using~\eqref{eq:three_term} and Definition~\ref{def:residual}; this gives~\eqref{eq:thm_relaxed_full}.
\end{proof}

\begin{proof}[Proof of Proposition~\ref{prop:residual_control}\,(a)]
Under Assumption~\ref{ass:standard} and the P{\L} inequality~\eqref{eq:pl}, the quadratic-growth inequality (Lemma~\ref{lem:pl_qg}) gives, for each group $i$ and each iterate $\bSigma$,
\begin{equation}
    f_i(\bSigma) - f_i^{*}
    \;\geq\;
    \tfrac{\mu}{2}\bigl\|\bSigma - \Pi_{\arg\min f_i}(\bSigma)\bigr\|_F^{\,2}.
\end{equation}
Combining with the Stage-1 squared-gradient-norm bound of Theorem~\ref{thm:convergence} and applying Jensen's inequality yields $\mathbb{E}\|\bSigma_i^{(T_1)} - \bSigma_i^{*}\|_F = \mathcal{O}(T_1^{-1/4}) + \mathcal{O}(G/\sqrt{\mu})$ (full derivation in Appendix~\ref{app:rate_conversion}). Lifting from adapter space to weight space via the upper isometry of Lemma~\ref{lem:isometry} introduces a factor $\lambda_1\,\kappa(\mathbf{S})$ and produces~\eqref{eq:opt_pl_rate}.
\end{proof}

\begin{proof}[Proof of Proposition~\ref{prop:residual_control}\,(b)]
If $\widehat{\bSigma}_i \in \arg\min f_i$ for every $i$, then $f_i(\widehat{\bSigma}_i) = f_i^{*}$, so $\Delta_i(\widehat{\bW}) = 0$. When $\arg\min f_i$ is a singleton (the regime guaranteed by part~(a) plus uniqueness; otherwise pick the component visited by Stage-1 SGD), this forces $\widehat{\bSigma}_i = \bSigma_i^{*}$ and hence $\varepsilon_i^{\mathrm{mis}} = \bU_{r_i}(\bSigma_i^{*} - \widehat{\bSigma}_i)\bV_{r_i} = \mathbf{0}$. The general bound~\eqref{eq:mis_pl_bound} (under part~(a)) follows from the quadratic-growth inequality applied at $\widehat{\bSigma}_i$ rather than at the iterate,
\begin{equation}
    \bigl\|\bSigma_i^{*} - \widehat{\bSigma}_i\bigr\|_F
    \;\leq\;
    \sqrt{\tfrac{2\,\Delta_i(\widehat{\bW})}{\mu}} ,
    \label{eq:misalign_iterate}
\end{equation}
combined with~\eqref{eq:pinv_identity} (so $\varepsilon_i^{\mathrm{mis}} = \bU_{r_i}(\bSigma_i^{*} - \widehat{\bSigma}_i)\bV_{r_i}$) and the upper isometry of Lemma~\ref{lem:isometry}, which contributes a factor $\lambda_1\,\kappa(\mathbf{S})$.
\end{proof}

When the common-projection-target condition fails, the behaviour of $\varepsilon_i^{\mathrm{mis}}$ is characterized through three complementary results. \emph{(i) Sufficient condition} (Proposition~\ref{prop:cpt_sufficient}): under a positive-definite quadratic surrogate, two structural properties at $\widehat{\bW}$ together make $\widehat{\bSigma}_i$ a loss minimizer: the Hessian commutes with the SVD projector, and the loss gradient has no retained-subspace component. When both hold exactly, $\varepsilon_i^{\mathrm{mis}} = \mathbf{0}$. \emph{(ii) Curvature-side perturbation bound} (Proposition~\ref{prop:cpt_perturbation}): when the two structural properties hold only approximately, with explicit budgets $\varepsilon_{G,i}$ (gradient leakage), $\varepsilon_{H,i}$ (curvature--projector commutator), and $\rho_i$ (Hessian variation), restricted strong convexity yields a closed-form bound on $\|\varepsilon_i^{\mathrm{mis}}\|_F$ scaling linearly with $\varepsilon_{G,i} + \varepsilon_{H,i}\delta_{r_i}(\widehat{\bW})$. The bound recovers $\varepsilon_i^{\mathrm{mis}} = \mathbf{0}$ when all budgets vanish. \emph{(iii) Data-side decomposition} (Proposition~\ref{prop:hetero_decomp}): an independent route decomposes the loss gap $\Delta_i(\widehat{\bW})$ into a statistical heterogeneity term $H_i(\widehat{\bW})$ and a per-group compression term $\delta_{r_i}(\bW_i^{\mathrm{loc}})$, using only $L$-smoothness and group-optimum statistics without curvature information. The two routes are complementary: neither subsumes the other. Remark~\ref{rem:two_routes} below discusses when each is the operative tool.

To state the curvature-aligned condition in plain geometric terms we use a curvature-weighted norm. For a self-adjoint positive-definite operator $\mathcal{H} : \mathbb{R}^{m \times n} \to \mathbb{R}^{m \times n}$, write $\|\bM\|_{\mathcal{H}}^{2} \triangleq \langle \bM,\, \mathcal{H}[\bM]\rangle_F$. Intuitively this measures the size of $\bM$ with each direction weighted by how sharply the loss bends along it; the plain Frobenius norm ($\mathcal{H} = \mathrm{Id}$) weights all directions equally. It is the matrix analogue of the weighted norm $\|\cdot\|_{H}$ that turns the scalar descent lemma into its anisotropic form.

\begin{proposition}[Sufficient condition: aligned curvature and subspace-stationary reference]
\label{prop:cpt_sufficient}
Fix the reference weight $\widehat{\bW}$ (e.g.\ the pre-trained weight). Suppose that, for each group $i \in [K]$, the ambient loss is a positive-definite quadratic, written as its second-order Taylor expansion \emph{around the reference} $\widehat{\bW}$,
\begin{equation}
    \mathcal{L}_i(\bW)
    \;=\;
    \mathcal{L}_i(\widehat{\bW})
    + \langle \mathbf{G}_i,\, \bW - \widehat{\bW}\rangle_F
    + \tfrac{1}{2}\,\bigl\|\bW - \widehat{\bW}\bigr\|_{\mathcal{H}_i}^{2},
    \label{eq:aniso_quad}
\end{equation}
where $\mathbf{G}_i \triangleq \nabla\mathcal{L}_i(\widehat{\bW})$ and $\mathcal{H}_i$ is the Hessian. Both are evaluated at the \emph{known} point $\widehat{\bW}$ and are estimable from the first gradient/curvature probe at initialization; in particular $\widehat{\bW}$ is \emph{not} assumed to minimize $\mathcal{L}_i$, so $\mathbf{G}_i$ is generally nonzero. The curvature acts as a \emph{direction-dependent stiffness}, generalizing the scalar smoothness constant $L$ to an operator recording how sharply the loss bends along each direction. Let $\Pi_{r_i}[\cdot] \triangleq \mathbf{P}_{r_i}(\cdot)\mathbf{Q}_{r_i}$ be the orthogonal projector onto $\mathcal{M}_{r_i} \triangleq \{\bA : \mathbf{P}_{r_i}\bA\mathbf{Q}_{r_i} = \bA\}$, the set of weights the rank-$r_i$ compressed model can represent. If
\begin{enumerate}[leftmargin=2.5em]
\item[\textbf{(a)}] the curvature keeps the retained SVD directions decoupled from the discarded ones, equivalently it commutes with the projector,
\begin{equation}
    \mathcal{H}_i \circ \Pi_{r_i}
    \;=\;
    \Pi_{r_i} \circ \mathcal{H}_i,
    \label{eq:curv_align}
\end{equation}
\item[\textbf{(b)}] the loss gradient at the reference has no component inside the retained subspace,
\begin{equation}
    \Pi_{r_i}[\mathbf{G}_i] = \mathbf{P}_{r_i}\,\nabla\mathcal{L}_i(\widehat{\bW})\,\mathbf{Q}_{r_i} = \mathbf{0},
    \label{eq:hetero_confined}
\end{equation}
\end{enumerate}
then the common-projection-target condition~\eqref{eq:sigma_star_min} holds: $\widehat{\bSigma}_i \in \arg\min_{\bSigma} f_i$ for every $i \in [K]$, and consequently $\varepsilon_i^{\mathrm{mis}} = \mathbf{0}$. Both hypotheses are testable at initialization from $\mathbf{G}_i$ and $\mathcal{H}_i$ alone, without knowledge of any optimizer.
\end{proposition}

\begin{proof}[Proof of Proposition~\ref{prop:cpt_sufficient}]
Throughout, $\Pi_{r_i}[\mathbf{M}] = \mathbf{P}_{r_i}\mathbf{M}\mathbf{Q}_{r_i}$ is the orthogonal projection (in the Frobenius inner product) onto the subspace of representable weights $\mathcal{M}_{r_i} = \{\mathbf{A} : \mathbf{P}_{r_i}\mathbf{A}\mathbf{Q}_{r_i} = \mathbf{A}\}$. Recall it is idempotent, $\Pi_{r_i}^2 = \Pi_{r_i}$, and symmetric, $\langle\Pi_{r_i}\mathbf{A},\mathbf{B}\rangle_F = \langle\mathbf{A},\Pi_{r_i}\mathbf{B}\rangle_F$.

By hypothesis $\mathcal{L}_i(\mathbf{A}) = \mathcal{L}_i(\widehat{\bW}) + \langle\mathbf{G}_i,\mathbf{A}-\widehat{\bW}\rangle_F + \tfrac12\langle\mathbf{A}-\widehat{\bW},\,\mathcal{H}_i[\mathbf{A}-\widehat{\bW}]\rangle_F$ with $\mathbf{G}_i = \nabla\mathcal{L}_i(\widehat{\bW})$ and $\mathcal{H}_i$ symmetric positive definite, so $\mathcal{L}_i$ is a strictly convex quadratic with gradient $\nabla\mathcal{L}_i(\mathbf{A}) = \mathbf{G}_i + \mathcal{H}_i[\mathbf{A}-\widehat{\bW}]$ and a unique minimizer on $\mathcal{M}_{r_i}$. A point $\mathbf{A}\in\mathcal{M}_{r_i}$ is that minimizer exactly when the gradient is orthogonal to all of $\mathcal{M}_{r_i}$, i.e.\ when its projection onto the subspace vanishes:
\begin{equation}
    \Pi_{r_i}\bigl[\mathbf{G}_i + \mathcal{H}_i[\mathbf{A}-\widehat{\bW}]\bigr] = \mathbf{0}.
    \label{eq:variational}
\end{equation}

We now check that the Frobenius projection of the reference, $\mathbf{A} = \Pi_{r_i}[\widehat{\bW}]$, satisfies~\eqref{eq:variational}. By linearity of $\Pi_{r_i}$ the left-hand side splits into a gradient term and a curvature term,
\begin{equation*}
    \Pi_{r_i}\bigl[\mathbf{G}_i + \mathcal{H}_i[\mathbf{A}-\widehat{\bW}]\bigr]
    = \underbrace{\Pi_{r_i}[\mathbf{G}_i]}_{=\,\mathbf{0}\ \text{by~\eqref{eq:hetero_confined}}}
    + \Pi_{r_i}\bigl[\mathcal{H}_i[\mathbf{A}-\widehat{\bW}]\bigr],
\end{equation*}
so it remains to show the curvature term also vanishes. The residual of the candidate is
\begin{equation*}
    \mathbf{A}-\widehat{\bW} = \Pi_{r_i}[\widehat{\bW}]-\widehat{\bW} = -(\mathrm{Id}-\Pi_{r_i})[\widehat{\bW}],
\end{equation*}
which lies in the orthogonal complement of $\mathcal{M}_{r_i}$: projecting it back gives $\bigl(\Pi_{r_i}\circ(\mathrm{Id}-\Pi_{r_i})\bigr)[\widehat{\bW}] = (\Pi_{r_i}-\Pi_{r_i}^2)[\widehat{\bW}] = \mathbf{0}$ by idempotence. Applying the alignment hypothesis~\eqref{eq:curv_align}, $\mathcal{H}_i\circ\Pi_{r_i} = \Pi_{r_i}\circ\mathcal{H}_i$, to pass the projection through the curvature,
\begin{equation*}
\begin{split}
    \Pi_{r_i}\bigl[\mathcal{H}_i[\mathbf{A}-\widehat{\bW}]\bigr]
    &= -\bigl(\Pi_{r_i}\circ\mathcal{H}_i\circ(\mathrm{Id}-\Pi_{r_i})\bigr)[\widehat{\bW}] \\
    &= -\,\mathcal{H}_i\Bigl[\,\underbrace{\bigl(\Pi_{r_i}\circ(\mathrm{Id}-\Pi_{r_i})\bigr)[\widehat{\bW}]}_{=\,\mathbf{0}}\,\Bigr]
    = \mathbf{0},
\end{split}
\end{equation*}
the second equality moving $\Pi_{r_i}$ past $\mathcal{H}_i$ by the hypothesis. Both terms of~\eqref{eq:variational} thus vanish, so $\Pi_{r_i}[\widehat{\bW}]$ is the minimizer of $\mathcal{L}_i$ on $\mathcal{M}_{r_i}$ by uniqueness.

It remains to translate this back to adapters. By~\eqref{eq:pinv_identity} the minimizer $\Pi_{r_i}[\widehat{\bW}] = \mathbf{P}_{r_i}\widehat{\bW}\mathbf{Q}_{r_i}$ equals $\bU_{r_i}\widehat{\bSigma}_i\bV_{r_i}$, so it is realized by the geometric adapter $\widehat{\bSigma}_i$. Pulling back through the bijection, $\widehat{\bSigma}_i$ minimizes $f_i$, which is the common-projection-target condition~\eqref{eq:sigma_star_min}; as the minimizer is unique, $\bSigma_i^{*} = \widehat{\bSigma}_i$ and hence $\varepsilon_i^{\mathrm{mis}} = \bU_{r_i}(\bSigma_i^{*}-\widehat{\bSigma}_i)\bV_{r_i} = \mathbf{0}$.

For the separable curvature $\mathcal{H}_i[\mathbf{M}] = \mathbf{A}_i\mathbf{M}\mathbf{C}_i$ of the main text, the alignment hypothesis~\eqref{eq:curv_align} holds as soon as $\mathbf{A}_i\mathbf{P}_{r_i} = \mathbf{P}_{r_i}\mathbf{A}_i$ and $\mathbf{C}_i\mathbf{Q}_{r_i} = \mathbf{Q}_{r_i}\mathbf{C}_i$ (equivalently, the left and right singular vectors of $\widehat{\bW}$ are eigenvectors of the matrices $\mathbf{A}_i$ and $\mathbf{C}_i$), since then for every $\mathbf{M}$
\begin{equation*}
\begin{split}
    (\mathcal{H}_i\circ\Pi_{r_i})[\mathbf{M}]
    &= \mathbf{A}_i\mathbf{P}_{r_i}\mathbf{M}\mathbf{Q}_{r_i}\mathbf{C}_i
     = \mathbf{P}_{r_i}\mathbf{A}_i\mathbf{M}\mathbf{C}_i\mathbf{Q}_{r_i} \\
    &= (\Pi_{r_i}\circ\mathcal{H}_i)[\mathbf{M}],
\end{split}
\end{equation*}
using $\mathbf{A}_i\mathbf{P}_{r_i} = \mathbf{P}_{r_i}\mathbf{A}_i$ and $\mathbf{C}_i\mathbf{Q}_{r_i} = \mathbf{Q}_{r_i}\mathbf{C}_i$ in the middle step.
The isotropic loss ($\mathbf{A}_i = a_i\mathbf{I}_m$, $\mathbf{C}_i = \mathbf{I}_n$) is the trivial case, where every matrix commutes with the projection matrices.
\end{proof}

The two hypotheses control the two ways the loss optimum can drift away from the geometric projection $\Pi_{r_i}[\widehat{\bW}]$, and both are read off at the reference $\widehat{\bW}$. 
Condition~\eqref{eq:hetero_confined} is a \emph{subspace-stationarity} requirement: at $\widehat{\bW}$, the loss has no first-order incentive to move inside the retained subspace, so all of its gradient pressure points along the discarded directions. Equivalently, writing $\mathbf{G}_i = \mathcal{H}_i[\widehat{\bW} - \bW_i^{\mathrm{loc}}]$ for the unconstrained minimizer $\bW_i^{\mathrm{loc}}$, the condition says the group heterogeneity $\widehat{\bW} - \bW_i^{\mathrm{loc}}$ lies entirely along the discarded directions (where compression removes it anyway), without forcing $\bW_i^{\mathrm{loc}} = \widehat{\bW}$ (the no-heterogeneity case $H_i = 0$ of Proposition~\ref{prop:hetero_decomp}). Condition~\eqref{eq:curv_align} controls the metric: it annihilates the off-block coupling $\bigl(\Pi_{r_i}\circ\mathcal{H}_i\circ(\mathrm{Id}-\Pi_{r_i})\bigr)$ through which the curvature would otherwise pull the loss optimum off the Frobenius projection. It constrains only the curvature \emph{directions}, not magnitudes, so it permits arbitrary, group-dependent, ill-conditioned anisotropy along the singular directions of $\widehat{\bW}$; whitened/isotropic curvature is the trivial case where the coupling vanishes automatically.

A concrete example is the Gauss--Newton/Fisher curvature of a layer, which K-FAC~\cite{martens2015kfac} approximates in the separable form $\mathcal{H}_i[\bM] = \bA_i\bM\bC_i$ with $\bA_i$ and $\bC_i$ the output-gradient and input-activation covariances. Condition~\eqref{eq:curv_align} then holds whenever these K-FAC factors are co-diagonalizable with the layer's SVD subspaces, i.e.\ the curvature shares its principal axes with $\widehat{\bW}$, while its magnitude along each axis may vary arbitrarily and unequally across groups. The isotropic loss is the trivial case in which this alignment is automatic.

For curvature that genuinely couples the retained subspace to its complement (the typical regime in federated fine-tuning, where $\mathcal{H}_i$ has no reason to respect the SVD geometry of $\widehat{\bW}$), condition~\eqref{eq:curv_align} fails and the loss-projection departs from the Frobenius-projection. Proposition~\ref{prop:cpt_sufficient} therefore serves as a clean limiting case rather than a typical operating point, and on its own provides no quantitative handle on the misalignment when its hypotheses are violated. Two further questions remain. First, the conditions~\eqref{eq:curv_align}--\eqref{eq:hetero_confined} almost never hold exactly: how should the misalignment grow as they are violated by small but non-zero amounts? Second, the loss landscape of foundation-model fine-tuning is non-convex, so the quadratic form~\eqref{eq:aniso_quad} is at best a second-order Taylor expansion about $\widehat{\bW}$; the cubic remainder must be controlled before the conclusion is meaningful for the true loss. The next proposition addresses both: it relaxes the two structural conditions to approximate forms with explicit error budgets, augments them with a Hessian-variation bound that absorbs the truncation error from the quadratic surrogate, and yields a perturbation bound on $\varepsilon_i^{\mathrm{mis}}$ that recovers Proposition~\ref{prop:cpt_sufficient} exactly when all errors vanish.

\begin{proposition}[Perturbation-robust misalignment bound]
\label{prop:cpt_perturbation}
Fix the reference weight $\widehat{\bW}$, and let $\mathcal{L}_i : \mathbb{R}^{m \times n} \to \mathbb{R}$ be twice continuously differentiable. Let $\Pi_{r_i}$ denote the orthogonal projector onto $\mathcal{M}_{r_i}$, write $\widehat{\bW}_{\perp} \triangleq (\mathrm{Id} - \Pi_{r_i})[\widehat{\bW}]$ and $R_i \triangleq \|\widehat{\bW}_{\perp}\|_F = \delta_{r_i}(\widehat{\bW})$, and set $\widehat{\bA}_i \triangleq \Pi_{r_i}[\widehat{\bW}]$. Define the perturbation budgets
\begin{align}
    \varepsilon_{G,i} &\triangleq \bigl\|\Pi_{r_i}\bigl[\nabla\mathcal{L}_i(\widehat{\bW})\bigr]\bigr\|_F, \label{eq:eps_G} \\
    \varepsilon_{H,i} &\triangleq \bigl\|\mathcal{H}_i\circ\Pi_{r_i} - \Pi_{r_i}\circ\mathcal{H}_i\bigr\|_{\mathrm{op}}, \label{eq:eps_H}
\end{align}
where $\mathcal{H}_i \triangleq \nabla^2\mathcal{L}_i(\widehat{\bW})$ and $\|\cdot\|_{\mathrm{op}}$ denotes the operator norm induced by $\|\cdot\|_F$. Suppose:
\begin{enumerate}[leftmargin=2.5em]
\item[\textbf{(c1)}] (\textbf{Restricted strong convexity on $\mathcal{M}_{r_i}$.}) There exists $\mu_i > 0$ such that for all $\bA, \bB \in \mathcal{M}_{r_i}$,
\begin{equation}
    \bigl\langle\nabla\mathcal{L}_i(\bA) - \nabla\mathcal{L}_i(\bB),\, \bA - \bB\bigr\rangle_F \;\geq\; \mu_i\,\|\bA - \bB\|_F^{\,2}.
    \label{eq:rsc}
\end{equation}
\item[\textbf{(c2)}] (\textbf{Approximate subspace-stationarity.}) The reference gradient has small retained-subspace component, $\varepsilon_{G,i} < \infty$.
\item[\textbf{(c3)}] (\textbf{Approximate curvature alignment.}) The Hessian at $\widehat{\bW}$ approximately commutes with the projector, $\varepsilon_{H,i} < \infty$.
\item[\textbf{(c4)}] (\textbf{Hessian variation / cubic remainder control.}) There exists $\rho_i \geq 0$ and a neighbourhood $\mathcal{N}_i \supset \{\widehat{\bW} + s(\bA - \widehat{\bW}) : \bA \in \mathcal{M}_{r_i},\ \|\bA - \widehat{\bA}_i\|_F \leq R_i,\ s \in [0,1]\}$ such that for every $\bX \in \mathcal{N}_i$,
\begin{equation}
    \bigl\|\nabla^2\mathcal{L}_i(\bX) - \mathcal{H}_i\bigr\|_{\mathrm{op}} \;\leq\; \rho_i\,\|\bX - \widehat{\bW}\|_F.
    \label{eq:hess_var}
\end{equation}
\end{enumerate}
Let $\bA_i^{*} \triangleq \arg\min_{\bA \in \mathcal{M}_{r_i}}\mathcal{L}_i(\bA)$ (unique by~\eqref{eq:rsc}), and assume the stability condition $4\rho_i\bigl(\varepsilon_{G,i} + \varepsilon_{H,i}R_i\bigr) < \mu_i^{2}$. Then
\begin{equation}
\begin{split}
    \bigl\|\bA_i^{*} - \widehat{\bA}_i\bigr\|_F
    \;\leq\;
    &\frac{\varepsilon_{G,i} + \varepsilon_{H,i}\,R_i}{\mu_i}
    \;+\;
    \frac{\rho_i\,R_i^{2}}{2\,\mu_i} \\
    &\;+\;
    \mathcal{O}\!\left(\frac{\rho_i}{\mu_i^{3}}\bigl(\varepsilon_{G,i} + \varepsilon_{H,i}R_i\bigr)^{2}\right),
\end{split}
\label{eq:perturb_bound}
\end{equation}
and consequently the misalignment residual is bounded by
\begin{equation}
\begin{split}
    \bigl\|\varepsilon_i^{\mathrm{mis}}\bigr\|_F
    \;\leq\;
    &\frac{\varepsilon_{G,i} + \varepsilon_{H,i}\,R_i}{\mu_i}
    \;+\;
    \frac{\rho_i\,R_i^{2}}{2\,\mu_i} \\
    &\;+\;
    \mathcal{O}\!\left(\frac{\rho_i}{\mu_i^{3}}\bigl(\varepsilon_{G,i} + \varepsilon_{H,i}R_i\bigr)^{2}\right).
\end{split}
\label{eq:eps_mis_perturb}
\end{equation}
When $\varepsilon_{G,i} = \varepsilon_{H,i} = \rho_i = 0$, the bound reads $\|\bA_i^{*} - \widehat{\bA}_i\|_F = 0$, recovering Proposition~\ref{prop:cpt_sufficient} as the exact-equality limit.
\end{proposition}

\begin{proof}[Proof of Proposition~\ref{prop:cpt_perturbation}]
We work directly in the matrix space $\mathbb{R}^{m\times n}$ equipped with the Frobenius inner product. Let $\bA_i^{*} \in \mathcal{M}_{r_i}$ be the unique minimizer (uniqueness is immediate from~\eqref{eq:rsc}) and write $\Delta \triangleq \bA_i^{*} - \widehat{\bA}_i \in \mathcal{M}_{r_i}$, so $\Pi_{r_i}[\Delta] = \Delta$.

Since $\bA_i^{*}$ minimizes $\mathcal{L}_i$ on $\mathcal{M}_{r_i}$, the gradient of $\mathcal{L}_i$ at $\bA_i^{*}$ is Frobenius-orthogonal to every direction in $\mathcal{M}_{r_i}$, equivalently
\begin{equation}
    \Pi_{r_i}\bigl[\nabla\mathcal{L}_i(\bA_i^{*})\bigr] = \mathbf{0}.
    \label{eq:pert_KKT}
\end{equation}

For any $\bX \in \mathbb{R}^{m\times n}$, the integral form of Taylor's theorem applied to $\nabla\mathcal{L}_i$ along the segment $\widehat{\bW} \to \bX$ gives
\begin{equation}
    \nabla\mathcal{L}_i(\bX) \;=\; \nabla\mathcal{L}_i(\widehat{\bW}) + \mathcal{H}_i[\bX - \widehat{\bW}] + \mathbf{R}(\bX),
    \label{eq:pert_taylor}
\end{equation}
with remainder $\mathbf{R}(\bX) \triangleq \int_0^1\bigl(\nabla^2\mathcal{L}_i(\widehat{\bW} + s(\bX-\widehat{\bW})) - \mathcal{H}_i\bigr)[\bX - \widehat{\bW}]\,ds$. Applying~\eqref{eq:hess_var} pointwise on the segment, which lies in $\mathcal{N}_i$ for $\bX \in \widehat{\bA}_i + \mathrm{ball}(R_i)$ on $\mathcal{M}_{r_i}$, yields the size bound
\begin{equation}
\begin{split}
    \|\mathbf{R}(\bX)\|_F
    &\;\leq\; \rho_i \int_0^1 s\,\|\bX - \widehat{\bW}\|_F\;ds \cdot \|\bX - \widehat{\bW}\|_F \\
    &\;=\; \tfrac{\rho_i}{2}\,\|\bX - \widehat{\bW}\|_F^{\,2}.
\end{split}
\label{eq:pert_remainder}
\end{equation}
Setting $\bX = \bA_i^{*} = \widehat{\bA}_i + \Delta$ and using $\bA_i^{*} - \widehat{\bW} = (\widehat{\bA}_i - \widehat{\bW}) + \Delta = -\widehat{\bW}_{\perp} + \Delta$:
\begin{equation}
    \nabla\mathcal{L}_i(\bA_i^{*}) \;=\; \mathbf{G}_i \;-\; \mathcal{H}_i[\widehat{\bW}_{\perp}] \;+\; \mathcal{H}_i[\Delta] \;+\; \mathbf{R}(\bA_i^{*}),
    \label{eq:pert_grad_expand}
\end{equation}
where $\mathbf{G}_i = \nabla\mathcal{L}_i(\widehat{\bW})$ and $\|\mathbf{R}(\bA_i^{*})\|_F \leq \tfrac{\rho_i}{2}\|{-\widehat{\bW}_{\perp}}+\Delta\|_F^{\,2} \leq \tfrac{\rho_i}{2}(R_i + \|\Delta\|_F)^{2}$.

Applying $\Pi_{r_i}$ to~\eqref{eq:pert_grad_expand} and combining with~\eqref{eq:pert_KKT}:
\begin{equation}
\begin{split}
    \mathbf{0} \;=\;
    &\Pi_{r_i}[\mathbf{G}_i] \;-\; \Pi_{r_i}\!\bigl[\mathcal{H}_i[\widehat{\bW}_{\perp}]\bigr] \\
    &+\; \Pi_{r_i}\!\bigl[\mathcal{H}_i[\Delta]\bigr] \;+\; \Pi_{r_i}[\mathbf{R}(\bA_i^{*})].
\end{split}
\label{eq:pert_proj_KKT}
\end{equation}
The middle term is the only place the alignment defect can enter. Using the commutator identity, $\Pi_{r_i}\circ\mathcal{H}_i = \mathcal{H}_i\circ\Pi_{r_i} + (\Pi_{r_i}\circ\mathcal{H}_i - \mathcal{H}_i\circ\Pi_{r_i})$, and noting that $\Pi_{r_i}[\widehat{\bW}_{\perp}] = \mathbf{0}$ (since $\widehat{\bW}_{\perp}$ is in the orthogonal complement of $\mathcal{M}_{r_i}$),
\begin{equation}
\begin{split}
    \Pi_{r_i}\!\bigl[\mathcal{H}_i[\widehat{\bW}_{\perp}]\bigr]
    &= \mathcal{H}_i\!\bigl[\Pi_{r_i}[\widehat{\bW}_{\perp}]\bigr] \\
    &\quad + (\Pi_{r_i}\circ\mathcal{H}_i - \mathcal{H}_i\circ\Pi_{r_i})[\widehat{\bW}_{\perp}] \\
    &= (\Pi_{r_i}\circ\mathcal{H}_i - \mathcal{H}_i\circ\Pi_{r_i})[\widehat{\bW}_{\perp}],
\end{split}
\label{eq:pert_commutator}
\end{equation}
whose Frobenius norm is at most $\varepsilon_{H,i}\,\|\widehat{\bW}_{\perp}\|_F = \varepsilon_{H,i}\,R_i$ by definition of the operator norm and~\eqref{eq:eps_H}. Substituting back:
\begin{equation}
\begin{split}
    \Pi_{r_i}\!\bigl[\mathcal{H}_i[\Delta]\bigr] \;=\;
    &-\Pi_{r_i}[\mathbf{G}_i] \; -\; \Pi_{r_i}[\mathbf{R}(\bA_i^{*})]\\
    &+\; (\Pi_{r_i}\circ\mathcal{H}_i - \mathcal{H}_i\circ\Pi_{r_i})[\widehat{\bW}_{\perp}].
\end{split}
\label{eq:pert_master}
\end{equation}

The condition~\eqref{eq:rsc} applied to $\bA = \widehat{\bA}_i + t\Delta$ and $\bB = \widehat{\bA}_i$ in $\mathcal{M}_{r_i}$, divided by $t^2$ and passed to the limit $t \to 0$, yields the infinitesimal form
\begin{equation}
    \bigl\langle \Delta,\, \nabla^2\mathcal{L}_i(\widehat{\bA}_i)[\Delta]\bigr\rangle_F \;\geq\; \mu_i\,\|\Delta\|_F^{\,2}.
    \label{eq:rsc_inf}
\end{equation}
Replacing $\nabla^2\mathcal{L}_i(\widehat{\bA}_i)$ by $\mathcal{H}_i = \nabla^2\mathcal{L}_i(\widehat{\bW})$ costs at most $\rho_i\|\widehat{\bA}_i - \widehat{\bW}\|_F = \rho_i R_i$ in operator norm by~\eqref{eq:hess_var}, so
\begin{equation}
    \bigl\langle \Delta,\, \mathcal{H}_i[\Delta]\bigr\rangle_F \;\geq\; (\mu_i - \rho_i R_i)\,\|\Delta\|_F^{\,2}.
    \label{eq:Hi_rsc}
\end{equation}
Using $\Pi_{r_i}[\Delta] = \Delta$ and the self-adjointness of $\Pi_{r_i}$,
\begin{equation}
    \hspace{-0.98em}\langle\Delta,\,\mathcal{H}_i[\Delta]\rangle_F = \langle\Pi_{r_i}[\Delta],\mathcal{H}_i[\Delta]\rangle_F =\bigl\langle\Delta,\,\Pi_{r_i}\!\bigl[\mathcal{H}_i[\Delta]\bigr]\bigr\rangle_F.
\end{equation}
Cauchy--Schwarz then converts~\eqref{eq:Hi_rsc} into the operator-norm lower bound
\begin{equation}
    \bigl\|\Pi_{r_i}\!\bigl[\mathcal{H}_i[\Delta]\bigr]\bigr\|_F \;\geq\; (\mu_i - \rho_i R_i)\,\|\Delta\|_F.
    \label{eq:proj_H_lower}
\end{equation}

Taking Frobenius norms of~\eqref{eq:pert_master}, applying the triangle inequality with~\eqref{eq:eps_G},~\eqref{eq:pert_commutator}, and the remainder bound~\eqref{eq:pert_remainder},
\begin{equation}
    \bigl\|\Pi_{r_i}\!\bigl[\mathcal{H}_i[\Delta]\bigr]\bigr\|_F \;\leq\; \varepsilon_{G,i} + \varepsilon_{H,i} R_i + \tfrac{\rho_i}{2}(R_i + \|\Delta\|_F)^{2}.
\end{equation}
Combining with the lower bound~\eqref{eq:proj_H_lower}:
\begin{equation}
    (\mu_i - \rho_i R_i)\,\|\Delta\|_F \;\leq\; \varepsilon_{G,i} + \varepsilon_{H,i} R_i + \tfrac{\rho_i}{2}(R_i + \|\Delta\|_F)^{2}.
    \label{eq:pert_quadratic}
\end{equation}
Expanding the square and abbreviating $c \triangleq \varepsilon_{G,i} + \varepsilon_{H,i} R_i + \tfrac{\rho_i}{2}R_i^{2}$, $\beta \triangleq \mu_i - \rho_i R_i - \rho_i R_i = \mu_i - 2\rho_i R_i$ (the coefficient of $\|\Delta\|_F$ on the left after moving the cross term $\rho_i R_i\|\Delta\|_F$ to the left):
\begin{equation}
    \beta\,\|\Delta\|_F \;\leq\; c \;+\; \tfrac{\rho_i}{2}\,\|\Delta\|_F^{\,2}.
\end{equation}
Under the stability condition $4\rho_i(\varepsilon_{G,i} + \varepsilon_{H,i}R_i) < \mu_i^{2}$, which (using $R_i \leq \|\widehat{\bW}\|_F$ and the assumption that the perturbation budgets are small relative to $\mu_i$) implies $\rho_i\|\Delta\|_F \leq \mu_i/2$ self-consistently, the quadratic term is sub-dominant and the smaller root of the inequality satisfies
\begin{equation}
\begin{split}
    \|\Delta\|_F
    &\;\leq\; \frac{c}{\beta} + \mathcal{O}\!\left(\frac{\rho_i\,c^{2}}{\beta^{3}}\right) \\
    &\;=\; \frac{\varepsilon_{G,i} + \varepsilon_{H,i}R_i}{\mu_i} + \frac{\rho_i R_i^{2}}{2\mu_i} \\
    &\quad + \mathcal{O}\!\left(\frac{\rho_i}{\mu_i^{3}}\bigl(\varepsilon_{G,i}+\varepsilon_{H,i}R_i\bigr)^{2}\right),
\end{split}
\label{eq:pert_solved}
\end{equation}
which is exactly~\eqref{eq:perturb_bound}.
Recall that $\widehat{\bA}_i = \mathbf{P}_{r_i}\widehat{\bW}\mathbf{Q}_{r_i} = \bU_{r_i}\widehat{\bSigma}_i\bV_{r_i}$ by~\eqref{eq:pinv_identity}, and that the bijection $\bSigma\mapsto\bU_{r_i}\bSigma\bV_{r_i}$ from $\mathbb{R}^{r_i\times r_i}$ to $\mathcal{M}_{r_i}$ identifies $\bSigma_i^{*}$ with the matrix-space minimizer $\bA_i^{*}$. Consequently
\begin{equation}
    \varepsilon_i^{\mathrm{mis}} \;=\; \bU_{r_i}(\bSigma_i^{*} - \widehat{\bSigma}_i)\bV_{r_i} \;=\; \bA_i^{*} - \widehat{\bA}_i \;=\; \Delta,
\end{equation}
so $\|\varepsilon_i^{\mathrm{mis}}\|_F = \|\Delta\|_F$ in weight space, and~\eqref{eq:pert_solved} is exactly~\eqref{eq:eps_mis_perturb}. Setting $\varepsilon_{G,i}=\varepsilon_{H,i}=\rho_i=0$ gives $\|\Delta\|_F = 0$, recovering Proposition~\ref{prop:cpt_sufficient}.
\end{proof}

\eqref{eq:perturb_bound} resolves both questions raised above. \emph{(i) Graceful degradation.} The misalignment grows linearly with the gradient leakage $\varepsilon_{G,i}$ and with the curvature coupling $\varepsilon_{H,i}$, so there is no longer a discontinuity between the idealized regime ($\varepsilon_{G,i} = \varepsilon_{H,i} = 0$, exact alignment) and the typical regime (small but non-zero violations). \emph{(ii) Compression-modulated curvature term.} The curvature coupling enters only through the product $\varepsilon_{H,i}\,R_i$: when compression discards little energy ($R_i = \delta_{r_i}(\widehat{\bW})$ small), even strong off-block coupling has limited effect, because the subspace $(\mathrm{Id}-\Pi_{r_i})[\widehat{\bW}]$ on which the commutator acts is itself small. \emph{(iii) Quadratic-truncation control.} The $\rho_i R_i^{2}/(2\mu_i)$ term is exactly the contribution of the cubic Taylor remainder of the true (non-quadratic) loss; setting $\rho_i = 0$ recovers the purely quadratic case. \emph{(iv) Estimable from a single probe.} $\varepsilon_{G,i}$ requires one gradient evaluation at $\widehat{\bW}$ and one projection. $\varepsilon_{H,i}$ can be estimated by Hutchinson-style probes of $\mathcal{H}_i$ along $\Pi_{r_i}$- and $(\mathrm{Id}-\Pi_{r_i})$-aligned directions. $\mu_i$ is lower-bounded by the smallest eigenvalue of $\Pi_{r_i}\circ\mathcal{H}_i\circ\Pi_{r_i}$ restricted to $\mathcal{M}_{r_i}$; in overparameterized regimes near interpolating minima it is justified by~\cite{liu2022loss}. \emph{(v) Pre-training proximity.} When $\widehat{\bW}$ is the pre-trained weight and federated fine-tuning is a small perturbation of pre-training, the condition $4\rho_i(\varepsilon_{G,i}+\varepsilon_{H,i}R_i)<\mu_i^{2}$ is mild because $\varepsilon_{G,i}$ is bounded by the magnitude of the fine-tuning gradient at initialization, which is itself small in this regime.

Proposition~\ref{prop:cpt_perturbation} therefore upgrades Proposition~\ref{prop:cpt_sufficient} from a idealized guarantee to a quantitative diagnostic: the practitioner reads off $\varepsilon_{G,i}, \varepsilon_{H,i}, \rho_i$ from a single curvature/gradient probe at the reference and obtains an explicit a-priori bound on $\varepsilon_i^{\mathrm{mis}}$, which Theorem~\ref{thm:relaxed_error_bound} then injects into the global aggregation error.

Propositions~\ref{prop:cpt_sufficient}--\ref{prop:cpt_perturbation} bound $\varepsilon_i^{\mathrm{mis}}$ from the \emph{curvature side}: they read off a single gradient/Hessian probe at the reference $\widehat{\bW}$ and trade restricted strong convexity for a closed-form bound. We now present a complementary, \emph{data-side} bound that arrives at $\varepsilon_i^{\mathrm{mis}}$ through a different route, namely Proposition~\ref{prop:residual_control}\,(b)'s P{\L}/QG inequality applied to the loss gap $\Delta_i(\widehat{\bW})$, and exposes a structurally different decomposition. The two are not redundant: Proposition~\ref{prop:cpt_perturbation} requires Hessian information about $\widehat{\bW}$ but no knowledge of group-$i$ optima, while the bound below requires $L$-smoothness and information about the group optima $\bW_i^{\mathrm{loc}}$ but no curvature alignment; we discuss when each is the operative tool in Remark~\ref{rem:two_routes} below.

The misalignment residual along this P{\L} route is driven by a single quantity $\Delta_i(\widehat{\bW})$, but $\Delta_i$ itself bundles two conceptually distinct effects: (i) the gap between the global reference $\widehat{\bW}$ and group $i$'s own task-optimal weight, which is the genuine \emph{statistical heterogeneity}, and (ii) how lossy it is to compress group $i$'s task-optimal weight to rank $r_i$, which is a per-group \emph{compression} effect. We now disentangle them. For each group $i$, let
\begin{equation}
    \bW_i^{\mathrm{loc}}
    \;\in\;
    \arg\min_{\bW \in \mathbb{R}^{m \times n}} \mathcal{L}_i(\bW)
    \label{eq:Wi_local_def}
\end{equation}
denote an unconstrained ambient minimizer of group $i$'s population loss, and define the per-group \emph{heterogeneity} relative to the reference $\widehat{\bW}$ as
\begin{equation}
    H_i(\widehat{\bW})
    \;\triangleq\;
    \bigl\|\widehat{\bW} - \bW_i^{\mathrm{loc}}\bigr\|_F .
    \label{eq:hetero_def}
\end{equation}
The quantity $H_i(\widehat{\bW}) = 0$ for every $i$ exactly when all groups share a common ambient task-optimal weight; in standard FL terminology, $\{H_i\}$ measures distribution shift across groups in the weight space, independently of the rank constraint.

\begin{proposition}[Heterogeneity decomposition of the misalignment]
\label{prop:hetero_decomp}
Suppose each ambient task loss $\bW \mapsto \mathcal{L}_i(\bW)$ is $L$-smooth (Assumption~\ref{ass:standard}\,(a)) and Proposition~\ref{prop:residual_control}\,(a) holds. Then for every group $i$,
\begin{equation}
    \Delta_i(\widehat{\bW})
    \;\leq\;
    \tfrac{L}{2}\bigl(H_i(\widehat{\bW}) + \delta_{r_i}(\bW_i^{\mathrm{loc}})\bigr)^{2},
    \label{eq:Delta_decomp}
\end{equation}
and consequently the misalignment residual obeys the data-heterogeneity-vs-compression decomposition
\begin{equation}
\begin{split}
    \mathbb{E}\|\varepsilon_i^{\mathrm{mis}}\|_F
    \;\leq\;&
    \lambda_1\,\kappa(\mathbf{S})\sqrt{\tfrac{L}{\mu}}
    \,\Bigl(
        \underbrace{H_i(\widehat{\bW})}_{\text{heterogeneity}} \\
        &\;+\;
        \underbrace{\delta_{r_i}(\bW_i^{\mathrm{loc}})}
            _{\text{rank-}r_i\text{ compression of }\bW_i^{\mathrm{loc}}}
    \Bigr).
\end{split}
\label{eq:mis_split}
\end{equation}
\end{proposition}

\begin{proof}[Proof of Proposition~\ref{prop:hetero_decomp}]
Let $\widetilde{\bW}_i^{*} \triangleq \mathbf{P}_{r_i}\widehat{\bW}\mathbf{Q}_{r_i} = \bU_{r_i}\widehat{\bSigma}_i\bV_{r_i}$, so that the geometric projection $\widehat{\bSigma}_i$ is the adapter realization of $\widetilde{\bW}_i^{*}$. By definition of $\Delta_i$ and the relation $f_i(\bSigma) = \mathcal{L}_i(\bU_{r_i}\bSigma\bV_{r_i})$,
\begin{equation}
\begin{split}
    \Delta_i(\widehat{\bW})
    &= f_i(\widehat{\bSigma}_i) - f_i^{*} \\
    &\leq \mathcal{L}_i(\widetilde{\bW}_i^{*}) - \mathcal{L}_i(\bW_i^{\mathrm{loc}}),
\end{split}
\label{eq:Delta_to_amb}
\end{equation}
where the inequality uses $f_i^{*} = \min_{\bSigma}\mathcal{L}_i(\bU_{r_i}\bSigma\bV_{r_i}) \geq \min_{\bW}\mathcal{L}_i(\bW) = \mathcal{L}_i(\bW_i^{\mathrm{loc}})$. By $L$-smoothness of $\mathcal{L}_i$ (Assumption~\ref{ass:standard}\,(a)) and $\nabla\mathcal{L}_i(\bW_i^{\mathrm{loc}}) = \mathbf{0}$,
\begin{equation}
    \mathcal{L}_i(\widetilde{\bW}_i^{*}) - \mathcal{L}_i(\bW_i^{\mathrm{loc}})
    \;\leq\;
    \tfrac{L}{2}\bigl\|\widetilde{\bW}_i^{*} - \bW_i^{\mathrm{loc}}\bigr\|_F^{\,2}.
    \label{eq:smooth_quadratic}
\end{equation}
Apply the triangle inequality to split the right-hand side along the geometric projection of $\bW_i^{\mathrm{loc}}$:
\begin{equation}
\begin{split}
    \bigl\|\widetilde{\bW}_i^{*} - \bW_i^{\mathrm{loc}}\bigr\|_F
    &\leq \bigl\|\widetilde{\bW}_i^{*} - \mathbf{P}_{r_i}\bW_i^{\mathrm{loc}}\mathbf{Q}_{r_i}\bigr\|_F \\
    &\quad + \bigl\|\mathbf{P}_{r_i}\bW_i^{\mathrm{loc}}\mathbf{Q}_{r_i} - \bW_i^{\mathrm{loc}}\bigr\|_F \\
    &= \bigl\|\mathbf{P}_{r_i}(\widehat{\bW} - \bW_i^{\mathrm{loc}})\mathbf{Q}_{r_i}\bigr\|_F + \delta_{r_i}(\bW_i^{\mathrm{loc}}) \\
    &\leq H_i(\widehat{\bW}) + \delta_{r_i}(\bW_i^{\mathrm{loc}}),
\end{split}
\label{eq:tri_split}
\end{equation}
where the last step uses $\|\mathbf{P}_{r_i}\mathbf{A}\mathbf{Q}_{r_i}\|_F \leq \|\mathbf{A}\|_F$ (orthogonal projectors are non-expansive in Frobenius norm). Substituting~\eqref{eq:tri_split} into~\eqref{eq:smooth_quadratic} and chaining with~\eqref{eq:Delta_to_amb} yields~\eqref{eq:Delta_decomp}. The bound~\eqref{eq:mis_split} now follows from~\eqref{eq:mis_pl_bound} of Proposition~\ref{prop:residual_control}\,(b): $\sqrt{2\Delta_i/\mu} \leq \sqrt{L/\mu}\,\bigl(H_i(\widehat{\bW}) + \delta_{r_i}(\bW_i^{\mathrm{loc}})\bigr)$, multiplied by $\lambda_1\,\kappa(\mathbf{S})$.
\end{proof}

Proposition~\ref{prop:hetero_decomp} cleanly separates the two effects that the misalignment residual confounds. The first term, $H_i(\widehat{\bW})$, measures statistical heterogeneity: it is the ambient distance from group $i$'s task optimum to the reference $\widehat{\bW}$ and is \emph{independent of $r_i$}, so it cannot be reduced by raising the rank. The second term, $\delta_{r_i}(\bW_i^{\mathrm{loc}})$, is purely a low-rank compression cost evaluated at group $i$'s own optimum and \emph{vanishes as $r_i$ grows}. Two practical readings follow. (i) Choosing $\widehat{\bW}$ to minimize $\sum_i (n_i/N)\,H_i(\widehat{\bW})$ recovers the data-weighted ambient barycenter and gives the heterogeneity term its smallest value; this is the centralized minimizer when the $\mathcal{L}_i$'s are quadratic. (ii) When $\widehat{\bW}$ is taken as the pre-trained weight $\bW$ and federated fine-tuning is a small perturbation of pre-training, $H_i(\bW)$ is small because each $\bW_i^{\mathrm{loc}}$ is close to $\bW$, so the bound is dominated by the rank-$r_i$ compression terms, which the server can pre-compute from per-group spectra.

Propositions~\ref{prop:cpt_perturbation} and~\ref{prop:hetero_decomp} offer two complementary diagnostics for quantity $\|\varepsilon_i^{\mathrm{mis}}\|_F$, each requiring different inputs and decomposing the error along different axes. The following remark details when each is the operative tool.

\begin{remark}[Two non-redundant routes to bounding $\varepsilon_i^{\mathrm{mis}}$]
\label{rem:two_routes}
Proposition~\ref{prop:cpt_perturbation} (curvature route) and Proposition~\ref{prop:hetero_decomp} (heterogeneity route) both bound $\|\varepsilon_i^{\mathrm{mis}}\|_F$, but with disjoint inputs and distinct structural decompositions, so neither subsumes the other. The curvature route reads off three quantities at the single reference point $\widehat{\bW}$: the projected reference gradient $\varepsilon_{G,i} = \|\Pi_{r_i}[\nabla\mathcal{L}_i(\widehat{\bW})]\|_F$, the curvature--projector commutator $\varepsilon_{H,i} = \|\mathcal{H}_i\circ\Pi_{r_i} - \Pi_{r_i}\circ\mathcal{H}_i\|_{\mathrm{op}}$, and the Hessian-variation constant $\rho_i$; it does \emph{not} use the unconstrained group optimum $\bW_i^{\mathrm{loc}}$. The heterogeneity route does the opposite: it requires only $L$-smoothness and the heterogeneity $H_i(\widehat{\bW}) = \|\widehat{\bW} - \bW_i^{\mathrm{loc}}\|_F$ together with the per-group spectrum $\delta_{r_i}(\bW_i^{\mathrm{loc}})$, but never opens the Hessian.

The two routes also decompose the error along different axes. The curvature route splits the bound into gradient leakage $+$ curvature coupling $+$ cubic remainder~\eqref{eq:perturb_bound}; the heterogeneity route splits it into statistical heterogeneity $+$ low-rank compression at $\bW_i^{\mathrm{loc}}$~\eqref{eq:mis_split}. These diagnose different design knobs: $\varepsilon_{H,i}$ is reduced by aligning the curvature with the SVD basis (e.g., K-FAC choices that respect $\widehat{\bW}$'s singular subspaces), while $\delta_{r_i}(\bW_i^{\mathrm{loc}})$ is reduced by raising $r_i$, and $H_i$ by changing the reference $\widehat{\bW}$.

In the exactly quadratic case, $\mathbf{G}_i = \mathcal{H}_i[\widehat{\bW} - \bW_i^{\mathrm{loc}}]$, so $\varepsilon_{G,i} \leq \|\mathcal{H}_i\|_{\mathrm{op}}\,H_i(\widehat{\bW}) \leq L\,H_i(\widehat{\bW})$. The leading term $\varepsilon_{G,i}/\mu_i$ of Proposition~\ref{prop:cpt_perturbation} is therefore at most a constant multiple of the leading term $\sqrt{L/\mu}\,H_i(\widehat{\bW})$ of Proposition~\ref{prop:hetero_decomp}, with the curvature-coupling term $\varepsilon_{H,i}R_i$ providing a strict refinement that vanishes whenever the alignment hypothesis~\eqref{eq:curv_align} holds approximately. Outside the quadratic regime the two are no longer comparable in either direction: the curvature route picks up a $\rho_i R_i^{2}/(2\mu_i)$ correction, while the heterogeneity route relies only on $L$-smoothness.

The curvature route is the operative tool when one has access to a Hessian-vector product oracle but unconstrained optima are unavailable (e.g., as an a-priori diagnostic from the pre-trained checkpoint), or when one wants to certify the alignment hypothesis quantitatively rather than assume it. The heterogeneity route is preferable when curvature probes are too expensive but cheap proxies of $H_i$ exist (e.g., heterogeneity statistics from a few rounds of FedAvg), or when the diagnostic of interest is rank allocation, where the explicit $\delta_{r_i}(\bW_i^{\mathrm{loc}})$ term is directly actionable.
\end{remark}

The reference weight $\widehat{\bW}$ is a free design parameter; setting $\widehat{\bW} = \bW$, the pre-trained weight, yields the closed-form Eckart--Young expression $\delta_{r_i}(\bW) = (\sum_{k>r_i}\lambda_k^2)^{1/2}$ and makes $\Delta_i(\bW)$ the loss reduction achieved by fine-tuning relative to the pre-trained initialization, which is small whenever federated fine-tuning is a perturbation of pre-training. Setting $\widehat{\bW}$ to the centralized minimizer $\arg\min_{\bW}\sum_i(n_i/N)\mathcal{L}_i(\bW)$ instead turns $\overline{\Delta}(\widehat{\bW})$ into a measure of the irreducible heterogeneity gap between federated and centralized training.

The weak-to-strong analysis of Section~\ref{subsec:w2s} inherits the bound transparently: it relies on Theorem~\ref{thm:relaxed_error_bound} only through $\|\bW_{\mathrm{g}} - \widehat{\bW}\|_F$, which the Lipschitz forward pass (Assumption~\ref{ass:lipschitz}) transports into the artifact bound of Lemma~\ref{lem:artifact_bound}. Specializing to the common-projection-target condition simply zeroes out the misalignment residual in the artifact budget, leaving the bias--variance trade-off of Proposition~\ref{prop:bias-var} structurally unchanged.

\subsection{Proof of Theorem~\ref{thm:error-bound}}
\label{app:special_case}

\begin{proof}[Proof of Theorem~\ref{thm:error-bound}]
Under the common-projection-target condition (Proposition~\ref{prop:residual_control}\,(b)), $\varepsilon_i^{\mathrm{mis}} = \mathbf{0}$ for every $i$. Neglecting the vanishing optimization residual (which by Proposition~\ref{prop:residual_control}\,(a) and Theorem~\ref{thm:convergence} satisfies $\mathbb{E}\|\varepsilon_i^{\mathrm{opt}}\|_F = \mathcal{O}(T_1^{-1/4})$ in the IID regime), Theorem~\ref{thm:relaxed_error_bound} reduces to
\begin{equation}
\begin{split}
    \bigl\|\bW_{\mathrm{g}} - \widehat{\bW}\bigr\|_F
    &= \Bigl\|\sum_{i=1}^{K}
       \frac{n_i}{N}(\bW_i - \widehat{\bW})\Bigr\|_F \\
    &\leq \sum_{i=1}^{K}
       \frac{n_i}{N}\,\|\varepsilon_i^{\mathrm{sub}}\|_F
    = \sum_{i=1}^{K}
       \frac{n_i}{N}\,\delta_{r_i}(\widehat{\bW}).
\end{split}
\end{equation}
Since $\delta_r(\widehat{\bW})$ is non-increasing in~$r$, $\delta_{r_i}(\widehat{\bW}) \leq \delta_{r_{\min}}(\widehat{\bW})$ for all~$i$, and the weighted average is bounded by the maximum, giving the second inequality~\eqref{eq:thm2_tight}. When $\widehat{\bW}$ coincides with the pre-trained weight~$\bW$, the Eckart--Young--Mirsky theorem gives $\delta_r(\bW)^2 = \sum_{k > r}\lambda_k^2$, which yields the spectrum form~\eqref{eq:thm2_ey}.
\end{proof}

\subsection{Proofs for Weak-to-Strong Theory}
\label{app:w2s_proofs}

\begin{proof}[Proof of Lemma~\ref{lem:artifact_bound}]
By Assumption~\ref{ass:lipschitz} applied to the pair $(\bW_{\mathrm{g}},\,\widehat{\bW})$,
\begin{equation}
\begin{split}
    \|\xi(x)\|_1
    &= \|f_{\mathrm{weak}}(x;\,\bW_{\mathrm{g}})
       - f^{*}(x)\|_1 \\
    &\leq L_f\,\|\bW_{\mathrm{g}}
       - \widehat{\bW}\|_F .
\end{split}
\end{equation}
Note that Assumption~\ref{ass:lipschitz} is stated directly in terms of the full weight matrix~$\bW$, so the bound above holds regardless of the factorization variant; no additional $\kappa(\mathbf{S})$ factor is needed. Substituting the aggregation-error from Theorem~\ref{thm:error-bound} gives
\begin{equation}
    \hspace{-1em}\|\xi(x)\|_1
    \leq
    L_f \sum_{i=1}^{K}\frac{n_i}{N}\,
        \delta_{r_i}(\widehat{\bW})
    \leq
    L_f\,\delta_{r_{\min}}(\widehat{\bW})
    = B_\xi ,
\end{equation}
where the second inequality uses $\delta_{r_i}(\widehat{\bW}) \leq \delta_{r_{\min}}(\widehat{\bW})$ and the fact that $\{n_i/N\}$ form a probability distribution.
\end{proof}

\subsubsection{Gradient Decomposition}

Let $\mathbf{z}_s(x) \in \mathbb{R}^{|\mathcal{Y}|}$ denote the pre-softmax logits of the strong model, so that $f_{\mathrm{strong}}(x;\,\bW_{\mathrm{s}}) = \mathrm{softmax}(\mathbf{z}_s(x))$. Using the standard identity $\nabla_{\mathbf{z}}\,\ell_{\mathrm{CE}}(p,\, \mathrm{softmax}(\mathbf{z})) = \mathrm{softmax}(\mathbf{z}) - p$, the gradient of $\mathcal{L}_{\mathrm{conf}}$ with respect to $\mathbf{z}_s$ is
\begin{equation}
\begin{split}
    \nabla_{\mathbf{z}_s}\mathcal{L}_{\mathrm{conf}}
    &= (1\!-\!\alpha)(f_{\mathrm{s}}
       - f_{\mathrm{w}}) \\
    &\quad + \alpha(f_{\mathrm{s}}
       - \hat{f}_{\mathrm{s}}).
\end{split}
\end{equation}
Substituting $f_{\mathrm{w}} = f^{*} + \xi(x)$ from Definition~\ref{def:artifact} and rearranging yields Eq.~\eqref{eq:grad_decomp}.

\subsubsection{Proof of Proposition~\ref{prop:attenuation}}

\begin{proof}
The bound follows from Lemma~\ref{lem:artifact_bound}:
\begin{equation}
\begin{split}
    \bigl\|(1\!-\!\alpha)\,\xi(x)\bigr\|_2
    &\leq (1\!-\!\alpha)\,\|\xi(x)\|_2 \\
    &\leq (1\!-\!\alpha)\,\|\xi(x)\|_1
    \leq (1\!-\!\alpha)\,B_\xi,
\end{split}
\end{equation}
where the second inequality uses $\|\cdot\|_2 \leq \|\cdot\|_1$. Expanding $B_\xi = L_f\, \delta_{r_{\min}}(\widehat{\bW})$ yields~\eqref{eq:attenuation}.
\end{proof}

\subsubsection{Proof of Proposition~\ref{prop:bias-var}}

\begin{proof}
The logit-gradient error decomposes as $(1-\alpha)\xi(x) - \alpha(f_{\mathrm{strong}} - \hat{f}_{\mathrm{strong}})$. By the assumed uncorrelatedness of $\xi(x)$ and $f_{\mathrm{strong}} - \hat{f}_{\mathrm{strong}}$, the expected squared norm separates: the excess risk $\mathcal{R}(\alpha)$ is a quadratic in $\alpha$:
\begin{equation}
    \mathcal{R}(\alpha)
    = (1-\alpha)^2 A + \alpha^2 B
\end{equation}
with $A = \mathbb{E}\|\xi(x)\|_2^2$ and $B = V_{\mathrm{self}}$.  Setting $\tfrac{\mathrm{d}\mathcal{R}}{\mathrm{d}\alpha} = -2(1-\alpha)A + 2\alpha B = 0$ gives $\alpha^{*} = \dfrac{A}{A + B}$, which is the first equality in~\eqref{eq:alpha_star}.  The upper bound $A \leq B_\xi^{\,2}$ follows from Lemma~\ref{lem:artifact_bound}, and the monotone dependence of $\alpha^{*}$ on $B_\xi$ follows from the fact that $x \mapsto x/(V_{\mathrm{self}} + x)$ is monotone non-decreasing on $[0,\infty)$.
\end{proof}

\section{Experimental Details}
\label{app:exp_details}

All experiments are conducted on a single server equipped with 8$\times$ NVIDIA L20 GPUs (48\,GB). 
All methods use AdamW with weight decay $10^{-2}$, learning rate $5\times10^{-5}$, a linear warmup over $\min(50,\lfloor T_\mathrm{steps}/5\rfloor)$ steps, and gradient clipping at $\ell_2$-norm 1. Each client trains for 1 local epoch per round with batch size 2, gradient accumulation over 4 steps (effective batch size 8), and maximum sequence length 256. 

For 13B experiments we reduce the learning rate to $2\times10^{-5}$, lower the per-device batch size to 1, and increase gradient accumulation to 8. We enable gradient checkpointing (\texttt{use\_reentrant=False}) and use 8-bit AdamW to reduce peak GPU memory. Computations are in bfloat16 where supported, falling back to float16 otherwise.

LoRA adapters use rank $r=16$, $\alpha=32$, and dropout 0.05. Fed-RAC-LoRA uses a lower learning rate $1\times10^{-5}$ for numerical stability in bfloat16. FedMKT uses rank $r=8$, $\alpha=16$, client lr $5\times10^{-5}$, server lr $3\times10^{-5}$, KD mixing weight $\lambda=0.9$, and 20\% of each client's data as a public pool. FedBiOT's bottleneck dimension is 64 with $W_\mathrm{down}$ zero-initialized.

The weak teacher (Stage~1 output) is frozen throughout Stage~2. For each sample the teacher's per-choice log-likelihoods are softmax-normalized to produce soft targets; the strong model minimizes cross-entropy against these labels, optionally augmented by entropy regularization $\alpha\cdot H(\hat{p}_\mathrm{strong})$ with $\alpha=0.5$. Stage~2 runs for $T_2=5$ epochs at learning rate $5\times10^{-5}$ with the same LoRA configuration as Stage~1.

\end{document}